%% file: main_arxiv.tex
\documentclass[accepted]{uai2026} %

\usepackage[british]{babel}

\usepackage{natbib} %
\usepackage{mathtools} %
\usepackage{booktabs} %
\usepackage{tikz} %

\usepackage{subcaption} %

\usepackage{amssymb}
\input{preamble.tex}

\title{Kernel Integrated $R^2$: A Measure of Dependence}

\author[1]{\href{mailto:s.mirrezaeiroudaki@lse.ac.uk?Subject=KMD}{Pouya Roudaki}{}}
\author[1,2]{Shakeel Gavioli-Akilagun}
\author[3]{Florian Kalinke}
\author[1]{Mona Azadkia}
\author[1]{Zoltán Szabó}
\affil[1]{%
    Department of Statistics\\
    London School of Economics\\
    London, UK
}
\affil[2]{%
    Department of Decision Analytics and Operations\\
    City University of Hong Kong\\
    Hong Kong, China
}
\affil[3]{%
    Chair of Information Systems\\
    Karlsruhe Institute of Technology\\
    Karlsruhe, Germany
}

\begin{document}

\maketitle

\begin{abstract}%
We introduce \emph{kernel integrated $R^2$}, a new measure of statistical dependence that combines the local normalization principle of the recently introduced \emph{integrated $R^2$} with the flexibility of reproducing kernel Hilbert spaces (RKHSs).
The proposed measure extends integrated $R^2$ from scalar responses to responses taking values on general spaces equipped with a characteristic kernel, allowing to measure dependence of multivariate, functional, and structured data, while remaining sensitive to tail behaviour and oscillatory dependence structures.
We establish that (i) this new measure takes values in $[0,1]$, (ii) equals zero if and only if independence holds, and (iii) equals one if and only if the response is almost surely a measurable function of the covariates. Two estimators are proposed: a graph-based method using $K$-nearest neighbours and an RKHS-based method built on conditional mean embeddings. We prove consistency and derive convergence rates for the graph-based estimator, showing its adaptation to intrinsic dimensionality. Numerical experiments on simulated data and a real data experiment in the context of dependency testing for media annotations demonstrate competitive power against state-of-the-art dependence measures, particularly in settings involving non-linear and structured relationships. 
\end{abstract}

\section{Introduction}

Measuring the degree of dependence between two random variables is a long-standing problem in machine learning and statistics, and numerous methods have been proposed over the years; see, for example, the recent surveys by \citet{josse2016measuring, han2021extensions, chatterjee2024survey}. Among the most widely used classical measures of statistical association are Pearson's correlation coefficient, Spearman's $\rho$, and Kendall's $\tau$. These coefficients are highly effective for detecting monotonic relationships, and their asymptotic behaviour is well understood. However, they perform poorly when the underlying association is non-monotonic. 

To overcome this limitation, many alternative measures have been proposed, including the maximal correlation coefficient \citep{hirschfeld1935connection, gebelein1941statistische, renyi1959measures, breiman1985estimating}, methods based on joint cumulative distribution functions and ranks \citep{hoeffding1948non,  blum1961distribution, yanagimoto1970measures, puri1971nonparametric, rosenblatt1975quadratic, csorgo1985testing, romano1988bootstrap, bergsma2014consistent, nandy2016large, weihs2016efficient, han2017distribution, wang2017generalized,  gamboa2018sensitivity, weihs2018symmetric, drton2018high, deb2019multivariate, zhou2025association}, entropy- and mutual information-based measures  \citep{linfoot1957informational, kraskov2004estimating,pal10estimation,poczos11estimation,reshef2011detecting, kandasamy15nonparametric}, copula-based coefficients \citep{sklar1959fonctions, schweizer1981nonparametric, krishner09ica,poczos10rego,dette2013copula, lopez2013randomized, kong2019composite, zhang2019bet, griessenberger2022multivariate}, measures based on pairwise distances \citep{friedman1983graph, szekely07measuring, szekely09brownian, heller2013consistent, lyons13distance, pan2020ball}, and kernel-based methods \citep{gretton05measuring, gretton08kernel, poczos12copula, sen2014testing, pfister18kernel, zhang18large}. Notice that, while developed independently in the machine learning and statistics communities, Hilbert-Schmidt independence criterion (\citealt{gretton05measuring,gretton08kernel}, based on kernels) and distance covariance (\citealt{szekely07measuring,szekely09brownian,lyons13distance}, based on metrics) are now known to be equivalent \citep{sejdinovic13equivalence}. 

More recently, \citet{chatterjee21newcoefficient} introduced a new correlation coefficient (i) that is as simple to compute as classical measures, (ii) yet serves as a consistent estimator of a dependence measure $\xi(X,Y)$, (iii) it takes values in $[0,1]$, (iv) it equals $0$ if and only if $X$ and $Y$ are independent, and (iv) it equals $1$ if and only if one variable is a measurable function of the other. While $\xi(X,Y)$ was already known as the limit of a copula-based estimator when both $Y$ and $X$ are continuous random variables~\citep{dette2013copula}, the simplicity, computational efficiency, and interpretability of Chatterjee’s correlation have generated substantial interest, leading to a rapidly growing literature on its theoretical properties and extensions to more complex settings \citep{cao2020correlations, deb20measuringassociationtopologicalspaces, azadkia2021simple, griessenberger2022multivariate, shi2022power, bickel2022measures, gamboa2022global, huang22kernel, lin2023boosting, zhang2023asymptotic, auddy2024exact, lin2024failure, fuchs2024quantifying, han2024azadkia, shi2024azadkia, strothmann2024rearranged,  bucher2024lack, tran2024rank, kroll2024asymptotic, ansari2022simple, dette2025simple, zhang2025relationships, yang2025coverage, huang2025multivariate}.

In particular, \citet{deb20measuringassociationtopologicalspaces} extended Chatterjee’s correlation to allow handling random variables $X$ and $Y$ taking values in topological spaces under mild conditions.\footnote{More precisely, $Y$ must take values in a Hausdorff space enriched with a characteristic kernel and the regular conditional distribution of $Y$ given $X$ must exist. The latter can be guaranteed if $Y$ takes values in a Polish space. Additional assumptions depend on the respective estimator; we do not recall these here and refer to their article for more details.} They employ the kernel mean embedding \citep{berlinet04reproducing,smola07hilbert} and its conditional variant \citep{fukumizu07kernel,song09hilbert,klebanov20conditional,park20conditionalmean}, which permit mapping (conditional) probability measures into a reproducing kernel Hilbert space (RKHS; \citealt{aronszajn50theory,steinwart08support,paulsen16introduction}) by using a symmetric positive definite function, the kernel function. If the mapping is injective, the kernel is called characteristic \citep{fukumizu07kernel,sriperumbudur10hilbert} and the RKHS distance of mean embeddings induces a metric on the space of probability measures, which underpins the well-known maximum mean discrepancy \citep{smola07hilbert,gretton12kernel}; it is also known as Hilbert-Schmidt independence criterion \citep{gretton05measuring,quadrianto09kernelized,pfister18kernel} if applied to measuring the distance of a joint distribution to the product of its marginals.
In this sense, the measure proposed by \citet{deb20measuringassociationtopologicalspaces} can be interpreted as the (normalized) average (w.r.t.\ $X$) MMD distance of the distribution of $Y$ and the distribution of $Y$ given $X$; the computational tractability of RKHS methods allows estimating this quantity. Besides rigorously analysing their proposed population quantity and different families of estimators, \citet{deb20measuringassociationtopologicalspaces} demonstrated empirically  that their extension can exhibit greater power for detecting dependence than the original scalar-based coefficient. Additionally, as, for example, kernels for strings \citep{watkins99dynamic,lodhi02text} or more generally for sequences \citep{kiraly19kernel}, sets \citep{haussler99convolution, gartner02multi}, rankings \citep{jiao16kendall}, fuzzy domains \citep{guevara17cross} and graphs \citep{borgwardt20graph} are known, their approach is broadly applicable.

Following the development of Chatterjee’s correlation, \citet{azadkia2025} recently introduced a new dependence measure, %
denoted by $\nu(Y,X)$. This measure retains all the desirable properties of Chatterjee’s correlation while exhibiting enhanced sensitivity to dependence structures that manifest in the tails of the distribution or display oscillatory behaviour. However, the structural form of the measure and its associated estimator restrict its applicability to settings in which $Y\in\R$ and $X\in\R^d$.

Motivated by \citet{deb20measuringassociationtopologicalspaces}, we leverage the flexibility of RKHSs to introduce a dependence measure that preserves the core structural features of $\nu(Y,X)$, while extending the applicability beyond real-valued responses and (finite-dimensional)  Euclidean covariates. This extension presents non-trivial technical challenges. Most notably, unlike \citet{deb20measuringassociationtopologicalspaces}, our construction employs a local normalization rather than a global one, which necessitates a more delicate proof strategy. In particular, we make the following \textbf{contributions}. 

\begin{enumerate}
    \item We introduce a kernel-based generalization of $\nu(Y,X)$,  extending it beyond the setting of $X \in \R^d$ and $Y\in\R$. In particular, our population quantity is well-defined if $X$ takes values in a topological space and $Y$ takes values in a Polish space equipped with a continuous characteristic kernel. 
    \item We prove that our proposed generalization has the properties expected of a dependence measure, that is, the quantity takes values between $0$ and $1$, is $0$ if and only if (iff.) $X$ and $Y$ are independent, and is $1$ iff.\ $Y$ is almost surely (a.s.) a measurable function of $X$.
    \item Under mild additional assumptions, we present a graph-based estimator using nearest neighbours, and an RKHS-based estimator of our kernel-based quantity. We provide consistency guarantees and convergence rates for the graph-based estimator.
    \item Experiments on synthetic and real-world datasets show that independence tests using our estimators perform competitively w.r.t.\ the state-of-the-art, especially when considering non-linear and structured associations.
\end{enumerate}

The remainder of the paper has the following \textbf{outline}. In Section~\ref{sec:notation}, we introduce the main notations. Section~\ref{sec:RelatedDependenceMeasures} reviews some closely-related measures of dependence. In Section~\ref{sec:KernelizedIntegratedR}, we introduce our general measure of dependence, termed ``kernel integrated $R^2$''. Section~\ref{sec:estimator} presents two estimation procedures, and Section~\ref{sec:Consistencyandrates} establishes theoretical guarantees, including consistency and convergence rates. The empirical performance of the proposed method is demonstrated in Section~\ref{sec:numerical}. All proofs are collected in the appendix.

\section{Notations}\label{sec:notation}

Next we introduce our notations used throughout the main body of the article: $[n]$, $\mathcal{O}$, $\mathbbm{1}_A$, $\mathbf{1}_n$, $\mathbf I_n$, $\mathbf A^\top$, $\operatorname{Tr}(\mathbf A)$, $\mathbf A^{-1}$, $\circ$, $\mathcal M_1^+(\mathcal Z)$, $\operatorname{supp}(\mathbb{P})$, $\delta_x$, $\mathcal{O}_\pp$, $\mathbb E_\mathbb{P}[\cdot]$, $\mathbb V_\mathbb{P}(\cdot)$, $\mathbb E_Z[\cdot]$, $\mathbb V_Z(\cdot)$, $\operatorname{Cov}$, $\mathbb{P}_X$, $\mathbb{P}_Y$, $\mathbb{P}_{XY}$, $\mathbb{P}_X\otimes \mathbb{P}_Y$, $\mathbb P_{Y\mid X}$, $\mathcal H_k$, $k(\cdot, z)$, $\mu_k$, $\operatorname{MMD}_k$, $ \mathcal H_\mathcal{Y}$, $\mathcal{H_{\mathcal X}}$.

\paragraph{General conventions.} For a positive integer $n \in \mathbb{N}\coloneq \{1,2,\ldots\}$, $[n] \coloneq \{1,\dots,n\}$.
For positive sequences $(a_n)_{n\in \N}$ and $(b_n)_{n\in \N}$, we write $a_n = \mathcal{O}(b_n)$ if there exist constants $C > 0$ and $N \in \mathbb{N}$ such that $a_n \le C\, b_n$ for all $n \ge N$. We denote by $\mathbbm{1}_A$ the indicator of a set $A$: $\mathbbm{1}_A(x)=1$ if $x \in A$; $\mathbbm{1}_A(x)=0$ otherwise. 
The $n$-dimensional vector of ones is denoted by $\mathbf{1}_n$. The identity matrix is $\mathbf I_n\in \R^{n\times n}$. For a matrix $\mathbf A \in \mathbb{R}^{n_1\times n_2}$, its transpose is written as $\mathbf A^\top \in \mathbb{R}^{n_2\times n_1}$; the trace of a square matrix $\mathbf A \in \mathbb R^{n\times n}$ is denoted by $\operatorname{Tr}(\mathbf A)$; for a non-singular matrix $\mathbf A \in \mathbb R^{n\times n}$, its inverse is denoted by $\mathbf A^{-1}\in \mathbb{R}^{n\times n}$. For two matrices $\mathbf A,\mathbf B \in\R^{n_1\times n_2}$, we write their Hadamard product as $\mathbf A \circ \mathbf B =\left[A_{i,j}B_{i,j}\right]_{i,j=1}^{n_1,n_2} \in \R^{n_1\times n_2}$. 

\paragraph{Probability measures and conditioning.}
Let $(\mathcal Z,\tau_{\mathcal Z})$ be a topological space and $\mathcal B(\tau_{\mathcal Z})$ its
Borel $\sigma$-algebra. We write $\mathcal M_1^+(\mathcal Z)$ for the set of Borel probability measures
on the measurable space $(\mathcal Z,\mathcal B(\tau_{\mathcal Z}))$. The support of $\mathbb{P}\in \mathcal M_1^+(\mathcal Z)$ is denoted by $\operatorname{supp}(\mathbb{P})$; it is the set of points 
$z\in\mathcal Z$ for which every open neighbourhood of $z$ has positive $\mathbb{P}$ measure. We write $\delta_z$ for the Dirac delta distribution at $z\in\mathcal Z$. A distribution $\mathbb{P} \in \mathcal M_1^+(\mathcal Z)$ is called degenerate iff.\ $\mathbb{P} = \delta_z$ for some $z\in \mathcal Z$. For a sequence of random variables $X_n$ and sequence $(a_n)_{n\in \N}$, $X_n = \mathcal{O}_\pp(a_n)$ means that $X_n/a_n$ is stochastically bounded, that is for any $\varepsilon > 0$ there exists a finite $M > 0$ and a finite $N > 0$ such that $\pp(\abs{X_n/a_n} > M) < \varepsilon$ for all $n > N$. 
Let $(\mathcal H, \langle \cdot,\cdot\rangle_{\mathcal H})$ be a Hilbert space and $f:\mathcal Z \to \mathcal H$ measurable. If
$
\int_{\mathcal Z}\|f(z)\|_{\mathcal H}\, \mathrm{d}\mathbb{P}(z) < \infty,
$ the expectation of $f$ w.r.t.\ $\mathbb{P}$ is defined as
$
\mathbb E_\mathbb{P}[f] \coloneq \int_{\mathcal Z} f(z)\, \mathrm{d}\mathbb{P}(z)\in\mathcal H,
$
where the integral is meant in Bochner's sense.
If additionally
$
\int_{\mathcal Z}\|f(z)\|_{\mathcal H}^2\, \mathrm{d}\mathbb{P}(z) < \infty,
$
then the variance of $f$ w.r.t.\ $\mathbb{P}$ is
$
\mathbb V_\mathbb{P}(f) \coloneq \int_{\mathcal Z}\|f(z)-\mathbb E_\mathbb{P}[f]\|_{\mathcal H}^2 \,\mathrm{d}\mathbb{P}(z).
$
When $\mathbb{P}$ is the law of a random variable $Z$, we write $\mathbb E_Z[\cdot]$ and $\mathbb V_Z(\cdot)$ for $\mathbb E_\mathbb{P}[\cdot]$ and $\mathbb V_\mathbb{P}(\cdot)$, respectively. The covariance of two real-valued random variables $Z_1$ and $Z_2$ with joint law $\mathbb{P}$ and marginal laws $\mathbb{P}_1$ and $\mathbb{P}_2$, respectively, is $\operatorname{Cov}(Z_1,Z_2) = \mathbb{E}_{Z_1Z_2}\!\left[(Z_1-\mathbb{E}_{Z_1}[Z_1])(Z_2-\mathbb{E}_{Z_2}[Z_2]\right]$.
Let $(\mathcal X,\mathcal B(\tau_{\mathcal X}))$ and $(\mathcal Y,\mathcal B(\tau_{\mathcal Y}))$ be measurable
spaces and $(X,Y)$ a pair of random variables taking values in $\mathcal X\times\mathcal Y$.
We denote their marginal and joint laws by $\mathbb{P}_X\in\mathcal M_1^+(\mathcal X)$,
$\mathbb{P}_Y\in\mathcal M_1^+(\mathcal Y)$, and $\mathbb{P}_{XY}\in\mathcal M_1^+(\mathcal X\times\mathcal Y)$, respectively.
Their product distribution is denoted by $\mathbb{P}_X\otimes \mathbb{P}_Y$. If $\mathcal Y$ is a Polish space, that is, 
a complete separable metrizable topological space, a (regular) conditional distribution of $Y$ given $X$ exists, which we write as $\mathbb P_{Y\mid X}$.

\paragraph{Kernels and RKHS.} \label{paragraph: kernels and RKHS} Let $\mathcal H_k$ be the RKHS on $\mathcal Z$ with (reproducing) kernel $k : \mathcal Z \times \mathcal Z \to \R$; it is the Hilbert space of functions $f : \mathcal Z \to \R$ such that $k(\cdot, z) \in \mathcal H_k$ and $\langle f,k(\cdot,z)\rangle_{\mathcal H_k} = f(z)$ for all $z \in \mathcal Z$ and $f\in \mathcal H_k$, where the canonical feature map $k(\cdot,z)$ stands for $ z' \mapsto k(z',z) \in \mathbb{R}$ for fixed $z$ and any $z' \in \mathcal Z$. Throughout this manuscript, we assume all kernels to be Borel measurable and bounded (for a kernel $k : \mathcal Z \times \mathcal Z \to \R$, the latter property is meant as $\sup_{z,z' \in \mathcal Z}k(z,z') < \infty$). 
For $\mathbb{P}\in\mathcal M_1^+(\mathcal Z)$, the kernel mean embedding of $\mathbb P$ w.r.t.\ $k$ is
$
\mu_k(\mathbb{P}) := \int_{\mathcal Z} k(\cdot,z)\, \mathrm{d}\mathbb{P}(z)\in\mathcal H_k,
$
where the integral is meant in Bochner's sense; the assumed boundedness of $k$ ensures its existence. %
The maximum mean discrepancy of $\mathbb{P},\mathbb{Q}\in\mathcal{M}_1^+(\mathcal Z)$ w.r.t.\ $k$ is $\operatorname{MMD}_k(\mathbb{P},\mathbb{Q}) = \|\mu_k(\mathbb{P})- \mu_k(\mathbb{Q})\|_{\mathcal H_k}$.
A kernel $k$ is called characteristic if the map $\mathbb{P} \mapsto \mu_k(\mathbb{P})$ is injective on
$\mathcal M_1^+(\mathcal Z)$. In this case, $\operatorname{MMD}_k$ induces a metric on $\mathcal M_1^+(\mathcal Z)$.
In the main text, we work with the RKHS %
$\mathcal H_\mathcal{Y} \coloneq \mathcal H_{k_\mathcal{Y}}$ induced by the kernel
$k_\mathcal{Y}:\mathcal Y\times\mathcal Y\to\mathbb R$; for the RKHS-based estimator (Section~\ref{subsection: ridge estimation}), we additionally use the RKHS $\mathcal H_{\mathcal X}\coloneq \mathcal H_{k_\mathcal{X}}$ induced by the kernel $k_{\mathcal X} : \mathcal X \times \mathcal X \to \R$.

\section{Related Dependence Measures}\label{sec:RelatedDependenceMeasures}

For a real-valued random variable $Y\in\rr$ and a random vector $X\in\rr^d$ for $d\geq 1$, \cite{azadkia2025} introduced $\nu(Y,X)$ to quantify the extent of dependence of $Y$ on $X$. If $Y$ admits a continuous density\footnote{The continuous density ensures that \eqref{eq:mona-pouya-coefficient} is well-defined. For the exact definition of $\nu$ in the general case, see~\citet{azadkia2025}.}, $\nu(Y,X)$ can be written as
\begin{align}\label{eq:mona-pouya-coefficient}
    \nu(Y,X)
    =
    \int_\rr
    \frac{\mathbb{V}_X\bigl(\ee_{Y\mid X}[\bone_{\{Y>t\}}]\bigr)}
         {\mathbb{V}_Y\bigl(\bone_{\{Y>t\}}\bigr)}
    \, \mathrm{d}\pp_Y(t),
\end{align}
which is closely related to Chatterjee’s correlation coefficient \citep{chatterjee21newcoefficient},
\begin{align*}
    \xi(X, Y)
    =
    \int_\rr
    \frac{\mathbb{V}_X\bigl(\ee_{Y\mid X}[\bone_{\{Y>t\}}]\bigr)}
         {\int_\rr \mathbb{V}_Y\bigl(\bone_{\{Y>u\}}\bigr)\, \mathrm{d}\pp_Y(u)}
    \, \mathrm{d}\pp_Y(t).
\end{align*}
While introduced for $X,Y \in \mathbb R$, $\xi$ can be extended to allow handling $X\in\mathbb R^d$~\citep{azadkia2021simple}. \footnote{Note that $\xi(X, Y)$ measures the extend of dependence of $Y$ on $X$. Although similar measures are typically written with arguments ordered as $(Y, X)$, we retain the ordering $(X, Y)$ for $\xi$ to remain consistent with~\cite{chatterjee21newcoefficient}, where $\xi$ was originally introduced.}

Although the difference between $\nu$ and $\xi$ may appear marginal, in practice $\nu$ is often more powerful for detecting dependence. The distinction lies in the normalization of $\mathbb{V}_X(\ee_{Y\mid X}[\bone_{\{Y>t\}}])$. In $\nu(Y,X)$, this quantity is normalized pointwise by $\mathbb{V}_Y(\bone_{\{Y>t\}})$, so the conditional variation of $\bone_{\{Y>t\}}$ given $X$ is compared directly to its marginal variation, rather than to an average over all thresholds $t \in \mathbb{R}$. As a consequence, even when $\mathbb{V}_Y(\bone_{\{Y>t\}})$ is small, the dependence of $\bone_{\{Y>t\}}$ on $X$ is not masked by averaging over values of $t$ with larger marginal variability. 

In a different line of work, \citet{deb20measuringassociationtopologicalspaces} proposed $\eta_{k_\mathcal{Y}}$, a general measure of association inspired by $\xi(X, Y)$, by leveraging RKHS methods. Let $Y$ and $Y^\prime$ be conditionally i.i.d.\ given $X$, and let $Y$, $Y_1$, and $Y_2$ be marginally i.i.d.\  Then the measure $\eta_{k_\mathcal{Y}}$ is defined as
\begin{align}\label{eq:etaDef}
    \eta_{k_\mathcal{Y}}&(Y,X) \coloneq
    1
    \!-\!
    \frac{
        \ee_X\!\Big[\ee_{Y\mid X}\!\big[\bigl\|k_\mathcal{Y}(\cdot,Y)-k_\mathcal{Y}(\cdot,Y^\prime)\bigr\|_{\mathcal{H}_\mathcal{Y}}^2\big]\Big]
    }{
        \ee_Y\!\big[\bigl\|k_\mathcal{Y}(\cdot,Y_1)-k_\mathcal{Y}(\cdot,Y_2)\bigr\|_{\mathcal{H}_\mathcal{Y}}^2\big]
    } \\
    &= \int_{\mathcal{X}} \frac{ \operatorname{MMD}^2_{k_\mathcal{Y}} \left ( \mathbb{P}_{Y \mid X = x}, \mathbb{P}_Y \right)}{\int_{\mathcal{Y}} \left \| k (\cdot, y) - \mathbb{E} \left [ k (\cdot, Y) \right ] \right \|_{\mathcal{H}_{\mathcal{Y}}}^2 \mathrm{d} \mathbb{P}_Y (y)} \mathrm{d} \mathbb{P}_X (x). \nonumber
\end{align}

An advantage of $\eta_{k_\mathcal{Y}}$ over $\xi$ is that, by leveraging the RKHS framework, it naturally extends to multivariate responses $Y$ and general data types such as graphs, manifolds, and functional data. Moreover, an appropriate choice of kernel allows domain knowledge about similarity to be incorporated into the measure. Note that for $Y, X\in\rr$, and the Brownian kernel $k_{\mathcal{Y}}(y_1, y_2) = \abs{y_1} + \abs{y_2} - \abs{y_1 - y_2}$, we get $\eta_{k_\mathcal{Y}}(Y, X) = \xi(X, Y)$, hence $\eta_{k_\mathcal{Y}}$ can be viewed as a generalization of $\xi$.

Motivated by the complementary strengths of $\nu$ and $\eta_{k_\mathcal{Y}}$, we combine the power of both approaches by introducing a kernelized version of $\nu$, with the goal of further enhancing its ability to detect dependence and to broaden its applicability.

In doing so, we tackle two key challenges. First, the denominator of the integrand in $\nu(Y, X)$ is $\mathbb{V}_Y(\bone_{\{Y > t\}})$, which can be small in the tails and therefore requires careful control. In contrast, $\eta_{k_\mathcal{Y}}(Y, X)$ involves a single global normalization term. When extending $\nu(Y, X)$ to an RKHS-based framework, one must therefore ensure uniform control of the corresponding denominator.

Second, note that for i.i.d.\ random variables $Z$ and $Z^\prime$, we have 
\begin{align*}
    \mathbb{V}_Z(Z) = \frac{1}{2}\ee_Z\left[\left(Z - Z^\prime\right)^2\right].
\end{align*}
Therefore the numerator and denominator in~\eqref{eq:etaDef} are related to the conditional variance of $Y$ given $X$ and the variance of $Y$, respectively. Hence, \eqref{eq:etaDef} shows that $\eta_{k_\mathcal{Y}}(Y, X)$ can be interpreted as the ratio of two variances in a Hilbert space.  By contrast, constructing a kernel analogue of $\nu$ requires considering the variance of conditional objects across the range of values of the response variable, which does not admit an immediate representation in an RKHS.

Finally, observe that the integrand in~\eqref{eq:mona-pouya-coefficient} can be interpreted as the $R^2$ (coefficient of determination) from the linear regression of $\bone_{\{Y > t\}}$ on $X$, 
\begin{align*}
    R_t^2 = \frac{\mathbb{V}_X(\ee_{Y\mid X}[\bone_{\{Y > t\}}])}{\mathbb{V}_{Y}(\bone_{\{Y > t\}})} = 1 - \frac{\ee_X[\mathbb{V}_{Y\mid X}(\bone_{\{Y > t\}})]}{\mathbb{V}_{Y}(\bone_{\{Y > t\}})}.
\end{align*}
Consequently, $\nu(Y,X) = \int_\rr  R_t^2 \mathrm{d}\pp_Y(t)$ can be viewed as an integrated $R^2$ over the thresholded responses $\bone_{\{Y > t\}}$. Motivated by this perspective, we introduce in the next section a kernel-based analogue, which we refer to as the kernel integrated $R^2$.

\section{Kernel Integrated \texorpdfstring{$R^2$}{R\^{}2}}\label{sec:KernelizedIntegratedR}

In this section, we introduce our measure of dependence, the kernel integrated $R^2$. We first state the set of assumptions that we require for our measure to be well-defined.

\begin{assumption} \label{assumption:general}
(i) $\mathcal{X}$ is a topological space and $\mathcal{Y}$ is a Polish space.
(ii) %
$\mathrm{supp}(\mathbb{P}_Y) = \mathcal{Y}$ and $\mathbb{P}_Y$ is non-degenerate.
(iii) %
$k_{\mathcal{Y}}$ is continuous, characteristic, and there exists no $y\in\mathcal Y$ such that $k_{\mathcal{Y}}(\cdot,y)$ is a constant function.\footnote{This condition holds, for instance, if $\mathcal Y =\R^d$ and $k_{\mathcal Y}$ is the Gaussian kernel.}
\end{assumption}

The following remark elaborates our assumptions.

\begin{remark}\label{remark:main-assumptions}~
\begin{enumerate}
    
    \item \label{remark:item:existence-regular-cond-prob} The assumption that $\mathcal{Y}$ is Polish ensures the existence of regular conditional probabilities. \citep[Theorem~10.2.2]{dudley04real}.

    \item \label{remark:item:separation} As $k_\mathcal{Y}$ is continuous and $\operatorname{supp}(\mathbb{P}_Y) = \mathcal{Y}$, for any $f \in \mathcal{H}_\mathcal{Y}$, we have $f = 0 \, \, \mathbb{P}_{Y}\text{-a.e.}$ iff.\ $f \equiv 0$ on $\mathcal{Y}$ \citep[Assumption~2.1(f) and footnote~7]{klebanov20conditional}. %
    
    \item The characteristic property implies that $k_{\mathcal Y}$ is point-separating \citep[p.~5]{bonnier23cumulants}, that is, $y\mapsto k_{\mathcal Y}(\cdot,y) \in \mathcal H_{\mathcal Y}$ is injective for $y\in\mathcal Y$.\footnote{The point-separating property follows from the characteristic-ness as $k_{\mathcal Y}(\cdot,y) = \mu_{k_{\mathcal Y}}(\delta_{y}) \overset{\text{(char.)}}{\neq} \mu_{k_{\mathcal Y}}(\delta_{y'}) = k_{\mathcal Y}(\cdot, y')$  for any distinct $y,y'\in\mathcal Y$.}
\end{enumerate}
\end{remark}

Together, these properties pave the way to ensuring that the following definition of our quantity of main interest satisfies various natural requirements of a dependence measure, as we elaborate it in Theorem~\ref{theorem: rho properties}.

\begin{definition}[Kernel integrated $R^2$] \label{def:KIR}
    Under Assumption~\ref{assumption:general}, let
\begin{align}
D(Y,X) &\coloneq D(Y,X; k_\mathcal{Y}) \nonumber \\
&\coloneq 1 - \int_\mathcal{Y} \frac{\mathbb E_X\! \left  [ \mathbb{V}_{Y\mid X} \!\left[k_\mathcal{Y}(Y,y)\right] \right ]}{\mathbb{V}_Y \!\left[k_\mathcal{Y}(Y,y)\right]}\mathrm d  \mathbb P_Y(y).
\label{eq:kmd-population-version1}
\end{align}
\end{definition}

\newpage
\begin{remark}~
\begin{enumerate}
    \item When ignoring that $k_{\mathcal Y}$ must be a kernel, and when $\mathcal X = \mathbb R^d$, $\mathcal Y = \mathbb R$, then choosing $k_{\mathcal{Y}}(u,y) = \mathbbm{1}_{\{u > y\}}$, we have\footnote{The function $k_\mathcal{Y}$ is not symmetric, hence it is not a kernel.}
    \begin{equation*}
        D(Y,X) = 1- \int_{\mathbb R} \frac{\mathbb E_X\!\left[\mathbb V_{Y\mid X}[\mathbbm{1}_{\{Y>y\}}]\right]}{\mathbb V_{Y}[\mathbbm{1}_{\{Y>y\}}]} \mathrm d \mathbb P_Y(y), 
    \end{equation*}
    which is formally equivalent to $\nu(Y, X)$. In this sense, $D$ can be thought of as a ``kernel'' version of $\nu$.

    \item The characteristic property of $k_\mathcal{Y}$ and $\operatorname{supp}(\mathbb{P}_Y) = \mathcal{Y}$  ensure that the denominator $\mathbb{V}_Y \!\left[k_\mathcal{Y}(Y,y)\right]$ is non-zero for all $y\in\mathcal Y$; we prove this claim as part of the following Theorem~\ref{theorem: rho properties}.
    
    \end{enumerate}
\end{remark}

Our first theorem shows that $D(Y, X)$ is well-defined and satisfies the standard properties expected of a dependence measure.

\begin{theorem}
\label{theorem: rho properties}
Under Assumption~\ref{assumption:general}, $D(Y, X)$ is well-defined and 
 \begin{enumerate}[label=(\roman*)]
     \item $D(Y, X) \in [0,1]$,
     \item $D(Y, X) = 0$ iff.\ $Y$ and $X$ are independent,  and
     \item $D(Y, X) = 1$ iff.\ there exists a Borel measurable function $f:\mx \to \my$ such that $Y = f(X)$ holds $\mathbb{P}_{XY}$-a.s. 
 \end{enumerate}
\end{theorem}

Note that $D$, like $\xi$, $\eta_{k_{\mathcal{Y}}}$, and $\nu$, is not symmetric in its arguments; this sets these measures apart from many classical measures of dependence, such as Spearman’s $\rho$ or mutual information. Although symmetrized variants can be constructed, dependence itself need not be intrinsically symmetric.\footnote{For example, with $X \sim N(0,1)$ and $Y = X^2$, $Y$ is a measurable function of $X$, but $X$ is not a measurable function of $Y$.} Finally, while there is interest in quantifying the extent of dependence among more than two random variables, our focus here is restricted to pairs of variables. The following Table~\ref{tab:dependence_measures_transposed} summarizes key properties of $D$ in comparison with related measures.

\begin{table}[ht]
\centering
\begin{tabular}{lcccc}
\toprule
Property & $\xi$ & $\nu$ & $\eta_{k_{\mathcal{Y}}}$ & $D$ \\
\midrule
Value in $[0,1]$ 
  & $\checkmark$ & $\checkmark$ & $\checkmark$ & $\checkmark$ \\

0-independent 
  & $\checkmark$ & $\checkmark$ & $\checkmark$ & $\checkmark$ \\

1-full dependence 
  & $\checkmark$ & $\checkmark$ & $\checkmark$ & $\checkmark$ \\

Bijective invariance
  & $\checkmark$ & $\checkmark$ & $\times$ & $\times$ \\

Multivariate $Y$ 
    & $\times$ & $\times$ & $\checkmark$ & $\checkmark$ \\
Kernel-endowed domain $X, Y$
    & $\times$ & $\times$ & $\checkmark$ & $\checkmark$ \\ 
    \bottomrule
\end{tabular}
\caption{Comparison of dependence measures and their properties.}
\label{tab:dependence_measures_transposed}
\end{table}
In the next section, we develop estimators for the population quantity $D(Y,X)$.

\section{Estimation}\label{sec:estimator}
In this section, we introduce two approaches for estimating $D(Y,X)$ from an i.i.d.\ sample $(X_i, Y_i)_{i=1}^n$. %
Both estimators follow the same general strategy.
In estimating $D(Y, X)$, we can approximate the integral with respect to $\pp_Y$ by an average over the observed sample $\{Y_j\}_{j=1}^n$,
\begin{align}
    \lefteqn{\int_\mathcal{Y} \frac{\mathbb E_X \left  [ \mathbb{V}_{Y\mid X} \left[k_\mathcal{Y}(Y,y)\right] \right ]}{\mathbb{V}_Y \left[k_\mathcal{Y}(Y,y)\right]}\mathrm d  \mathbb P_Y(y) \approx } \notag \\
    & \qquad\qquad\qquad\frac{1}{n}\sum_{i = 1}^n \frac{\mathbb E_X\left[ \mathbb{V}_{Y\mid X} \left[k_\mathcal{Y}(Y, Y_i)\right] \right ]}{\mathbb{V}_Y\left[k_\mathcal{Y}(Y,Y_i)\right]}, \label{eq:approximate-by-mean}
\end{align}
where the random part in each summand is $Y$ and $Y_i$ is the observed sample. Hence one needs to provide an estimate for each summand.

While both estimators rely on \eqref{eq:approximate-by-mean}, we emphasize that the two estimators are fundamentally different in nature. The first estimator relies on the nearest-neighbour graph structure induced by the $X_i$'s and therefore requires $\mathcal{X}$ to be a metric space. The second estimator does not impose such a requirement on $\mathcal{X}$ but rather assumes the existence of a kernel $k_{\mathcal X}$ on $\mathcal X$, and makes use of the conditional mean embedding.

\subsection{Nearest-Neighbour Estimator} \label{subsection: graph based estimator}

We begin by introducing an estimator based on approximating the conditional distribution $\mathbb{P}_{Y\mid X}$ using the $K$-nearest neighbours method. To formalize this construction, we impose the following assumption.

\begin{assumption} The space $(\mathcal X, d_{\mathcal X})$ is a metric space. 
\label{assumption: X is metric space}
\end{assumption}

Note that the nominator of \eqref{eq:approximate-by-mean} contains a conditional variance, which for any $y\in\mathcal Y$ can be expressed as
\begin{align*}
    \mathbb{V}_{Y \mid X}\!\left[k_{\mathcal{Y}}(Y_j, y)\right]
    =
    \frac{1}{2}
    \mathbb{E}_{Y \mid X}
    \Big[
        \big(
            k_{\mathcal{Y}}(Y_j, y)
            -
            k_{\mathcal{Y}}(Y_j', y)
        \big)^2
    \Big],
\end{align*}
where $Y_j^\prime$ is an i.i.d.\ copy of $Y_j$ conditional on $X_j$.
To estimate this quantity, we approximate $Y_j^\prime$ using a surrogate.
Given a sample $\{(X_i, Y_i)\}_{i=1}^n$, for each $Y_j$,
we select $Y_k$ as a surrogate for $Y_j^\prime$ whenever
$d_{\mathcal{X}}(X_k, X_j)$ is sufficiently small.
To make the notion of ``small'' precise, we restrict attention to those
$X_k$ that belong to the set of $K$-nearest neighbours of $X_j$.
More precisely, for distinct indices $i$ and $j$, let
$\mathcal{N}_j^{\setminus i}$ denote the set of the $K$ nearest neighbours
of $X_j$ among $\{X_k\}_{k \neq i,j}$, with ties broken uniformly at random. We let for each $i \in [n]$,
{\begin{align}
\MoveEqLeft E_{n, i}^{\text{K-NN}} \coloneq \label{equation: KNN sample moments1}\\
&\frac{1}{2 K (n - 1)}\sum_{j\neq i}\sum_{k\in\mathcal{N}_j^{\setminus i}} \big(k_{\mathcal{Y}}(Y_j, Y_i) - k_{\mathcal{Y}}(Y_k, Y_i)\big)^2, \notag
\\
\MoveEqLeft V_{n, i}^{\text{K-NN}} \coloneq \label{equation: KNN sample moments} \\&\frac{1}{n-1}
\sum_{j \neq i}
k_\mathcal{Y}^2(Y_j, Y_i)
-
\Bigg[
\frac{1}{n-1}
\sum_{j \neq i}
k_\mathcal{Y}(Y_j, Y_i)
\Bigg]^2,\notag
\end{align}
to be the estimators for $\ee_X[\mathbb{V}_{Y\mid X}[k_\mathcal{Y}(Y, Y_i)]]$ and $\mathbb{V}_{Y}[k_\mathcal{Y}(Y, Y_i)]$, respectively. Consequently, we construct an estimator of $D(Y,X)$ as follows.

\begin{definition}[Nearest-neighbour estimator] 
\label{definition: graph based estimator}
Suppose that Assumption~\ref{assumption: X is metric space} holds. Given an i.i.d.\ sample $(X_i,Y_i)_{i=1}^n$ from $\mathbb{P}_{XY}$, 
let
\begin{equation*}
\hat{D}^{\text{K-NN}} (Y,X) := 1 - \frac{1}{n} \sum_{i=1}^n \frac{E_{n,i}^{\text{K-NN}}}{V_{n,i}^{\text{K-NN}}},
\end{equation*}
with $E_{n,i}^{\text{K-NN}}$ and $V_{n,i}^{\text{K-NN}}$ as in  \eqref{equation: KNN sample moments1} and \eqref{equation: KNN sample moments}, respectively.
\end{definition}

We emphasize that the use of the $K$-nearest neighbours is not essential to the construction. In principle, it may be replaced by any geometric graph built on $(X_i)_{i=1}^n$; see \cite{bhattacharya19asymptoticframework} for a general framework.

\subsection{RKHS Estimator} \label{subsection: ridge estimation}

Next, by interpreting the conditional expectations in \eqref{eq:kmd-population-version1} as conditional mean embeddings, we present an alternative estimator of $D(Y,X)$. We begin by introducing the key additional assumptions and necessary notation.

\begin{assumption} \label{ass:rkhs-estimator-assumption}
    (i) $\mathcal X$ is separable, (ii) $\operatorname{supp}(\mathbb P_X) = \mathcal X$, (iii) $k_{\mathcal X}$ is characteristic and continuous, (iv) for all $g\in\mathcal H_{\mathcal Y}$ there exists $h_g\in\mathcal H_{\mathcal X}$ such that
\small        $\operatorname{Cov}(h_g(X) - f_g(X), h(X)) = 0 \text{ for all } h\in \mathcal H_{\mathcal X}$,
    where $f_g(x) = \mathbb E_{Y\mid X=x}[g(Y)]$ for $x \in \mathcal X$, and (v) $k_{\mathcal Y}^2$ is characteristic.\footnote{$k_{\mathcal Y}^2$ is characteristic, for instance, if $k_{\mathcal Y}$ is $c_0$-universal~\citep{szabo18characteristic2}}.
\end{assumption}

\begin{remark} Suppose that Assumption~\ref{assumption:general} and Assumption~\ref{ass:rkhs-estimator-assumption} both hold.
\begin{enumerate}
    \item The separability of $\mathcal{X}$ and $\mathcal{Y}$, together with the continuity of $k_{\mathcal{X}}$ and $k_{\mathcal{Y}}$, imply the separability of $\mathcal{H}_{\mathcal{X}}$ and $\mathcal{H}_{\mathcal{Y}}$, respectively \citep[Lemma~4.33]{steinwart08support}.
    \item As in Remark~\ref{remark:main-assumptions}(\ref{remark:item:separation}), it holds for any $f \in \mathcal{H}_\mathcal{X}$ that $f = 0 \, \, \mathbb{P}_{X}\text{-a.e.}$ iff.\ $f \equiv 0$ on $\mathcal{X}$.
    \item Condition (iv) is technical and needed to ensure that for any $x \in \mathcal{X}$, the conditional mean embedding
    \begin{equation}
       \mu_{k_{\mathcal Y}}(\mathbb{P}_{Y\mid X = x}) = \mathbb{E}_{Y \mid X =x}[k_{\mathcal{Y}}(\cdot,Y)] \in \mathcal{H}_{\mathcal{Y}}, \label{eq:conditional-mean-lin-alg}
    \end{equation}
    can be realized via linear algebra involving covariance operators \citep[Theorem~4.3]{klebanov20conditional}.
    \item To define~\eqref{eq:conditional-mean-lin-alg}, \citet[(2.4)]{klebanov20conditional} assign $\mu_{k_{\mathcal Y}}(\mathbb{P}_{Y\mid X = x}) \coloneq 0$ for $x\in\mathcal X\setminus\mathcal X_{\mathcal Y}$, with $\mathcal X_{\mathcal Y}\coloneq \{x\in\mathcal X \mid \mathbb{ E}_{Y\mid X=x}\|k_{\mathcal Y}(\cdot,Y)\|_{\mathcal H_{\mathcal Y}}^2 < \infty \}$. In our case $\mathcal X_{\mathcal Y} = \mathcal X$ as the boundedness of $k_{\mathcal Y}$ implies that of $\|k_{\mathcal Y}(\cdot, y)\|_{\mathcal H_{\mathcal Y}}$ for $y\in\mathcal Y$ \citep[p.~124]{steinwart08support}; hence, this distinction is not needed.
\end{enumerate}
\end{remark}

Given an i.i.d.\ sample $(X_i,Y_i)_{i=1}^n$ from $\mathbb P_{XY}$ and $\epsilon_n > 0$, define the $n\times n$ matrices
\begin{equation}
\begin{aligned}
&\mathbf K_X = [k_\mathcal{X}(X_i,X_j)]_{i,j=1}^n, && \tilde{\mathbf K}_X = \mathbf H\mathbf K_X\mathbf H, \\
&\mathbf K_Y = [k_\mathcal{Y}(Y_i,Y_j)]_{i,j=1}^n,\\
&\mathbf M = \mathbf K_Y\tilde{\mathbf K}_X
(\tilde{\mathbf K}_X+n\epsilon_n\mathbf I_n)^{-1}.
\end{aligned}
\label{eq:gram-matrices}
\end{equation}
where $\mathbf H = \mathbf I_n - \frac1n\bm 1_n\bm1_n^\top \in \mathbb R^{n\times n}$ is the centering matrix. Denote the $j$-th canonical basis vector by $\mathbf e_j \in \R^n$. With these notations in place, $\mathbb{E}_X [\mathbb{V}_{Y | X} [ k_\mathcal{Y} (Y,Y_i)]]$ and $\mathbb{V}_Y [k_\mathcal{Y} (Y,Y_i)]$ ($i\in [n]$) can be estimated, respectively, via
\begin{align}
E_{n,i}^{\operatorname{RKHS}}&\coloneq  \frac{1}{n}\big\langle \mathbf e_i,(\mathbf K_Y\!\circ\!\mathbf K_Y)\bm1_n \notag
\\
&\quad+(\mathbf K_Y\!\circ\!\mathbf K_Y)\tilde{\mathbf K}_X
(\tilde{\mathbf K}_X+n\epsilon_n\mathbf I_n)^{-1}\bm1_n \notag \\
&\quad-\frac{1}{n}\mathbf K_Y\bm1_n\bm1_n^\top\mathbf K_Y\mathbf e_i
-\frac{2}{n}\mathbf K_Y\bm1_n\bm1_n^\top\mathbf M^\top\mathbf e_i \notag \\
&\quad-\mathbf M\mathbf M^\top\mathbf e_i
\big\rangle_{\mathbb R^n}, \label{eq:rkhs-expectation-estimator}\\
V_{n,i}^{\operatorname{RKHS}}
&\coloneq \frac{1}{n}\bm1_n^\top(\mathbf K_Y\!\circ\!\mathbf K_Y)\mathbf e_i
-
\left( \frac{1}{n} \bm1_n^\top\mathbf K_Y\mathbf e_i\right)^2. \label{eq:rkhs-var-estimator}
\end{align}
The underlying idea is to use the linear algebraic formulation of \eqref{eq:conditional-mean-lin-alg} to obtain a quantity which then permits plug-in estimation. The first step relies on Assumption~\ref{assumption:general} and Assumption~\ref{ass:rkhs-estimator-assumption}; see the formal justification in Appendix~\ref{sec:derivation-rkhs-estimator}.

Using the above notations, we are ready to introduce the following RKHS-based estimator of $D(Y,X)$. 

\begin{definition}[RKHS-based estimator] \label{definition:ridge_estimator} Given a sample $(X_i,Y_i)_{i=1}^n\overset{\text{i.i.d.}}{\sim}\mathbb P_{XY}$, $\epsilon_{n} > 0$, and $\mathbf K_X$, $\mathbf{K}_Y$, $\tilde{\mathbf K}_X$, and $\mathbf{M}$ as in \eqref{eq:gram-matrices}, let 
\begin{equation*}
\hat{D}^{\operatorname{RKHS}}(Y,X) \coloneq 1 - \frac{1}{n} \sum_{i=1}^n \frac{E_{n,i}^{\operatorname{RKHS}}}{V_{n,i}^{\operatorname{RKHS}}},
\end{equation*}
where $E_{n,i}^{\operatorname{RKHS}}$ is as in \eqref{eq:rkhs-expectation-estimator} and $V_{n,i}^{\operatorname{RKHS}}$ as in \eqref{eq:rkhs-var-estimator}.
\end{definition}

\subsection{Computational Complexity} \label{sec:computational-complexity}
In this section, we establish the runtimes of the nearest-neighbour estimator (Definition~\ref{definition: graph based estimator}) and the RKHS-based estimator (Definition~\ref{definition:ridge_estimator}).

Let us start with the nearest-neighbour estimator. Assume that we are given access to a data structure that allows retrieving the $K$ nearest neighbours of any $X_j$ point in $\mathcal O(\log n)$, for example, a vantage point tree has been set up; the latter costs $\mathcal O(n\log n)$. Then, in the nominator \eqref{equation: KNN sample moments}, the computational cost consists of, for each $j\neq i$, retrieving the $K$ nearest neighbours and performing $\mathcal O(K)$ elementary operations in the inner sum, which adds up to $\mathcal O(n(K+\log n))$. The computation of the denominator \eqref{equation: KNN sample moments1} has a cost of $\mathcal O(n)$. The dominant cost of the nearest neighbour estimator is thus repeating the computations in the nominator $n$ times, which yields a total runtime complexity of $\mathcal O(n^2(K+\log n))$.

For the RKHS-based estimator, one must first compute the matrices in \eqref{eq:gram-matrices}, costing $\mathcal O(n^3)$ by the matrix multiplications in $\mathbf M$ and by the cost of matrix inversion encountered in practice.  Next, to compute the nominator~\eqref{eq:rkhs-expectation-estimator} for $i\in[n]$, notice that the second line in~\eqref{eq:rkhs-expectation-estimator} is independent of $i$; hence, while its computation costs $\mathcal O(n^3)$, it must only be computed once. The remaining operations in~\eqref{eq:rkhs-expectation-estimator} are matrix-vector and vector-vector products, which have a cost of at most $\mathcal O(n^2)$. Similarly, the denominator \eqref{eq:rkhs-var-estimator} has a computational cost of $\mathcal O(n^2)$. Repeating these $n$ times adds up to a total complexity of $\mathcal O(n^3)$ for the RKHS-based estimator.

\section{Consistency and Rate of Convergence}\label{sec:Consistencyandrates}

In this section, we establish the consistency and convergence rate of the nearest-neighbour estimator in Definition~\ref{definition: graph based estimator}, under the following assumptions.
\begin{assumption} \label{assumption: properties of KNN}
    
    Given a sample $\{X_i\}_{i = 1}^n$, let $\delta_\ell\coloneq \abs{\{j:\ell\in\mathcal{N}_j\}}$ for $\ell\in [n]$ where $\mathcal{N}_j$ is the set of the $K$-nearest neighbours of $X_j$ in $\{X_i\}_{i\neq j}$ with ties broken at random. There exists constants $C_{\mathcal{X}} > 0$ such that $\delta_\ell\leq C_{\mathcal{X}} K$ for all $\ell\in[n]$. 
\end{assumption}

\begin{assumption}
\leavevmode
\begin{enumerate}[label=(\roman*)]
        \item \label{assumption: properties of P item: integrable inverse var} There exists a finite constant $C_3 > 0$ such that 
            $\int_{\mathcal{Y}}(\mathbb{V}_Y  (k_\mathcal{Y} (Y, y))^{-3} \mathrm{d}\pp_Y(y) < C_3$.

        \item There exists an $x^* \in \mathcal{X}$, and constants $\alpha,C_1,C_2>0$ such that for any $t>0$
        $\mathbb{P}_X \left( d_{\mathcal{X}}(X_1,x^*)\geq t\right)\leq C_1\exp(-C_2t^{\alpha})$.

        \item There exist constants $d>0$ and $c_0>0$ such that for every radius $T>0$, for all $x\in \mathcal X$ with $d_{\mathcal{X}}\left(x^*, x\right) \leq T$ and all $r>0$, $\pp_X\left(d_{\mathcal{X}}(X_1, x) \leq r\right) \geq c_0 r^d$. 

        \item For $y\in\mathcal{Y}$ and $x\in\mathcal X$, let $m_{1,y}(x)\coloneq \ee_{Y\mid X = x}[k_\mathcal{Y}(Y, y)]$ and $m_{2,y}(x)\coloneq \ee_{Y\mid X = x}[k_\mathcal{Y}^2(Y, y)]$. Then there exists constants $L, \beta > 0$ such that for any $y\in\mathcal Y$ and for any $x,x'\in\mathcal X$ 
            $\abs{m_{1,y}(x) - m_{1,y}(x')} \leq L d_\mathcal{X}^\beta(x, x')$, and 
            $\abs{m_{2,y}(x) - m_{2,y}(x')} \leq L d_\mathcal{X}^\beta(x, x')$.

\end{enumerate}
\label{assumption: properties of P} 
\end{assumption}

We elaborate the assumptions in the following remarks.

\begin{remark}
    Assumption~\ref{assumption: properties of KNN} ensures that replacing $(X_\ell, Y_\ell)$ for any $\ell\in[n]$ by an i.i.d.\ copy $(X_\ell', Y_\ell')$ affects only finitely many terms in $E_{n,i}^{\text{K-NN}}$. 
\end{remark}
\begin{remark}
     Assumption~\ref{assumption: properties of P} part (i) is satisfied for example whenever $\mathcal{Y}$ is a metric space (with distance $d_\mathcal{Y}$) and bounded ($\sup_{y, y'\in\mathcal{Y}} d_{\mathcal{Y}}(y, y') < \infty$). Part (ii) can be interpreted as a tail bound on $X$, while part (iii) asserts that the distribution of $X$ is nowhere too thin, meaning that it is locally $d$-dimensional and hence $d$ can be seen as the intrinsic dimension. Part (iv) assumes that the first and second conditional moments of the canonical feature maps depend smoothly on $x$. This is exactly what ensures that replacing $X_j$ by a nearby $X_k$ introduces only a $\mathcal{O}(d_\mathcal{X}^\beta(X_j, X_k))$ bias for the $K$-NN estimator. As noted in \cite[Lemma 4]{azadkia2025} and \cite[Section 5]{deb20measuringassociationtopologicalspaces}, without regularity conditions of this type, the convergence rate may be arbitrarily slow. Comparable smoothness assumptions in related settings are imposed in \cite{asgupta24optknn}, \cite{bhattacharya19asymptoticframework}, and \cite{deb20measuringassociationtopologicalspaces}. 
\end{remark}

Under these assumptions, we obtain the following result on the rate of convergence of the nearest neighbour estimator.

\begin{theorem}\label{theorem: rate of convergence}
Under  Assumptions~\ref{assumption:general},~\ref{assumption: X is metric space},~\ref{assumption: properties of KNN}, and~\ref{assumption: properties of P}, 
\begin{align*}
    &\abs{\hat{D}^{\text{K-NN}} (Y, X) - D (Y, X)} = \\
    &\qquad\mathcal{O}_\mathbb{P}\!\left(\frac{1}{\sqrt{n}}+\left(\frac{K}{n}\right)^{\beta/d}+\frac{(\log n)^{\beta / \alpha}}{n^{2}}\right). 
\end{align*}
\end{theorem}

Theorem~\ref{theorem: rate of convergence} shows that the rate of convergence of $\hat{D}^{\text{K-NN}} (Y,X)$ adapts to the intrinsic dimension $d$ of $X$. For instance, the result guarantees, for suitable $\beta$, faster convergence if $\mathcal X = \R^{d_0}$ but $X$ is only supported on a $d < d_0$ dimensional hyperplane.

\section{Numerical Illustrations}\label{sec:numerical}

In this section, we apply our proposed estimators to  simulated and real-world datasets.\footnote{The implementation is provided in \href{https://github.com/PouyaRoudaki/KernelIR}{KernelIR}.} In real-world data analysis, we also normalize each
real-valued variable to have mean $0$ and variance $1$.

\subsection{Simulation Studies} \label{sec:simulation-studies}
\paragraph{Power comparison.}
We assess the power of independence tests based on $\hat{D}^{\operatorname{RKHS}}$ and
$\hat{D}^{\text{K-NN}}$, benchmarking them against several competitive tests
based on the related dependence measures discussed in
Section~\ref{sec:RelatedDependenceMeasures}. The competing statistics are: Chatterjee's
$\xi_n$ (estimator of $\xi$) correlation coefficient~\cite{chatterjee21newcoefficient}, the $T_n$ (estimator of $\xi$ for the case where $X$ can be multidimensional) statistic~\cite{azadkia2021simple}, the integrated $R^2$ dependence measure
$\nu_n$ (estimator for $\nu$)~\cite{azadkia2025}, and the kernel measure of association in both its
$K$-nearest-neighbour form $\hat\eta^{\text{K-NN}}$ and its RKHS form
$\hat\eta^{\operatorname{RKHS}}$ (estimators for $\eta_{k_\mathcal{Y}}$)~\cite{huang22kernel}. These are computed using the
\texttt{R} packages \texttt{XICOR}, \texttt{FOCI}, \texttt{FORD}, and \texttt{KPC},
respectively. The sample size is $n = 100$; all $p$-values are obtained via $1000$ independent permutations, and power is estimated from $500$ simulation replications at the $5\%$ significance level.

\paragraph{Euclidean data.} We draw $n$ i.i.d.\ samples $(X_i, Y_i)_{i=1}^n$ from a distribution on
$\mathbb{R}^2$, following the experimental setup of \citet[Example~6.4(6)]{chatterjee21newcoefficient}.
Indeed, $X \sim \mathrm{Uniform}[-1,1]$, the noise
$\varepsilon \sim N(0,1)$ is independent of $X$, and the alternative hypothesis is 
$Y = 3\bigl(\sigma(X)(1-\lambda) + \lambda\bigr)\varepsilon$,
where $\sigma(X) = \mathbbm{1}_{\{\lvert X\rvert \leq 0.5\}}$ and the noise level $\lambda \in [0,1]$, that is, we consider a heteroscedastic setting.
The statistics $\hat{\eta}^{\text{K-NN}}$,
$\hat{\eta}^{\operatorname{RKHS}}$, $\hat{D}^{\text{K-NN}}$, and
$\hat{D}^{\operatorname{RKHS}}$ all use the Gaussian kernel with the median of pairwise distances as bandwidth, and their regularization is $\epsilon_n = 10^{-4}$. The graph-based methods use
$5$ nearest neighbours.

\begin{figure}[h]
  \centering
  \input{fig/tikz_combined/power_final}
  \caption{Comparison of power of independence tests for the heteroscedastic 
  (left) and SO(3) (right) alternatives as a function of  homoscedasticity (left) and the level of noise (right). }
  \label{fig:power_main_text}
\end{figure}
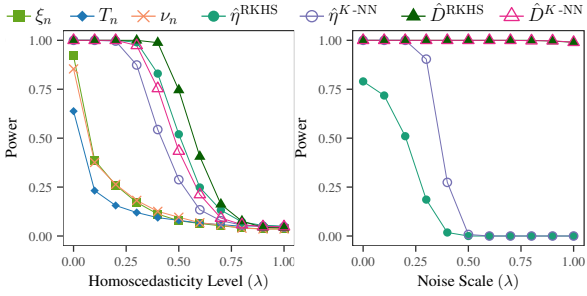

\paragraph{Non-Euclidean data.}
Following \citet[Section~6]{huang22kernel}, we additionally consider the setting where $Y$
takes values in the special orthogonal group $\mathrm{SO}(3)$, the manifold of $3\times
3$ orthogonal matrices with determinant $1$. We equip $\mathrm{SO}(3)$ with the
characteristic kernel
$
k_\mathcal{Y}(\mathbf{A},\mathbf{B}) = \frac{\pi\theta(\pi-\theta)}{8\sin\theta}$,
where $\theta \in [0,\pi]$ is defined by $\cos\theta = (\operatorname{Tr}(\mathbf{B}^{-1}\mathbf{A})-1)/2$,
so that $e^{\pm \sqrt{-1}\theta}$ are the eigenvalues of $\mathbf{B}^{-1}\mathbf{A}$. Let $R_1(x)$ and
$R_3(z)$ denote rotations in the $y-z$ plane and $x-y$ plane with angle $x$ and $z$, respectively.
The predictor is $X \sim N(\mathbf{0}, \mathbf{I}_3)$. Let the independent noise
variables be $\varepsilon_1, \varepsilon_2 \overset{\text{i.i.d.}}{\sim} N(0,1)$, the noise scale
$\lambda \in [0,1]$, and the the alternative hypothesis is 
$Y = R_1(X_1 + \lambda\varepsilon_1)R_3(X_2 X_3 +
\lambda\varepsilon_2) \in \operatorname{SO}(3)$.
For $\hat{D}^{\operatorname{RKHS}}$ the kernel $k_\mathcal{X}$ is the Gaussian kernel with the median of pairwise distances as bandwidth, and their regularization parameter is set to $\epsilon_n = 10^{-4}$. Moreover, $\hat{D}^{\text{K-NN}}$ uses
$5$ nearest neighbours.

The results in Figure~\ref{fig:power_main_text} show that $\hat{D}^{\text{K-NN}}$
and $\hat{D}^{\operatorname{RKHS}}$ consistently outperform all competing tests in both
the Euclidean and non-Euclidean examples. The kernel-based measures
$\hat{\eta}^{\text{K-NN}}$ and $\hat{\eta}^{\operatorname{RKHS}}$ already improve
upon non-kernel methods in the heteroscedastic case---owing to their greater flexibility,
as discussed in Section~\ref{sec:RelatedDependenceMeasures}---and naturally extend to
non-Euclidean data. Nevertheless, $\hat{D}^{\text{K-NN}}$ and
$\hat{D}^{\operatorname{RKHS}}$ achieve uniformly higher power while retaining the same
generality.

\paragraph{Runtime comparison.} In this experiment, we compare the runtimes of $\hat{D}^{\text{K-NN}}$ and $\hat{D}^{\operatorname{RKHS}}$ with that of closely related measures (elaborated in  Section~\ref{sec:RelatedDependenceMeasures}). The right plot in Figure~\ref{fig:MSD_time} shows the average runtime of each estimator over $100$ repetitions for sample size $n$ ranging from $10$ to $5000$. We observe the higher runtime of RKHS-based methods compared to graph-based methods, in line with Section~\ref{sec:computational-complexity}.

\subsection{Real Data Example}

\paragraph{Million Song Dataset}
The Million Song Dataset~\citep{bertin11million} contains $515{,}345$ songs. Each song is described by $90$ features $X$ and its year of release $Y$; the latter ranges from 1992 to 2011. The objective is to detect statistical dependence between the features and the release year. To assess empirical power as a function of the sample size (ranging from $50$ to $1500$), we proceed as follows: for each fixed sample size $n$, we draw $200$ independent subsamples from the full dataset and estimate the power at significance level $0.01$. The p-values are computed using $200$ permutations. For the RKHS-based estimators, we use the Gaussian kernel with the median of pairwise distances as the bandwidth, and the regularization parameter $\epsilon_n = 10^{-4}$ (kept fixed for simplicity). For the graph-based estimators, we consider $5$ nearest neighbours. The kernel-based methods are generally expected to provide greater flexibility in capturing complex dependence structures and to achieve higher power as it is showed in left plot of Figure~\ref{fig:MSD_time}; we also observe in this example the RKHS-based estimator outperforms the graph-based approach.
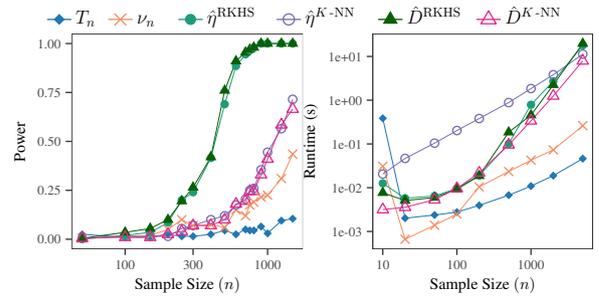
\begin{figure}[h]
  \centering
\input{fig/tikz_combined/power_MSD_time_final}
  \caption{(Left) Comparison of power of independence tests for Million Songs Data as a function of the sample size ($n$). (Right) Comparison of time complexity of the compared dependence measures as a function of the sample size ($n$).}
  \label{fig:MSD_time}
\end{figure}

\begin{acknowledgements} %

   FK is supported by the pilot program Core-Informatics of the Helmholtz Association (HGF).
   This work used the Cirrus UK National Tier-2 HPC Service at EPCC (http://www.cirrus.ac.uk) funded by The University of Edinburgh, the Edinburgh and South East Scotland City Region Deal, and UKRI via EPSRC.
\end{acknowledgements}

\bibliography{bib/collected_Florian,bib/collected_plus}

\newpage

\onecolumn

\title{Kernel Integrated $R^2$: A Measure of Dependence\\(Supplementary Material)}
\maketitle

\appendix
\setcounter{equation}{0}
\renewcommand{\theequation}{\thesection.\arabic{equation}} %

\section{Proofs}\label{sec:proof}

This appendix collects our proofs. We prove Theorem~\ref{theorem: rho properties} in Appendix~\ref{sec:proof-key-properties-kir}, derive the RKHS-based estimator (Definition~\ref{definition:ridge_estimator}) in Appendix~\ref{sec:derivation-rkhs-estimator}, and prove Theorem~\ref{theorem: rate of convergence} in Appendix~\ref{sec:proof-rates-knn}.

\subsection{Proof of Theorem \ref{theorem: rho properties}} \label{sec:proof-key-properties-kir}

We prove the well-definedness and parts (i)--(iii) one by one.

\begin{proof}[Proof of well-definedness] By Assumption~\ref{assumption:general}, the kernel $k_\mathcal{Y}$ is continuous and by the blanket assumption in Section~\ref{sec:notation}, it is also bounded. Hence, $k_\mathcal{Y}(\cdot,y)$ is Lebesgue-integrable for every $y\in \mathcal Y$ and so is its square; therefore the nominator is finite. It remains to show that the denominator
    $\mathbb{V}_Y \!\left[k_\mathcal{Y}(Y,y)\right]$
is non-zero for all $y\in \mathcal Y$, which holds if $k_{\mathcal Y}(Y,y)$ is not a.s.\ constant. We prove this in the following.

Indeed, it is known that a random variable is a.s.\ constant iff.\ its distribution is degenerate. Hence, it suffices to show that the distribution of $k_{\mathcal Y}(Y,y)$ is non-degenerate for all $y\in\mathcal Y$. Let us argue by contradiction, in other words, there exists a $y\in\mathcal Y$ such that the distribution of $k_{\mathcal Y}(Y,y)$ is degenerate, that is, for some $r\in\mathcal \R$ it holds that $\mathbb P_Y\circ k_{\mathcal Y}^{-1}(\cdot,y) = \delta_{r}$. 
Then there exists $u\in\mathcal Y$ such that $k_{\mathcal Y}(u,y) = r$; indeed, assuming that $k_{\mathcal Y}^{-1}(\{r\},y) = \emptyset$ would mean that $\mathbb P_{\mathcal Y} \circ k_{\mathcal Y}^{-1}(\{r\},y) = 0 \neq \delta_{r}(\{r\})= 1$, which is a contradiction. Moreover, as by Assumption~\ref{assumption:general}, $k_{\mathcal Y}(\cdot,y)$ is non-constant, there exist $v\in\mathcal Y$ and $s\in\R$ such that $k_{\mathcal Y}(v,y) \eqcolon s \neq r$. As $\R$ is Hausdorff and $r\neq s$, we can find disjoint open sets $R,S\subset \R$ such that $r\in R$ and $s\in S$. As $k_{\mathcal Y}$ is continuous, so is $k_{\mathcal Y}(\cdot,y)$ \citep[Lemma~4.29]{steinwart08support}, which implies that also $U \coloneq k_{\mathcal Y}^{-1}(S,y) \subset \mathcal Y$ is open. Hence, by the assumed full support of $\mathbb P_{\mathcal Y}$, $\mathbb P_{\mathcal Y}(U) > 0$. But then $\mathbb{P}_{\mathcal Y} \circ k_{\mathcal Y}^{-1}(S,y) > 0$ while $\delta_r(S) = 0$, contradicting the assumption that $\mathbb{P}_{\mathcal Y} \circ k_{\mathcal Y}^{-1}(\cdot,y) =\delta_r$.
\end{proof}

\begin{proof}[Proof of (i)] 
     By the law of total variance and the non-negativity of variances, we have
    \begin{align*}
         \mathbb{V}_Y \!\left[k_\mathcal{Y}(Y,y)\right] &= \mathbb{V}_X \!  \left  [ \ee_{Y \mid X}\!\left[k_\mathcal{Y}(Y,y)\right] \right ] + \underbrace{\ee_{X} \!  \left  [ \mathbb{V}_{Y \mid X}\!\left[k_\mathcal{Y}(Y,y)\right] \right ]}_{\ge 0}
         \ge \mathbb{V}_X\!  \left  [ \ee_{Y\mid X}\!\left[k_\mathcal{Y}(Y,y) \right] \right ] \ge 0,
    \end{align*}
    implying that 
    \begin{equation}
         \mathbb{V}_Y \!\left[k_\mathcal{Y}(Y,y)\right] \ge \mathbb{V}_X\!  \left  [ \ee_{Y\mid X}\!\left[k_\mathcal{Y}(Y,y) \mid X\right] \right ] \ge 0.
        \label{eq:var_ineq}
    \end{equation}
    The combination of \eqref{eq:var_ineq} and the alternative expression \eqref{eq:kmd-population-version2} in Lemma~\ref{prop:kmd-population-version2} shows that $D(Y,X) \in [0,1]$.
\end{proof}

\begin{proof}[Proof of (ii)]
    $(\impliedby)$ We will show that in this case the numerator in \eqref{eq:kmd-population-version2} is zero, which implies that \eqref{eq:kmd-population-version2} itself is zero. The claim then follows by the equivalence of \eqref{eq:kmd-population-version1} and \eqref{eq:kmd-population-version2}, established in Lemma~\ref{prop:kmd-population-version2}.

    Indeed, by the assumed independence of $X$ and $Y$, for any $y\in\mathcal Y$,
    \begin{equation*}
            \ee_{Y \mid X}\!\left[k_\mathcal{Y}(Y,y) \right] = \ee_Y\!\left[k_\mathcal{Y}(Y,y)\right] \eqcolon g(y),
    \end{equation*}
    and the numerator of \eqref{eq:kmd-population-version2} becomes
    \begin{equation*}
        \mathbb{V}_X\!\left[\ee_{Y \mid X}\!\left[k_\mathcal{Y}(Y,y)\right]\right] = \mathbb{V}_X[g(y)] = 0,
    \end{equation*}
    proving the first direction.
    
    $(\implies)$ Assume that $D(Y,X) = 0$. Then, using the equivalence of \eqref{eq:kmd-population-version1} and \eqref{eq:kmd-population-version2}, we have that
    \begin{equation*}
    \mathbb{V}_X\!\big[\ee_{Y\mid X}[k_\mathcal{Y}(Y,y)]\big] = 0 \text{ for } \mathbb P_Y\text{-a.e. } y,
    \end{equation*}
    as the integrand of \eqref{eq:kmd-population-version2} is non-negative. Hence, for $\mathbb P_Y$-almost every $y$, $\ee_{Y\mid X}[k_\mathcal{Y}(Y,y)]$ is $\mathbb P_X$-almost surely constant, that is, there exists $c(y)\in \mathbb R$ such that
    \begin{equation}
        c(y)  = \ee_{Y\mid X}[k_\mathcal{Y}(Y,y)] \quad  \mathbb P_X\text{-a.s.\ for } \mathbb{P}_Y \text{-a.e. } y. \label{eq:f_y = c_y}
    \end{equation}
    Integrating the last expression w.r.t.\ $\mathbb P_X$ and using the tower property of expectations, we get
     \begin{equation}
        c(y) = \mathbb E_X\!\left[ \mathbb E_{Y \mid X}[k_\mathcal{Y}(Y,y)] \right] = \mathbb E_Y[k_\mathcal{Y}(Y,y)] = \mu_{k_\mathcal{Y}}(\mathbb P_Y)(y),
    \label{eq:c(y)}
    \end{equation}
    for $\mathbb P_Y$-almost every $y$. Notice that  \eqref{eq:f_y = c_y} can be written as
    \begin{align}
        \ee_{Y\mid X}[k_\mathcal{Y}(Y,y)] = \ee_{Y\mid X}[\fip{k_{\mathcal Y}(\cdot, Y),k_{\mathcal{Y}}(\cdot,y)}{\mathcal H_{\mathcal Y}}] =  \fip{\ee_{Y\mid X}[k_{\mathcal Y}(\cdot, Y)],k_{\mathcal{Y}}(\cdot,y)}{\mathcal H_{\mathcal Y}} = 
        \mu_{k_\mathcal{Y}}(\mathbb P_{Y\mid X})(y), \label{eq:conditoinal-mean-at-y}
    \end{align} by the reproducing property, \citet[(A.32)]{steinwart08support}, and the definition of conditional mean embeddings with the reproducing property. As the l.h.s.\ of \eqref{eq:f_y = c_y} and \eqref{eq:c(y)} coincide, so do the r.h.s.\ of \eqref{eq:f_y = c_y} and \eqref{eq:conditoinal-mean-at-y}, which shows that
    \begin{equation*}
        \mu_{k_\mathcal{Y}}(\mathbb P_{Y\mid X})(y) = \mu_{k_\mathcal{Y}}(\mathbb P_Y)(y)  \quad  \mathbb P_X\text{-a.s.\ for } \mathbb{P}_Y \text{-a.e. } y,
    \end{equation*}
    which, by Remark~\ref{remark:main-assumptions}(\ref{remark:item:separation}), implies that $\mu_{k_\mathcal{Y}}(\mathbb P_{Y\mid X}) = \mu_{k_\mathcal{Y}}(\mathbb P_Y)$ holds $\mathbb P_X$-a.s. (as $\mu_{k_\mathcal{Y}}(\mathbb P_{Y\mid X}) - \mu_{k_\mathcal{Y}}(\mathbb P_Y) \in \mathcal H_{k_\mathcal{Y}}$ and $\mu_{k_\mathcal{Y}}(\mathbb P_{Y\mid X})(y) - \mu_{k_\mathcal{Y}}(\mathbb P_Y)(y) = 0$ holds $\mathbb{P}_X$-a.s.\ for $\mathbb P_Y$-a.e.\ $y$). Thus, using that $k_\mathcal{Y}$ is characteristic by Assumption~\ref{assumption:general}, 
    \begin{align}
        \mathbb P_{Y\mid X} = \mathbb{P}_Y  \quad \mathbb P_X\text{-a.s.} \label{eq:equal-px-as.}
    \end{align}

    To conclude the proof of (ii), we now show that \eqref{eq:equal-px-as.} implies the independence of $X$ and $Y$. Indeed, let $A \coloneqq B\times C \in \mathcal B(\mathcal X)\otimes \mathcal B(\mathcal Y)$ be arbitrary, where $ \mathcal B(\mathcal X)\otimes \mathcal B(\mathcal Y)$ denotes the product sigma-algebra of $\mathcal B(\mathcal X)$ and $\mathcal B(\mathcal Y)$. Then $A$, $B$, and $C$ are measurable w.r.t.\ $\mathbb P_{XY}$, $\mathbb P_X$, and $\mathbb P_Y$, respectively. Using that $\mathbbm{1}_A = \mathbbm{1}_B\mathbbm{1}_C$, by the decomposition in \citet[Theorem~10.2.1]{dudley04real}, we have
    \begin{align*}
        \mathbb{P}_{XY}(A) &= \int_{\mathcal X \times \mathcal Y} \mathbbm{1}_A(x,y) \mathrm d \mathbb P_{XY}(x,y) = \int_{\mathcal X}\int_{\mathcal Y}\mathbbm{1}_A(x,y)\mathrm d\mathbb P_{Y\mid X=x}(y)\mathrm d \mathbb{P}_{X}(x) \\
        &\overset{\eqref{eq:equal-px-as.}}{=} \int_{\mathcal X}\int_{\mathcal Y}\mathbbm{1}_A(x,y)\mathrm d\mathbb P_{Y}(y)\mathrm d \mathbb{P}_{X}(x) = \int_{\mathcal X}\int_{\mathcal Y}\mathbbm{1}_B(x)\mathbbm{1}_C(y)\mathrm d\mathbb P_{Y}(y)\mathrm d \mathbb{P}_{X}(x) \\
        &=\int_{\mathcal X}\mathbbm{1}_B(x)\mathrm d \mathbb{P}_{X}(x)\int_{\mathcal Y}\mathbbm{1}_C(y)\mathrm d\mathbb P_{Y}(y) = \mathbb P_X(B)\mathbb P_Y(C),
    \end{align*}
    that is, the joint distribution of $(X,Y)$ factorizes to the product of the marginals, showing the independence of $X$ and $Y$.
\end{proof}

\begin{proof}[Proof of (iii)]
$(\impliedby)$ Suppose that $Y = f(X)$ for some Borel measurable function $f : \mathcal X \to \mathcal Y$. Then, for any fixed $y\in \mathcal Y$, $k_{\mathcal Y}(Y,y) = k_{\mathcal Y}(f(X),y)$ is a Borel measurable function of $X$ as the composition of measurable functions is a measurable function. Then for any $y\in\mathcal Y$, by the definition of the conditional variance in (a) and the properties of conditional expectations in (b)
\begin{equation*}
    \mathbb{V}_{Y \mid X} (k_\mathcal{Y}(Y,y)) \overset{(a)}{=} \mathbb{E}_{Y\mid X}[k_\mathcal{Y}^2(Y,y)] - (\mathbb{E}_{Y\mid X}[k_\mathcal{Y}(Y,y)])^2 \overset{(b)}{=} k_\mathcal{Y}^2(Y,y) - k_\mathcal{Y}^2(Y,y) = 0 \text{ holds } \mathbb{P}_X\text{-a.s.}
\end{equation*}
hence,
\begin{equation*}
    \ee_X[\mathbb{V}_{Y \mid X} (k_\mathcal{Y}(Y,y))] = \ee_X[0] = 0.
\end{equation*}
As $y\in\mathcal Y$ was arbitrary, the numerator is zero everywhere, which implies that
\begin{align*}
D(Y,X) = 1 - \int_\mathcal{Y} \frac{\mathbb E_X\! \left  [ \mathbb{V}_{Y\mid X} \!\left(k_\mathcal{Y}(Y,y) \right) \right ]}{\mathbb{V}_Y \!\left(k_\mathcal{Y}(Y,y)\right)}\mathrm d  \mathbb P_Y(y) = 1.
\end{align*}

$(\implies)$ Our goal is to show that if $D(Y,X)=1$, then $Y\mid X = x$ is a.s.\ constant for a.e.\ $x\in\mathcal X$. This implies that $\mathbb P_{Y\mid X = x}$ is degenerate for a.e.\ $x\in \mathcal X$ and an application of Lemma~\ref{lemma:measurable-if-degenerate} then yields the claim.

Indeed,  assume that $D(Y,X)=1$ and let us show that then $Y\mid X = x$ is a.s.\ constant for a.e.\ $x\in\mathcal X$. Rearranging $D$ by using linearity of the integral, and flipping the integrals (permitted by Tonelli's theorem as all terms are non-negative), we obtain
\begin{equation*}
    D(Y,X) = 1 -  \int_\mathcal{Y} \frac{ \int_{\mathcal X}\mathbb{V}_{Y\mid X = x} \!\left[k_\mathcal{Y}(Y,y)  \right ]\mathrm d \mathbb P(x)}{\mathbb{V}_Y \!\left[k_\mathcal{Y}(Y,y)\right]}\mathrm d  \mathbb P_Y(y) 
    = 1 - \int_{\mathcal X} \int_\mathcal{Y} \frac{ \mathbb{V}_{Y\mid X = x} \!\left[k_\mathcal{Y}(Y,y)  \right ]}{\mathbb{V}_Y \!\left[k_\mathcal{Y}(Y,y)\right]}\mathrm d  \mathbb P_Y(y) \mathrm d \mathbb P(x) \overset{\text{(Ass.)}}{=} 1,
\end{equation*}
which shows that, given $\mathbb{P}_X$ almost any $x\in\mathcal X$, $\mathbb{V}_{Y\mid X = x} \!\left[k_\mathcal{Y}(Y,y)  \right ] = 0$ for $\mathbb P_Y$-a.e.\ $y$. As the variance is zero, it must hold that 
\begin{equation*}
    \mathbb P_{Y\mid X = x}(k_{\mathcal Y}(Y,y) = c_x(y)) = 1 \text{ for } \mathbb P_Y\text{-a.e.\ } y \text{ and for } \mathbb P_X\text{-a.e.\ } x,
\end{equation*}
where $c_x(y)$ denotes a constant depending on $y \in \mathcal Y$ and $x\in\mathcal X$. Notice that for $\mathbb{P}_X$ almost any $x\in\mathcal X$  this defines a function $c_x : \mathcal Y \to \R$ for $\mathbb P_Y$-a.e.\ $y$.

With this setup in place, we now show that $Y\mid X = x$ is a.s.\ constant for a.e.\ $x\in\mathcal X$. Let us argue by contradiction, that is, suppose that with positive $\mathbb{P}_X$ probability there exists $x\in\mathcal X$ such that  $Y\mid X = x$ is not a.s.\ constant. Then there exist distinct $\omega_1,\omega_2$ in the preimage of $Y\mid X = x$ satisfying $y_1 \coloneq (Y\mid X = x)(\omega_1) \neq (Y\mid X = x)(\omega_2) \eqcolon y_2$ where $y_1$ and $y_2$ are elements in sets with positive $\mathbb P_{Y\mid X=x}$-probability. But then,
\begin{align*}
    k_{\mathcal Y}(y_1,y) = c_x(y) \text{ for } \mathbb P_Y\text{-a.e.\ } y, && \text{while also} &&  k_{\mathcal Y}(y_2,y) = c_x(y) \text{ for } \mathbb P_Y\text{-a.e.\ } y,
\end{align*}
that is $k_{\mathcal Y}(y_1,y) = k_{\mathcal Y}(y_2,y)$ for $\mathbb P_Y$-a.e.\ $y$. Using the assumed continuity of $k_{\mathcal Y}$ and the full support of $\mathbb{P}_Y$, by Remark~\ref{remark:main-assumptions}(\ref{remark:item:separation}), we obtain $k_{\mathcal Y}(y_1,\cdot) = k_{\mathcal Y}(y_2,\cdot)$, contradicting the point-separating property of the characteristic $k_{\mathcal Y}$. Hence, $Y\mid X = x$ is a.s.\ constant for a.e.\ $x\in\mathcal X$.

As indicated in the beginning of the proof of this direction, $Y \mid X=x$ being a.s.\ constant for a.e.\ $x\in\mathcal X$ implies that $\mathbb P_{Y\mid X=x}$ is degenerate for a.e.\ $x\in\mathcal X$. An application of Lemma~\ref{lemma:measurable-if-degenerate} concludes the proof.
\end{proof}

\subsection{Derivation of the RKHS Estimator in Definition~\ref{definition:ridge_estimator}} \label{sec:derivation-rkhs-estimator}

We start by introducing the additional notations and background used in this section only.

For a kernel $k : \mathcal Z \times \mathcal Z \to \mathbb{R}$ and a sample $Z = (Z_1,\ldots,Z_n) \in \mathcal Z^n$ with $Z_i \sim \mathbb P_Z \in \mathcal M_1^+(\mathcal Z)$ ($i\in[n]$), denote by $S_{k,Z} : \mathcal H_k \to \mathbb{R}^n$ the sampling operator, which is defined by $h \mapsto (h(Z_i))_{i=1}^n$. It has adjoint $S_{k,Z}^* : \mathbb{R}^n \to \mathcal H_k$, $\bm \alpha = (\alpha_i)_{i=1}^n \mapsto \sum_{i=1}^n \alpha_i k(\cdot, Z_i)$ \citep{smale07learningtheory}. Furthermore, it is known that $S_{k,Z}S_{k,Z}^* = \left[k(Z_i,Z_j)\right]_{i,j=1}^n = \mathbf K_Z \in \mathbb{R}^{n\times n}$.
The centered (cross-)covariance operator associated to $Y$ and $X$ (as defined in the main part) is given by
\begin{equation*}
    C_{YX} \coloneq \int_{\mathcal X \times \mathcal Y} [k_{\mathcal Y}(\cdot,y) - \mu_{k_{\mathcal Y}}(\mathbb{P}_Y)]\otimes [k_{\mathcal X}(\cdot,x) - \mu_{k_{\mathcal X}}(\mathbb{P}_X)]\mathrm d \mathbb P_{XY}(x,y) \in \mathcal H_{\mathcal Y} \otimes \mathcal H_{\mathcal X},
\end{equation*}
where, for $f \in \mathcal H_{\mathcal Y}$ and $g\in\mathcal H_{\mathcal X}$, $f\otimes g : \mathcal H_{\mathcal X} \to \mathcal H_{\mathcal Y}$ denotes the rank-one operator defined by $h \mapsto f\langle g,h\rangle_{\mathcal H_{\mathcal X}}$; the operator is an element of the tensor product RKHS $\mathcal H_{\mathcal Y} \otimes \mathcal H_{\mathcal X}$. The centered covariance of $X$ is its centered cross-covariance with itself, that is,
\begin{equation*}
    C_X \coloneq C_{XX} = \int_{\mathcal X}[k_{\mathcal X}(\cdot,x) - \mu_{k_{\mathcal X}}(\mathbb{P}_X)]\otimes [k_{\mathcal X}(\cdot,x) - \mu_{k_{\mathcal X}}(\mathbb{P}_X)]\mathrm d \mathbb P_X(x) \in \mathcal H_{\mathcal X} \otimes \mathcal H_{\mathcal X}.
\end{equation*}
For a bounded linear operator $A$, we write $A^-$ for its (Moore-Penrose) pseudo-inverse; see, for example, \citet[Definition~2.2]{engl96regularization}.

With the notation established, let us present the derivation. We tackle the denominator and the numerator separately.
\begin{itemize}
    \item For the \textbf{denominator}, using that 
        $\mathbb{V}_Y\!\left[k_\mathcal{Y}(Y,Y_i)\right] = \mathbb E_Y \!\left[k_\mathcal{Y}^2(Y,Y_i)\right] - \left( \mathbb E_Y \!\left[k_{\mathcal{Y}}(Y,Y_i)\right]\right)^2$ and replacing all expectations with their empirical counterparts, we obtain that
    \begin{equation}
        \mathbb{V}_Y\!\left[k_\mathcal{Y}(Y,Y_i)\right] \approx \frac1n\sum_{j=1}^nk_\mathcal{Y}^2(Y_j,Y_i)-\left(\frac1n \sum_{j=1}^nk_\mathcal{Y}(Y_j,Y_i)\right)^2 = \frac1n \bm1_n^\top (\mathbf K_Y\circ \mathbf K_Y)\mathbf e_i - \left(\frac1n \bm1_n^\top \mathbf K_Y \mathbf e_i\right)^2 = V_{n,i}^{\operatorname{RKHS}}.
        \label{eq:RKHS_num}
    \end{equation}
    
    \item To derive the expression for the \textbf{numerator}, we first observe that given Assumption~\ref{assumption:general} and Assumption~\ref{ass:rkhs-estimator-assumption}, by \citet[Theorem~4.3]{klebanov20conditional}, for $\mathbb P_X$-a.e. $x \in \mathcal X$
    \begin{equation*}
        \mu_{k_{\mathcal Y}}\!\left(\mathbb P_{Y\mid X=x}\right) = \mu_{k_{\mathcal Y}}\!\left(\mathbb P_Y\right) + C_{YX}C_{X}^-\!\left(k_{\mathcal X}(\cdot,x)-\mu_{ k_{\mathcal X}}\!\left(\mathbb P_X\right)\right).
    \end{equation*}
    Using the plug-in estimator and the definition of $S_{k_{\mathcal Y},Y}^*$ for the first term and
     \citet[p.~48(bottom)]{huang22kernel} for the second term, we obtain, for $X=X_j$ ($j\in[n]$), the estimator
    \begin{equation}
        \hat \mu_{k_{\mathcal Y}}\!\left(\mathbb P_{Y\mid X=X_j}\right) \coloneq \frac1nS_{k_{\mathcal Y},Y}^*\bm1_n + S_{k_{\mathcal Y},Y}^*\tilde {\mathbf K}_X\!\left( \tilde{\mathbf K}_X+n\epsilon_n \mathbf I_n\right)^{-1}\mathbf e_j. \label{eq:conditional-mean-estimator}
    \end{equation}
    Let $\hat {\mathbb{P}}_{X,n} = \frac1n\sum_{i=1}^n\delta_{X_i}$ be the empirical measure associated to the observed $X_i$-s. Coming back to the quantity which we want to estimate, notice that, for any $i,j\in[n]$, we have
    \begin{equation*}
        \mathbb{V}_{Y\mid X = X_j}\!\left[k_\mathcal{Y}(Y,Y_i)\right] = \underbrace{\mathbb E_{Y\mid X=X_j}\!\left[k_\mathcal{Y}^2(Y,Y_i) \right]}_{\eqcolon t_1} - \big(\underbrace{\mathbb E_{Y\mid X=X_j}\!\left[k_\mathcal{Y}(Y,Y_i) \right]}_{\eqcolon t_2}\big)^2,
    \end{equation*}
    and will therefore estimate
    \begin{equation*}
        \mathbb{E}_{\hat {\mathbb{P}}_{X,n}}\!\left[\mathbb{V}_{Y\mid X = X_j}\!\left[k_\mathcal{Y}(Y,Y_i)\right]\right] = \mathbb{E}_{\hat {\mathbb{P}}_{X,n}}[t_1] - \mathbb{E}_{\hat {\mathbb{P}}_{X,n}}\big[t_2^2\big];
    \end{equation*}
    we will use \eqref{eq:conditional-mean-estimator} to approximate $t_1$ and $t_2$, respectively. Having obtained these approximations, we approximate the expectation of $t_1$ (resp.\ the expectation of $t_2^2$) by plug-in estimation.
    \begin{itemize}
        \item \textbf{Term $t_1$.} One has for $i,j \in[n]$ that
    \begin{align*}
         \mathbb E_{Y\mid X=X_j}\!\left[k_\mathcal{Y}^2(Y,Y_i) \right] &\overset{(a)}{=}  \mathbb E_{Y\mid X=X_j}\!\left[\fip{ k_{\mathcal Y}(\cdot,Y),k_{\mathcal Y}(\cdot,Y_i)}{\mathcal H_{\mathcal Y}}^2\right] \\
         &\overset{(b)}{=}\mathbb E_{Y\mid X=X_j}\!\left[\fip{k_{\mathcal Y}(\cdot,Y)\otimes k_{\mathcal Y}(\cdot,Y), k_{\mathcal Y}(\cdot,Y_i)\otimes k_{\mathcal Y}(\cdot,Y_i)}{\mathcal H_{\mathcal Y}\otimes \mathcal H_{\mathcal Y}}\right] \\
         &\overset{(c)}{=}\fip{\mathbb E_{Y\mid X=X_j}\!\left[k_{\mathcal Y}(\cdot,Y) \otimes k_{\mathcal Y}(\cdot,Y)\right],k_\mathcal{Y}(\cdot,Y_i)\otimes k_\mathcal{Y}(\cdot,Y_i)}{\mathcal{H}_{\mathcal Y}\otimes \mathcal{H}_{\mathcal Y}}\\
         &\overset{(d)}{=}\fip{\mu_{k_{\mathcal Y} \otimes k_{\mathcal Y}}\!\left(\mathbb P_{Y\mid X=X_j}\right),k_\mathcal{Y}(\cdot,Y_i)\otimes k_\mathcal{Y}(\cdot,Y_i)}{\mathcal{H}_{\mathcal Y}\otimes \mathcal{H}_{\mathcal Y}},
    \end{align*}
    where the reproducing property implies (a), the properties of tensor products yield (b), and \citet[(A.32)]{steinwart08support} allows flipping the expectation and inner product in (c). We apply the definition of the conditional mean embedding with the product kernel in (d); it is characteristic by Assumption~\ref{ass:rkhs-estimator-assumption} and continuous, which allows the estimation by adapting \eqref{eq:conditional-mean-estimator}. 
    
    Indeed, replacing $\mu_{k_{\mathcal Y} \otimes k_{\mathcal Y}}\!\left(\mathbb P_{Y\mid X=X_j}\right)$ by $\hat \mu_{k_{\mathcal Y} \otimes k_{\mathcal Y}}\!\left(\mathbb P_{Y\mid X=X_j}\right)$, we get
    \begin{align}
    \MoveEqLeft t_1 \approx  \fip{\hat \mu_{k_{\mathcal Y} \otimes k_{\mathcal Y}}\!\left(\mathbb P_{Y\mid X=X_j}\right),k_\mathcal{Y}(\cdot,Y_i)\otimes k_\mathcal{Y}(\cdot,Y_i)}{\mathcal{H}_{\mathcal Y}\otimes \mathcal{H}_{\mathcal Y}} 
    \overset{(a)}{=} \fip{\hat \mu_{k_{\mathcal Y} \otimes k_{\mathcal Y}}\!\left(\mathbb P_{Y\mid X=X_j}\right),S_{k_ {\mathcal Y} \otimes k_{\mathcal Y},Y}^*\mathbf e_i}{\mathcal{H}_{\mathcal Y}\otimes\mathcal{H}_{\mathcal Y}} \notag \\
        &\overset{(b)}{=} \fip{S_{k_ {\mathcal Y} \otimes k_{\mathcal Y},Y} \hat \mu_{k_{\mathcal Y} \otimes k_{\mathcal Y}}\!\left(\mathbb P_{Y\mid X=X_j}\right),\mathbf e_i}{\R^n} \notag \\
        &\overset{\eqref{eq:conditional-mean-estimator}}{=} \fip{\frac1nS_{k_ {\mathcal Y} \otimes k_{\mathcal Y},Y}S_{k_ {\mathcal Y} \otimes k_{\mathcal Y},Y}^*\bm1_n + S_{k_ {\mathcal Y} \otimes k_{\mathcal Y},Y}S_{k_ {\mathcal Y} \otimes k_{\mathcal Y},Y}^*\tilde {\mathbf K}_X\!\left( \tilde{\mathbf K}_X+n\epsilon_n \mathbf I_n\right)^{-1}\mathbf e_j,\mathbf e_i}{\R^n} \notag \\
        &\overset{(c)}{=} \fip{\frac1n\!\left(\mathbf K_Y\circ \mathbf K_Y\right)\bm1_n + \left(\mathbf K_Y\circ \mathbf K_Y\right)\tilde {\mathbf K}_X\!\left( \tilde{\mathbf K}_X+n\epsilon_n \mathbf I_n\right)^{-1}\mathbf e_j,\mathbf e_i}{\R^n}, \label{eq:t1-intermediate-step}
    \end{align}
    by using the definition of the sampling operator in (a), the defining property of adjoint operators in (b), and  $S_{k_ {\mathcal Y} \otimes k_{\mathcal Y},Y}S_{k_ {\mathcal Y} \otimes k_{\mathcal Y},Y}^* = \mathbf K_Y\circ \mathbf K_Y$ in (c). 

    Let us now consider the outer expectation $\mathbb{E}_{\hat{\mathbb{P}}_{X,n}}$ of $t_1$'s approximation. We sum \eqref{eq:t1-intermediate-step} over $j \in [n]$, divide by $n$, and, by the linearity of the inner product and as $\frac{1}{n}\sum_{j=1}^n\mathbf e_j = \frac1n \mathbf 1_n$, obtain
    \begin{equation}
        \mathbb{E}_{\hat{\mathbb{P}}_{X,n}}[t_1] \approx \fip{\frac1n\!\left(\mathbf K_Y\circ \mathbf K_Y\right)\bm1_n + \frac 1n \left(\mathbf K_Y\circ \mathbf K_Y\right)\tilde {\mathbf K}_X\!\left( \tilde{\mathbf K}_X+n\epsilon_n \mathbf I_n\right)^{-1}\bm 1_n,\mathbf e_i}{\R^n} \label{eq:term-t1-final}
    \end{equation}
    
        \item \textbf{Term $t_2$.} We have for $i,j\in[n]$ that
\begin{align*}
      \mathbb E_{Y\mid X=X_j}\!\left[k_\mathcal{Y}(Y,Y_i) \right] &\overset{(a)}{=}\mathbb E_{Y\mid X=X_j}\!\left[\fip{k_{\mathcal Y}(\cdot,Y),k_{\mathcal Y}(\cdot,Y_i)}{\mathcal H_{\mathcal Y}}\right] 
      \overset{(b)}{=} \fip{\mathbb E_{Y\mid X=X_j}\!\left[k_{\mathcal Y}(\cdot,Y)\right],k_{\mathcal Y}(\cdot,Y_i)}{\mathcal H_{\mathcal Y}} \\
      &\overset{(c)}{=} \fip{\mu_{k_{\mathcal Y}}\!\left(\mathbb P_{Y\mid X=X_j}\right),k_\mathcal{Y}(\cdot,Y_i)}{\mathcal{H}_{\mathcal Y}},
\end{align*}
by using the reproducing property in (a), flipping the integral and the inner product using \citet[(A.32)]{steinwart08support} in (b), and by the definition of the conditional mean embedding in (c).

Again replacing $\mu_{k_{\mathcal Y}}\!\left(\mathbb P_{Y\mid X=X_j}\right)$ by its empirical counterpart $\hat \mu_{k_{\mathcal Y}}\!\left(\mathbb P_{Y\mid X=X_j}\right)$, we obtain the approximation
    \begin{align}
        t_2 &\approx
         \fip{\hat \mu_{k_{\mathcal Y}}\!\left(\mathbb P_{Y\mid X=X_j}\right),k_{\mathcal Y}(\cdot,Y_i)}{\mathcal{H}_{\mathcal Y}} \overset{(a)}{=} \fip{\hat \mu_{k_{\mathcal Y}}\!\left(\mathbb P_{Y\mid X=X_j}\right),S_{k_{\mathcal Y},Y}^*\mathbf e_i}{\mathcal{H}_{\mathcal Y}} \notag \\
         &\overset{(b)}{=}  \fip{S_{k_{\mathcal Y},Y}\hat \mu_{k_{\mathcal Y}}\!\left(\mathbb P_{Y\mid X=X_j}\right),\mathbf e_i}{\R^n} \notag \\
        &\overset{\eqref{eq:conditional-mean-estimator}}{=} \fip{\frac1nS_{k_{\mathcal Y},Y}S_{{k_{\mathcal Y}},Y}^*\bm1_n + S_{k_ {\mathcal Y},Y}S_{k_{\mathcal Y},Y}^*\tilde {\mathbf K}_X\!\left( \tilde{\mathbf K}_X+n\epsilon_n \mathbf I_n\right)^{-1}\mathbf e_j,\mathbf e_i}{\R^n} \notag \\
        &\overset{(c)}{=} \fip{\frac1n\mathbf K_Y\bm1_n + \mathbf K_Y\tilde {\mathbf K}_X\!\left( \tilde{\mathbf K}_X+n\epsilon_n \mathbf I_n\right)^{-1}\mathbf e_j,\mathbf e_i}{\R^n}
        \stackrel{(d)}{=}\fip{\frac1n\mathbf K_Y\bm1_n + \mathbf M \mathbf e_j,\mathbf e_i}{\R^n} \nonumber\\
        & \stackrel{(e)}{=} \frac{1}{n} \mathbf e_i^\top \mathbf K_Y\bm1_n + \mathbf e_i^\top \mathbf M \mathbf e_j, \label{eq:derivation-term-t2-final}
    \end{align}
    where (a) is by the definition of the sampling operator, (b) comes from the definition of the adjoint operator, and in (c), we use that $S_{k_{\mathcal Y},Y}S_{k_{\mathcal Y},Y}^* = \mathbf K_Y$, (d) comes from the definition of $\mathbf M$ in  \eqref{eq:gram-matrices}, (e) follows from the definition and the linearity of the inner product in $\mathbb R^n$. Squaring \eqref{eq:derivation-term-t2-final} gives
    \begin{align}
         t_2^2 &\approx \frac{1}{n^2}\mathbf e_i^\top\mathbf K_Y\bm1_n\mathbf e_i^\top\mathbf K_Y\bm1_n + \frac2n \mathbf e_i^\top\mathbf K_Y\bm1_n\mathbf e_i^\top\mathbf M\mathbf e_j + \mathbf e_i^\top\mathbf M\mathbf e_j\mathbf e_i^\top\mathbf M\mathbf e_j \notag \\
        &= \frac{1}{n^2}\mathbf e_i^\top\mathbf K_Y\bm1_n\bm1_n^\top\mathbf K_Y\mathbf e_i + \frac2n \mathbf e_i^\top\mathbf K_Y\bm1_n\mathbf e_j^\top\mathbf M^\top\mathbf e_i + \mathbf e_i^\top\mathbf M\mathbf e_j\mathbf e_j^\top\mathbf M^\top\mathbf e_i, \label{eq:t2-squared-intermediate-step}
    \end{align}
    where we use that a real number equals its transpose and the symmetry of Gram matrices. 
    
    To get the outer expectation, we sum \eqref{eq:t2-squared-intermediate-step} over $j \in [n]$, divide by $n$, and obtain
    \begin{equation}
        \mathbb{E}_{\hat {\mathbb{P}}_{X,n}}\!\left[t_2^2\right] \approx \frac{1}{n^2}\mathbf e_i^\top\mathbf K_Y\bm1_n\bm1_n^\top\mathbf K_Y\mathbf e_i + \frac{2}{n^2} \mathbf e_i^\top\mathbf K_Y\bm1_n\bm 1_n^\top\mathbf M^\top\mathbf e_i + \frac1n\mathbf e_i^\top\mathbf M\mathbf I_n\mathbf M^\top\mathbf e_i. \label{eq:derivation-t2-final}
    \end{equation}
    \end{itemize}
Hence, subtracting \eqref{eq:derivation-t2-final} from \eqref{eq:term-t1-final} and rearranging, we have for the numerator
    \begin{equation}
        \begin{aligned}
        E_{n,i}^{\operatorname{RKHS}}
        &= \frac1n\Big\langle\mathbf e_i,
        \left(\mathbf K_Y\circ \mathbf K_Y\right)\bm1_n 
        + \left(\mathbf K_Y\circ \mathbf K_Y\right)
        \tilde {\mathbf K}_X
        \left( \tilde{\mathbf K}_X+n\epsilon_n \mathbf I_n\right)^{-1}
        \bm 1_n \\
        &\quad-\frac{1}{n}\mathbf K_Y\bm1_n\bm1_n^\top\mathbf K_Y\mathbf e_i 
        - \frac{2}{n} \mathbf K_Y\bm1_n\bm 1_n^\top\mathbf M^\top\mathbf e_i 
        - \mathbf M\mathbf M^\top\mathbf e_i
        \Big\rangle_{\R^n}.
        \end{aligned}
    \label{eq:RKHS_denom}
    \end{equation}

\end{itemize}
Combining the numerator \eqref{eq:RKHS_denom} and denominator \eqref{eq:RKHS_num} concludes the derivation.

\subsection{Proof of Theorem \ref{theorem: rate of convergence}} \label{sec:proof-rates-knn}
\begin{proof}
In the sequel we let $C$ denote an absolute constant which may change from line to line, and we use the symbol $\cleq$ to indicate weak inequality up to an absolute constant which again may change from line to line. We will also denote the complement of an event $A$ by $A^c$. Finally, we recall that in the main text the kernel $k_\mathcal{Y}$ is assumed to be bounded and let $\kappa := \sup_{y, y'\in \mathcal Y}k_{\mathcal{Y}}(y, y')$. For each $i \in [n]$, remember that 
\begin{align*}
& E_{n,i}^{\text{K-NN}} = \frac{1}{2 K (n - 1)}\sum_{j\neq i}\sum_{k\in\mathcal{N}_j^{\setminus i}} \big(k_{\mathcal{Y}}(Y_j, Y_i) - k_{\mathcal{Y}}(Y_k, Y_i)\big)^2,
\\
& V_{n,i}^{\text{K-NN}} = \frac{1}{n-1}
\sum_{j \neq i}
k_\mathcal{Y}^2(Y_j, Y_i)
-
\Big[
\frac{1}{n-1}
\sum_{j \neq i}
k_\mathcal{Y}(Y_j, Y_i)
\Big]^2.
\end{align*}
Additionally for $i\in [n]$ let 
\begin{align*}
V_i := \mathbb{V}_Y^{\setminus i}[k_\mathcal{Y}(Y, Y_i)],\qquad\qquad E_i := \ee_X^{\setminus i}\mathbb{V}_{Y\mid X}^{\setminus i}[k_\mathcal{Y}(Y, Y_i)],
\end{align*}
where $(X_i, Y_i)$ is treated as a fixed observed value and the expectation is taken over $(X, Y)$. With this, let
\begin{align*}
    Q := \int_\mathcal{Y} \frac{\mathbb{E}_X \left [ \mathbb{V}_{Y|X} \left [ k_\mathcal{Y} \left ( Y, y \right ) \right ] \right ]}{\mathbb{V}_Y [k_\mathcal{Y}(Y,y)]} \mathrm{d} \mathbb{P}_Y (y),\qquad \hat{Q}_n:= \frac{1}{n} \sum_{i=1}^n \frac{E_{n,i}^{\text{K-NN}}}{V_{n,i}^{\text{K-NN}}},\qquad\hat{Q}_n' := \frac{1}{n} \sum_{i=1}^n \frac{E_{i}^{\text{K-NN}}}{V_i}, \qquad\hat{Q}_n'' := \frac{1}{n} \sum_{i=1}^n \frac{E_i}{V_i}.
\end{align*}
Note that $\abs{\hat{D}_n^{\text{K-NN}} - D} = \abs{Q - \hat{Q}_n}$. Introduce the event 
\begin{align}
\Omega_n = \bigcap_{i=1}^n \left \{\abs{V_{n,i}^{\text{K-NN}} - V_i} \leq \frac{1}{2} V_i\right \}.
\label{equation: good variance event}
\end{align}
Using triangle inequality, we have
\begin{align}
\abs{Q - \hat{Q}_n} \leq \abs{Q - \hat{Q}_n''} + \abs{\hat{Q}_n'' - \hat{Q}_n'} + \abs{\hat{Q}_n' - \hat{Q}_n}.
\label{equation: Q tringale bound}
\end{align}
Since $V_{n,i}^{\text{K-NN}}$ and $V_i$ are non-negative, $\Omega_n$ implies $\frac{1}{2} V_i \geq V_{n,i}^{\text{K-NN}}$ for all $i \in [n]$. Therefore, conditional on the event $\Omega_n$ it holds that 
\begin{align*}
\abs{\hat{Q}_n - \hat{Q}_n'} = \left | \frac{1}{n} \sum_{i=1}^n (V_{n,i}^{\text{K-NN}} - V_i) \frac{E_{n,i}^{\text{K-NN}}}{V_i V_{n,i}^{\text{K-NN}}} \right | \leq \frac{1}{n} \sum_{i=1}^n \abs{V_{n,i}^{\text{K-NN}} - V_i} \frac{E_{n,i}^{\text{K-NN}}}{V_i  V_{n,i}^{\text{K-NN}}} \leq \frac{2}{n} \sum_{i=1}^n \abs{V_{n,i}^{\text{K-NN}} - V_i} \frac{E_{n,i}^{\text{K-NN}}}{V_i^2} \coloneq \tilde{Q}_n. 
\end{align*}
Therefore, for any $\delta > 0$ it holds that 
\begin{align}
\pp(\abs{\hat{Q}_n - Q} > \delta) & \leq \mathbb{P}(\abs{Q - \hat{Q}_n''} + \abs{\hat{Q}_n'' - \hat{Q}_n'} + \abs{\hat{Q}_n' - \hat{Q}_n} > \delta \mid \Omega_n) \pp(\Omega_n) + \pp(\Omega_n^c) \nonumber \\
& \leq \pp(\abs{Q - \hat{Q}_n''} + \abs{\hat{Q}_n'' -  \hat{Q}_n'} + \tilde{Q}_n  > \delta) + \pp(\Omega_n^c).
\label{equation: Q tilde prob bound}
\end{align}
In the following we bound each of the terms appearing in \eqref{equation: Q tilde prob bound}, where we frequently use the following equality:
\begin{equation}
\mathbb{E}[|Z|] = \int_0^\infty \mathbb{P} (|Z| > t) \mathrm{d}t, 
\label{equation: expectation tail integral equality}
\end{equation}
which holds for any real-valued random variable $Z$. For the first term in \eqref{equation: Q tilde prob bound} observe that 
by the law of total variance, for $i\in [n]$ we have $E_i/V_i \in [0,1]$. Additionally $E_i/V_i$ is an i.i.d.\ sequence. Therefore by Hoeffding's inequality we have
\begin{align*}
    \pp(\abs{\hat{Q}_n'' - \mathbb{E} [\hat{Q}_n''] } > t)\leq  2e^{-2nt^2}. 
\end{align*}
Moreover $\mathbb{E}[\hat{Q}_n''] = \frac{1}{n} \sum_{i=1}^n \ee[E_i/V_i] = Q$. Therefore 
\begin{align*}
\mathbb{E}\abs{Q - \hat{Q}_n''} = \mathbb{E} \abs{\hat{Q}_n'' - \mathbb{E}[\hat{Q}_n'']} = \int_0^\infty \mathbb{P}(\abs{\hat{Q}_n'' - \mathbb{E} [\hat{Q}_n'']} > t ) \mathrm{d} t \leq 2 \int_0^\infty e^{ -2nt^2} \mathrm{d}t \lesssim \frac{1}{\sqrt{n}},
\end{align*}
and by Markov's inequality $\abs{Q - \hat{Q}_n''} = \mathcal{O}_\mathbb{P}(n^{-1/2})$. For the second term in \eqref{equation: Q tilde prob bound}, observe first that for each $i \in [n]$, due to the boundedness of $k_\mathcal{Y}$ the statistic $V_{n,i}^{\text{K-NN}}$ enjoys the bounded difference property (see Definition~\ref{equation: bounded difference property}) with absolute finite difference bounded from above by $(n-1)^{-1}$ up to a multiplicative constant depending on $\kappa$. Therefore, by Theorem~\ref{theorem: bounded difference} we have that 
\begin{equation}
\mathbb{P}^{\setminus i}(\abs{V_{n,i}^{\text{K-NN}} - \mathbb{E}^{\setminus i}[V_{n,i}^{\text{K-NN}}] } > t) \leq 2 \exp(-C n t^2) \quad \forall t > 0,
\label{equation: varaince exponential concentration}
\end{equation}
where the term on the right does not depend on $Y_i$. Additionally for fixed $i$ note that $\{ k_{\mathcal{Y}} (Y_j, Y_i) \}_{j \neq i }$ is an i.i.d.\ sequence of bounded random variables. Therefore, by standard result, the bias of the estimated variance for each $i \in [n]$ is 
\begin{equation}
\abs{\mathbb{E}^{\setminus i}[V_{n,i}^{\text{K-NN}}] - V_i} = \frac{1}{n-1} V_i \lesssim \frac{1}{n}, 
\label{equation: varaince bias}
\end{equation}
where the final inequality follows from the boundedness of $k_{\mathcal{Y}}$. Consequently we have that 
\begin{subequations}
\begin{align}
\mathbb{E}[\tilde{Q}_n] & \cleq \frac{1}{n} \sum_{i=1}^n \mathbb{E} \left[\frac{\abs{V_{n,i}^{\text{K-NN}} - V_i}}{V_i^2} \right ] \label{equation: Q tilde expectation bound i}\\
& \leq \frac{1}{n} \sum_{i=1}^n \mathbb{E} \left [ \frac{\abs{V_{n,i}^{\text{K-NN}} - \ee[V_{n,i}^{\text{K-NN}}]}}{V_i^2} + \frac{\abs{\ee[V_{n,i}^{\text{K-NN}}] - V_i}}{V_i^2} \right ] \nonumber \\
& \cleq \frac{1}{n}\int_{\mathcal{Y}} \Big( \mathbb{V}_{Y}\big(k_\mathcal{Y}(Y,y)\big)\Big)^{-2} \mathrm{d} \mathbb{P}_Y (y) \nonumber \\
& \quad + \frac{1}{n} \sum_{i=1}^n  \ee\Big[ \int_0^\infty \pp^{\setminus i}\big(\abs{V_{n,i}^{\text{K-NN}} - V_i} > t V_i^2 \big) \mathrm{d} t \Big]  \label{equation: Q tilde expectation bound ii} \\
& \leq \frac{1}{n}\int_{\mathcal{Y}} \Big(\mathbb{V}_{Y}\big(k_\mathcal{Y}(Y,y)\big)\Big)^{-2} \mathrm{d} \mathbb{P}_Y (y) + \int_{\mathcal{Y}} \int_{0}^{\infty}\exp\Big(-C n t^2\big(\mathbb{V}_{Y}[k_\mathcal{Y}(Y,y)]\big)^4\Big)\mathrm{d} t \mathrm{d} \mathbb{P}_{Y} (y) \label{equation: Q tilde expectation bound iii} \\
& \cleq \Big\{ \frac{1}{n} + \frac{1}{\sqrt{n}}\Big\} \int_{\mathcal{Y}} \big(\mathbb{V}_{Y}[k_\mathcal{Y}(Y,y)] \big)^{-2} \mathrm{d} \mathbb{P}_Y (y) \cleq \frac{1}{\sqrt{n}}.  \label{equation: Q tilde expectation bound iv}
\end{align}
\end{subequations}
In particular \eqref{equation: Q tilde expectation bound i} follows from the the boundedness of $E_{n,i}^{\text{K-NN}}$ for all $i \in [n]$, which holds due to the boundedness of $k_{\mathcal{Y}}$, \eqref{equation: Q tilde expectation bound ii} follows \eqref{equation: expectation tail integral equality}, \eqref{equation: Q tilde expectation bound iii} follows from \eqref{equation: varaince exponential concentration}, and \eqref{equation: Q tilde expectation bound iv} follows from part (i) of Assumption~\ref{assumption: properties of P}. Consequently it holds that $\tilde{Q}_n = \mathcal{O}_{\mathbb{P}} (n^{-1/2})$. For term $\abs{\hat{Q}_n' - \hat{Q}_n''}$ using triangle inequality and then Jensen's inequality we have
\begin{align*}
    \ee[\abs{\hat{Q}_n' - \hat{Q}_n''}] &\leq  \ee\left[\frac{\sqrt{\mathbb{V}^{\setminus i}(E_{n,i}^{\text{K-NN}})}}{V_i}\right] + \ee\left[\frac{\abs{\ee^{\setminus i}[E_{n,i}^{\text{K-NN}}] - E_i}}{V_i}\right].
\end{align*}
For the variance term, note that by Efron-Stein inequality we have
\begin{align*}
    \mathbb{V}^{\setminus i}\left(E_{n,i}^{\text{K-NN}}\right) \leq \frac{1}{2}\sum_{\ell \neq i}\ee^{\setminus i}\left[\left(E_{n,i}^{\text{K-NN}} - E_{n,i, \ell}^{\text{K-NN}}\right)^2\right],
\end{align*}
where $E_{n,i, \ell}^{\text{K-NN}}$ is calculated using sample $\{(X_i, Y_i)\}_{i\neq \ell}\cup \{(X_\ell', Y_\ell')\}$, such that $(X_\ell', Y_\ell')$ is an i.i.d.\  copy of $(X_\ell, Y_\ell)$. Note that in replacing $(X_\ell, Y_\ell)$ by the i.i.d.\  copy $(X_\ell', Y_\ell')$, observation $\ell$ appears in two roles:
\begin{enumerate}
    \item As center $j = \ell$: in this case all $K$ terms in the neighbourhood $\ell$ can change which contributes to at most $8K\kappa^2/2K(n - 1) = 4\kappa^2/(n - 1)$,
    \item As neigbour $k = \ell$: in this case by Assumption~\ref{assumption: properties of KNN}, $X_\ell$ is at most in the neighbourhood of $C K$ other $X_j$'s and the therefore the total contributed value of this change is $8C K\kappa^2/2K(n - 1) = 4C\kappa^2/(n - 1)$.
\end{enumerate}
Therefore 
\begin{align*}
    \abs{E_{n,i}^{\text{K-NN}} - E_{n,i, \ell}^{\text{K-NN}}} \leq \frac{4\kappa^2}{n - 1}\left(1 + C\right).
\end{align*}
Hence 
\begin{align*}
    \mathbb{V}^{\setminus i}\left(E_{n,i}^{\text{K-NN}}\right) \lesssim \frac{1}{n}.
\end{align*}
This gives us
\begin{align*}
    \ee\left[\frac{\sqrt{\mathbb{V}^{\setminus i}(E_{n,i}^{\text{K-NN}})}}{V_i}\right] \lesssim \frac{1}{\sqrt{n}}\ee\left[\frac{1}{V_i}\right].
\end{align*}
For the bias term, note that 
\begin{align*}
     \abs{\ee^{\setminus i}[E_{n,i}^{\text{K-NN}}] - E_i} &= \left |\ee^{\setminus i}\left[\frac{1}{K}\sum_{k\in\mathcal{N}_i^{\setminus j}}(k_\mathcal{Y}(Y_j, Y_i) - k_\mathcal{Y}(Y_k, Y_i))^2 - \ee^{\setminus i}_{Y\mid X}\left[(k_\mathcal{Y}(Y, Y_i) - k_\mathcal{Y}(Y', Y_i))^2\right]\right]\right|,
\end{align*}
where $Y$ and $Y'$ are i.i.d.\ conditional on $X$. Let 
\begin{align*}
    m_1(x):=\ee^{\setminus i}[k_\mathcal{Y}(Y, Y_i) \mid X=x], \quad m_2(x):=\mathbb{E}[k_\mathcal{Y}^2(Y, Y_i) \mid X=x].
\end{align*}
Then 
\begin{align*}
    \ee^{\setminus i}\left[\left(k_\mathcal{Y}(Y, Y_i)-k_\mathcal{Y}(Y', Y_i)\right)^2 \mid X=x\right] = 2 \mathbb{V}^{\setminus i}(k_\mathcal{Y}(Y, Y_i) \mid X=x) = 2(m_2(x)-m_1(x)^2),
\end{align*}
and
\begin{align*}
    \ee^{\setminus i}\left[\left(k_\mathcal{Y}(Y_j, Y_i)-k_\mathcal{Y}(Y_k, Y_i)\right)^2 \mid X_j = x, X_k=x^{\prime}\right] = m_2(x) + m_2(x^{\prime})-2 m_1(x) m_1(x^{\prime}).
\end{align*}
Let 
\begin{align*}
    \Delta\left(x, x^{\prime}\right)&:=\ee^{\setminus i}\left[\left(k_\mathcal{Y}(Y_j, Y_i)-k_\mathcal{Y}(Y_k, Y_i)\right)^2 \mid x, x^{\prime}\right] - \ee^{\setminus i}\left[\left(k_\mathcal{Y}(Y, Y_i) - k_\mathcal{Y}(Y', Y_i)\right)^2 \mid X=x\right] \\
    &= \left(m_2(x^{\prime})-m_2(x)\right) - 2 m_1(x)\left(m_1(x^{\prime}) - m_1(x)\right),
\end{align*}
which gives us 
\begin{align*}
    \abs{\Delta\left(x, x^{\prime}\right)} \leq \abs{m_2(x^{\prime})-m_2(x)} - 2 \kappa \abs{m_1(x^{\prime}) - m_1(x)} \lesssim d_{\mathcal{X}}(x, x^{\prime})^\beta.
\end{align*}
where the last inequality is by Assumption~\ref{assumption: properties of P} . Then by taking average over all the neighbours and using Lemma \ref{lmm:maxDist} together with Assumption~\ref{assumption: properties of P} we have 
\begin{align*}
    \abs{\ee^{\setminus i}[E_{n,i}^{\text{K-NN}}] - E_i} \lesssim \ee^{\setminus i}\left[\max_{k\in\mathcal{N}_j^{\setminus i}}d_{\mathcal{X}}^\beta(X_k, X_i)\right] \lesssim \left[\left(\frac{K}{n}\right)^{\beta / d}+n^{-2}(\log n)^{\beta / \alpha}\right].
\end{align*}
Consequently,
\begin{align}\label{eq:EhatE}
    \ee\left[\frac{\abs{\ee^{\setminus i}[E_{n,i}^{\text{K-NN}}] - E_i}}{V_i}\right] \lesssim \left[\left(\frac{K}{n}\right)^{\beta / d}+n^{-2}(\log n)^{\beta / \alpha}\right] \ee\left[\frac{1}{V_i}\right] \lesssim \left[\left(\frac{K}{n}\right)^{\beta / d}+n^{-2}(\log n)^{\beta / \alpha}\right].
\end{align}
Finally for $\pp(\Omega_n^c)$, we have that 
\begin{subequations}
\begin{align}
\mathbb{P}(\Omega_n^c) & \leq \sum_{i=1}^n \ee\left[\mathbb{P}^{\setminus i}\left(\abs{\hat{V}_{n,i}^{\text{K-NN}} - V_i} \geq \frac{1}{2} V_i\right)\right]  \nonumber \\
& \leq \sum_{i=1}^n \ee\left[\mathbb{P}^{\setminus i} \left(\abs{ \hat{V}_{n,i}^{\text{K-NN}} - \mathbb{E}^{\setminus i}[\hat{V}_{n,i}^{\text{K-NN}}]} + \abs{\mathbb{E}^{\setminus i}[\hat{V}_{n,i}^{\text{K-NN}}] - V_i} \geq \frac{1}{2} V_i \right )\right] \nonumber 
\end{align}
\end{subequations}
Note that by \eqref{equation: varaince bias} we have 
\begin{align*}
    \mathbb{P}^{\setminus i} \left(\abs{ \hat{V}_{n,i}^{\text{K-NN}} - \mathbb{E}^{\setminus i}[\hat{V}_{n,i}^{\text{K-NN}}]} + \abs{\mathbb{E}^{\setminus i}[\hat{V}_{n,i}^{\text{K-NN}}] - V_i} \geq \frac{1}{2} V_i \right) \leq \mathbb{P}^{\setminus i} \left(\abs{ \hat{V}_{n,i}^{\text{K-NN}} - \mathbb{E}^{\setminus i}[\hat{V}_{n,i}^{\text{K-NN}}]} \geq \left(\frac{1}{2} - \frac{1}{n-1}\right) V_i \right),
\end{align*}
therefore for $n\geq 4$ we can write
\begin{subequations}
\begin{align}
\mathbb{P}(\Omega_n^c) & \leq \sum_{i=1}^n \ee\left[\mathbb{P}^{\setminus i} \left(\abs{\hat{V}_{n,i}^{\text{K-NN}} - \mathbb{E}[\hat{V}_{n,i}^{\text{K-NN}}]} \geq \frac{1}{6}V_i \right)\right] \label{equation: Omega neg boud i} \\
& \leq n \ee\left[\exp\left(-\frac{1}{36} C n V_i^2\right)\right]\label{equation: Omega neg boud ii} \\
& \lesssim \frac{1}{\sqrt{n}} \ee\left[\frac{1}{V_i^{3}}\right] \lesssim \frac{1}{\sqrt{n}}, \label{equation: Omega neg boud iii}
\end{align}
\end{subequations}
where \eqref{equation: Omega neg boud ii} follows from \eqref{equation: varaince exponential concentration}, and finally from $e^{-x} < x^{-3/2}$ for $x > 0$, we have \eqref{equation: Omega neg boud iii}.

We already showed $\abs{Q - \hat{Q}_n''} = \mathcal{O}_\pp(n^{-1/2})$, $\tilde{Q}_n = \mathcal{O}_\pp(n^{-1/2})$, and $\pp(\Omega_n^c) \lesssim n^{-1/2}$, therefore by \eqref{eq:EhatE} we have
\begin{align*}
    \pp(\abs{\hat{Q}_n - Q} > \delta) &\leq \pp(\abs{Q - \hat{Q}_n''} + \abs{\hat{Q}_n'' -  \hat{Q}_n'} + \tilde{Q}_n  > \delta) + \pp(\Omega_n^c) \\
    &\lesssim \pp(\abs{Q - \hat{Q}_n''} > \delta /3) + \pp(\abs{\hat{Q}_n'' -  \hat{Q}_n'} > \delta/3) + \pp(\tilde{Q}_n  > \delta/3) + \frac{1}{\sqrt{n}} \\
    &\lesssim \frac{1}{\delta}\left(\frac{1}{\sqrt{n}}+\left(\frac{K}{n}\right)^{\beta / d}+(\log n)^{\beta / \alpha} n^{-2}\right),
\end{align*}
which gives us 
\begin{equation*}
    \abs{\hat{D}^{\text{K-NN}}(X,Y) - D(X,Y)} = \mathcal{O}_{\mathbb{P}}\left(\frac{1}{\sqrt{n}}+\left(\frac{K}{n}\right)^{\beta / d}+(\log n)^{\beta / \alpha} n^{-2}\right),
\end{equation*}
and completes the proof. 
\end{proof}

\section{Auxiliary Results}

This section collects our auxiliary results, used in the proofs of the results stated in the main text.

\begin{lemmaA}[Alternative expression]\label{prop:kmd-population-version2}   In the setting of Definition~\ref{def:KIR}, it holds that
    \begin{equation}
D(Y,X) = \int_\mathcal{Y} \frac{\mathbb V_X\! \left  [ \mathbb{E}_{Y\mid X} \!\left[k_\mathcal{Y}(Y,y) \right] \right ]}{\mathbb{V}_Y \!\left[k_\mathcal{Y}(Y,y)\right]}\mathrm d  \mathbb P_Y(y). \label{eq:kmd-population-version2}
\end{equation}  
\end{lemmaA}
\begin{proof}
    We have the chain of equalities
\begin{align*}
D(Y, X)
&\overset{\eqref{eq:kmd-population-version1}}{=} 1 -  \int_{\mathcal{Y}} \frac{\mathbb{E}_X\!\left[\,\mathbb{V}_{Y \mid X}\left[k_\mathcal{Y}(Y,y)\right]\,\right]}
                    {\mathbb{V}_Y \left[k_\mathcal{Y}(Y,y)\right]} \, \mathrm{d}\mathbb{P}_Y(y) 
                    \overset{(a)}{=} 1 + \int_{\mathcal{Y}} \frac{\pm \mathbb{V}_Y\!\left[k_\mathcal{Y}(Y,y)\right]
    - \mathbb{E}_X\!\left[\,\mathbb{V}_{Y \mid X}\left[k_\mathcal{Y}(Y,y)\right]\,\right]}
    {\mathbb{V}_Y \left[k_\mathcal{Y}(Y,y)\right]} \, \mathrm{d}\mathbb{P}_Y(y) \\
&\overset{(b)}{=} 1 + \int_{\mathcal{Y}} \frac{\mathbb{V}_Y\!\left[k_\mathcal{Y}(Y,y)\right]
    - \mathbb{E}_X\!\left[\,\mathbb{V}_{Y \mid X}\left[k_\mathcal{Y}(Y,y)\right]\,\right]}
    {\mathbb{V}_Y \left[k_\mathcal{Y}(Y,y)\right]} \, \mathrm{d}\mathbb{P}_Y(y) - 1 
    \stackrel{(c)}{=} \int_{\mathcal{Y}} \frac{\mathbb{V}_X\!\left[\,\mathbb{E}_{Y \mid X}\left[k_\mathcal{Y}(Y,y)\right]\,\right]}
                   {\mathbb{V}_Y \left[k_\mathcal{Y}(Y,y)\right]} \, \mathrm{d}\mathbb{P}_Y(y),
\end{align*}
where in (a), we add zero; in (b), we split the fraction and simplify; and in (c), we use that $1-1=0$ together with
\begin{equation*}
\mathbb{V}_Y[k_\mathcal{Y}(Y,y)]
=
\mathbb{E}_X\!\left[\mathbb{V}_{Y \mid X}\left[k_\mathcal{Y}(Y,y) \right]\right]
+
\mathbb{V}_X\!\left[\mathbb{E}_{Y \mid X}\left[k_\mathcal{Y}(Y,y)\right]\right]
\end{equation*}
by the law of total variance ($y\in \mathcal Y$).
\end{proof}

\begin{lemmaA}[Nearest-neighbour distance]\label{lmm:maxDist}
Let $(\mathcal{X}, d_\mathcal{X})$ be a metric space and $X_1, \ldots, X_n$ i.i.d. Let $d_j$ be the distance from $X_j$ to its $K$-th nearest neighbour among $\left\{X_{\ell}: \ell \neq j\right\}$, i.e.
    \[
        d_j := \max _{k \in \mathcal{N}_j} d_\mathcal{X}(X_k, X_j)
    \]
    where $\mathcal{N}_j$ is the set of $K$-nearest neighbours of $X_j$ in $\{X_k\}_{k\neq j}$. 
    Assume:
    
    \noindent (A1) There exist $x^* \in \mathcal{X}$ and $\alpha, C_1, C_2>0$ such that for all $t \geq 0$,
    \[
    \mathbb{P}\left(d_\mathcal{X}(X_1, x^*) \geq t\right) \leq C_1 e^{-C_2 t^\alpha} .
    \]
    
    \noindent(A2) There exist constants $d>0$ and $c_0>0$ such that for every radius $T>0$, for all $x\in \mathcal X$ with $d_\mathcal{X}(x^*, x) \leq T$ and all $r>0$,

    \[
    \mathbb{P}\left(d_\mathcal{X}(X_1, x) \leq r\right) \geq c_0 r^d
    \]

    Let $\beta>0$. Then for all $n \geq 3$ and $1 \leq K \leq n / 2$,

    \[
    \mathbb{E}[d_j^\beta] \lesssim \left(\frac{K}{n}\right)^{\beta / d}+(\log n)^{\beta / \alpha} n^{-2},
    \]
    where $\lesssim$ hides constants depending only on $\alpha, \beta, d, c_0, C_1, C_2$.
\end{lemmaA}
\begin{proof}[Proof of Lemma~\ref{lmm:maxDist}]
    Fix $\delta \in(0,1)$ and define
    \[
    T_{n, \delta}:=\left(\frac{1}{C_2} \log \frac{C_1 n}{\delta}\right)^{1 / \alpha}.
    \]
    By (A1) and a union bound,

    \[
    \mathbb{P}\left(\max_{1 \leq \ell \leq n} d_\mathcal{X}\left(X_{\ell}, x^*\right)>T_{n, \delta}\right) \leq n C_1 e^{-C_2 T_{n, \delta}^\alpha}=\delta .
    \]

    Introduce the event
    \[
    \mathcal{G}_\delta : =\left\{\max _{1 \leq \ell \leq n} d_\mathcal{X}(X_{\ell}, x^*) \leq T_{n, \delta}\right\} .
    \]

    Then $\mathbb{P}\left(\mathcal{G}_\delta\right) \geq 1-\delta$. On $\mathcal{G}_\delta$, for all sample points $X_i$ we have $d_\mathcal{X}(X_i, x^*) \leq T_{n, \delta}$, hence for any $j$,
    \[
        d_j \leq 2 T_{n, \delta}.
    \]
    Therefore, we have
    \begin{align*}
         \ee[d_j^\beta] = \ee[d_j^\beta \bone_{\{\mathcal{G}_\delta\}}] + \ee[d_j^\beta \bone_{\{\mathcal{G}_\delta^c\}}] \leq \ee[d_j^\beta \mid \mathcal{G}_\delta] + \ee[d_j^\beta \bone_{\{\mathcal{G}_\delta^c\}}].
    \end{align*}
    Note that 
    \begin{align*}
        d_j \leq 2\max_{1\leq \ell \leq n}d_\mathcal{X}(X_\ell, x^*).
    \end{align*}
    Therefore 
    \begin{align*}
        \ee\left[d_j^\beta \bone_{\{\mathcal{G}_\delta^c\}}\right] \leq 2^\beta \ee\left[\max_{\ell}d_{\mathcal{X}}^\beta(X_\ell, x^*)\bone_{\{\mathcal{G}_\delta^c\}}\right].
    \end{align*}
    Now using the tail bound on the maximum we have
    \begin{align*}
        \pp(\max_{\ell}d_{\mathcal{X}}^\beta(X_\ell, x^*)) \leq n C_1 e^{-C_2 t^\alpha},
    \end{align*}
    which with our choice of $T_{n, \delta}$ gives us
    \begin{align*}
        \ee[d_j^\beta \bone_{\{\mathcal{G}_\delta^c\}}]\lesssim (\log n)^{\beta/\alpha}\delta.
    \end{align*}
    Thus 
    \begin{align}\label{eq:djsplit}
         \ee[d_j^\beta] \lesssim \ee[d_j^\beta \mid \mathcal{G}_\delta] + (\log n)^{\beta/\alpha}\delta.
    \end{align}
    Now for fixed $j$, conditioning on $X_j = x$ and on $\mathcal{G}_\delta$ we have $d_\mathcal{X}(x, x^*)\leq T_{n, \delta}$. For any $r > 0$, note that by (A2)
    \[
        \pp(d_{\mathcal{X}}(X_1, x)\leq r) \geq c_0 r^d.
    \]
    Let $N_r := \sum_{\ell \neq j}\bone_{\{d_\mathcal{X}(X_\ell, x) \leq r\}}$. Conditional on $X_j = x$, $N_r\sim \text{Bin}(n - 1, \pp(d_{\mathcal{X}}(X_1, x)\leq r))$ . The event $\{d_j > r\}$ is exactly $\{N_r < K\}$. Therefore, for any $r$,
    \[
    \pp(d_j > r\mid X_j = x) = \pp(N_r < K). 
    \]
    Now choose $r_* := (2K/c_0(n - 1))^{1/d}$, so we have 
    \[
    \mu := \ee[N_{r_*}\mid X_j = x] = (n - 1)\pp(d_{\mathcal{X}}(X_1, x) \leq r_*) \geq 2 K.
    \]

    Using a standard multiplicative Chernoff bound for binomials, 
    \[
    \pp(N_{r_*} < K) \leq \exp(-\mu/8) \leq \exp(-K/4).
    \]
    More generally, for any $u\geq 1$, we have 
    \[
    \ee[N_{ur_*}\mid X_j = x] \geq (n - 1)c_0 u^d r_*^d = 2Ku^d.
    \]
    Using Chernoff bound again we have 
    \[
    \pp(d_j > ur_*\mid X_j = x) = \pp(N_{ur_*} < K) \leq \exp\left(-K u^d / 4\right).
    \]
    Note that these bounds hold for every $x$ such that $d_\mathcal{X}(x, x^*)\leq T_{n, \delta}$, hence they remain valid under $\mathcal{G}_\delta$.

    Using this tail bound we have 
    \begin{align*}
        \ee[d_j^\beta\mid X_j = x] &= \int_{0}^{r_*} \beta r^{\beta - 1}\pp(d_j > r\mid X_j = x)\mathrm{d}r + \int_{r_*}^\infty \beta r^{\beta - 1}\pp(d_j > r\mid X_j = x)\mathrm{d}r \\
        &\leq r_*^\beta + \beta r_*^\beta \int_{1}^\infty u^{\beta - 1}e^{-K u^{d / 4}}\mathrm{d} u \\
        &\lesssim r_*^\beta = \left(\frac{K}{n}\right)^{\beta/d}
    \end{align*}
    uniformly for all $x$ such that $d_\mathcal{X}(x, x^*)\leq T_{n, \delta}$. Therefore 
    \begin{align}\label{eq:djgood}
        \ee[d_j^\beta\mid \mathcal{G}_\delta] \lesssim \left(\frac{K}{n}\right)^{\beta/d}.
    \end{align}
    Finally combining \eqref{eq:djsplit} and \eqref{eq:djgood}, we have
    \[
    \ee[d_j^\beta] \lesssim \left(\frac{K}{n}\right)^{\beta/d} + (2T_{n, \delta})^\beta \delta.
    \]
    Choose $\delta = n^{-2}$ and we get 
    \[
    \ee[d_j^\beta] \lesssim \left(\frac{K}{n}\right)^{\beta/d} + (\log n)^{\beta/\alpha}n^{-2}.
    \]
\end{proof}

\section{External Results}

This section collects the external results that we use. Theorem~\ref{theorem: bounded difference} recalls McDiarmid's bounded differences inequality. Lemma~\ref{lemma:measurable-if-degenerate} gives a condition for the existence of a Borel measurable function relating two random variables a.s.

\begin{theoremA}[Bounded differences inequality; \citealt{boucheron13concentration}] \label{theorem: bounded difference}
Let $\mathcal{X}$ be a measurable space. A function $f: \mathcal{X}^n \to \mathbb{R}$ has the bounded difference property for some constants $c_1, \dots, c_n$ if, for each $i = 1, \dots, n$,
\begin{equation}
\sup_{\substack{x_1, \dots, x_n \\ x_i' \in \mathcal{X}}} \left | f \left ( x_1, \dots, x_{i-1}, x_i, x_{i+1}, x_n \right ) - f \left ( x_1, \dots, x_{i-1}, x_i', x_{i+1}, \dots x_n \right ) \right | \leq c_i. 
\label{equation: bounded difference property}
\end{equation}
Then, if $X_1, \dots, X_n$ is a sequence of independently distributed random variables and (\ref{equation: bounded difference property}) holds, putting $Z = f \left ( X_1, \dots, X_n \right )$ and $\nu = \frac{1}{4} \sum_{i=1}^n c_i^2$, for any $t > 0$, it holds that
\begin{equation*}
\mathbb{P} \left ( Z - \mathbb{E} \left ( Z \right ) > t \right ) \leq e^{-t^2 / \left ( 2 \nu \right )}.  
\end{equation*}
\end{theoremA}

\begin{lemmaA}[Remark A.2;\footnote{This remark appears in the \href{https://arxiv.org/abs/2012.14804}{arXiv version}.} \citealt{huang22kernel}] \label{lemma:measurable-if-degenerate} Let $(\Omega, \mathcal A, \mathbb P)$ be a probability space, $(\mathcal X, \tau_{\mathcal X})$ a topological space with Borel $\sigma$-algebra $\mathcal B(\tau_{\mathcal X})$, $(\mathcal Y, \tau_{\mathcal Y})$ a Polish space with Borel $\sigma$-algebra $\mathcal B(\tau_{\mathcal Y})$, and $X : (\Omega, \mathcal A) \to (\mathcal X, \mathcal B(\tau_{\mathcal X}))$ and $Y : (\Omega,\mathcal A) \to (\mathcal Y,\mathcal B(\tau_{\mathcal Y}))$ random variables. Denote the conditional distribution of $Y$ given $X$ by $\mathbb P_{Y\mid X}$. If $\mathbb P_{Y\mid X = x}$ is degenerate for a.e.\ $x\in\mathcal X$, then there exists a Borel measurable function $f : \mathcal X \to \mathcal Y$ such that $Y=f(X)$ a.s.
\end{lemmaA}

\end{document}

%% file: preamble.tex
\usepackage{xcolor}
\usepackage{bbm}
\usepackage{mathtools}
\usepackage{bm}
\usepackage{mathrsfs}
\usepackage{enumitem}
\usepackage{amsmath,amssymb,amsthm}

\providebool{cleanbuild}
\newcommand{\neutralshorten}[1]{%
  \ifbool{cleanbuild}{\ifvmode\else\unskip\fi\ignorespaces}{#1}%
}

\providecommand{\comment}{}
\renewcommand{\comment}[2]{%
  \neutralshorten{%
    \ifmmode%
      \text{\textcolor{#1}{\sffamily\small[{#2}]}}%
    \else%
      \noindent%
      \textcolor{#1}{\sffamily\small[{#2\unskip}]}%
    \fi%
  }%
}

\newtheorem{theorem}{Theorem}
\newtheorem{assumption}{Assumption}
\newtheorem{remark}{Remark}
\newtheorem{definition}{Definition}

\newtheorem{theoremA}{Theorem}[section] %
\newtheorem{lemmaA}[theoremA]{Lemma} %

\newcommand{\abs}[1]{| #1|}

\newcommand{\bone}{\mathbbm{1}}

\newcommand{\ee}{\mathbb{E}}

\newcommand{\mx}{\mathcal{X}}

\newcommand{\my}{\mathcal{Y}}

\newcommand{\pp}{\mathbb{P}}

\newcommand{\rr}{\mathbb{R}}

\newcommand{\N}{\mathbb{N}} 

\newcommand{\fip}[2]{\left\langle{#1}\right\rangle_{#2}} %

\usepackage{stackengine}

\newcommand\cleq{\mathrel{\stackunder{$<$}{$\sim$}}}

\newcommand{\R}{\mathbb{R}}

%% file: fig/tikz_combined/power_final.tex
\begin{tikzpicture}[x=1pt,y=1pt]
\definecolor{fillColor}{RGB}{255,255,255}
\path[use as bounding box,fill=fillColor,fill opacity=0.00] (0,0) rectangle (216.81,108.41);
\begin{scope}
\path[clip] (  0.00,  0.00) rectangle (216.81,108.41);
\definecolor{drawColor}{RGB}{0,0,0}
\definecolor{fillColor}{RGB}{255,255,255}

\path[draw=drawColor,line width= 1.1pt,line join=round,line cap=round,fill=fillColor] ( -2.60, 94.50) rectangle (219.41,112.46);
\end{scope}
\begin{scope}
\path[clip] (  0.00,  0.00) rectangle (216.81,108.41);
\definecolor{drawColor}{RGB}{0,0,0}
\definecolor{fillColor}{RGB}{255,255,255}

\path[draw=drawColor,line width= 1.1pt,line join=round,line cap=round,fill=fillColor] ( -2.60, 94.50) rectangle (219.41,112.46);
\end{scope}
\begin{scope}
\path[clip] (  0.00,  0.00) rectangle (216.81,108.41);
\definecolor{drawColor}{RGB}{102,166,30}

\path[draw=drawColor,line width= 0.5pt,line join=round] (  0.54,103.48) -- (  9.65,103.48);
\end{scope}
\begin{scope}
\path[clip] (  0.00,  0.00) rectangle (216.81,108.41);
\definecolor{fillColor}{RGB}{102,166,30}

\path[fill=fillColor] (  2.99,101.37) --
	(  7.20,101.37) --
	(  7.20,105.58) --
	(  2.99,105.58) --
	cycle;
\end{scope}
\begin{scope}
\path[clip] (  0.00,  0.00) rectangle (216.81,108.41);
\definecolor{drawColor}{RGB}{31,119,180}

\path[draw=drawColor,line width= 0.5pt,line join=round] ( 24.29,103.48) -- ( 33.39,103.48);
\end{scope}
\begin{scope}
\path[clip] (  0.00,  0.00) rectangle (216.81,108.41);
\definecolor{fillColor}{RGB}{31,119,180}

\path[fill=fillColor] ( 26.73,103.48) --
	( 28.84,105.58) --
	( 30.94,103.48) --
	( 28.84,101.37) --
	cycle;
\end{scope}
\begin{scope}
\path[clip] (  0.00,  0.00) rectangle (216.81,108.41);
\definecolor{drawColor}{RGB}{246,141,98}

\path[draw=drawColor,line width= 0.5pt,line join=round] ( 47.19,103.48) -- ( 56.29,103.48);
\end{scope}
\begin{scope}
\path[clip] (  0.00,  0.00) rectangle (216.81,108.41);
\definecolor{drawColor}{RGB}{246,141,98}

\path[draw=drawColor,line width= 0.5pt,line join=round,line cap=round] ( 49.64,101.37) -- ( 53.84,105.58);

\path[draw=drawColor,line width= 0.5pt,line join=round,line cap=round] ( 49.64,105.58) -- ( 53.84,101.37);
\end{scope}
\begin{scope}
\path[clip] (  0.00,  0.00) rectangle (216.81,108.41);
\definecolor{drawColor}{RGB}{27,158,119}

\path[draw=drawColor,line width= 0.5pt,line join=round] ( 73.45,103.48) -- ( 82.56,103.48);
\end{scope}
\begin{scope}
\path[clip] (  0.00,  0.00) rectangle (216.81,108.41);
\definecolor{fillColor}{RGB}{27,158,119}

\path[fill=fillColor] ( 78.00,103.48) circle (  2.10);
\end{scope}
\begin{scope}
\path[clip] (  0.00,  0.00) rectangle (216.81,108.41);
\definecolor{drawColor}{RGB}{117,112,179}

\path[draw=drawColor,line width= 0.5pt,line join=round] (110.22,103.48) -- (119.32,103.48);
\end{scope}
\begin{scope}
\path[clip] (  0.00,  0.00) rectangle (216.81,108.41);
\definecolor{drawColor}{RGB}{117,112,179}

\path[draw=drawColor,line width= 0.5pt,line join=round,line cap=round] (114.77,103.48) circle (  2.10);
\end{scope}
\begin{scope}
\path[clip] (  0.00,  0.00) rectangle (216.81,108.41);
\definecolor{drawColor}{RGB}{8,95,2}

\path[draw=drawColor,line width= 0.5pt,line join=round] (148.25,103.48) -- (157.35,103.48);
\end{scope}
\begin{scope}
\path[clip] (  0.00,  0.00) rectangle (216.81,108.41);
\definecolor{fillColor}{RGB}{8,95,2}

\path[fill=fillColor] (152.80,106.75) --
	(155.63,101.84) --
	(149.97,101.84) --
	cycle;
\end{scope}
\begin{scope}
\path[clip] (  0.00,  0.00) rectangle (216.81,108.41);
\definecolor{drawColor}{RGB}{231,41,138}

\path[draw=drawColor,line width= 0.5pt,line join=round] (185.02,103.48) -- (194.12,103.48);
\end{scope}
\begin{scope}
\path[clip] (  0.00,  0.00) rectangle (216.81,108.41);
\definecolor{drawColor}{RGB}{231,41,138}

\path[draw=drawColor,line width= 0.5pt,line join=round,line cap=round] (189.57,106.75) --
	(192.40,101.84) --
	(186.73,101.84) --
	cycle;
\end{scope}
\begin{scope}
\path[clip] (  0.00,  0.00) rectangle (216.81,108.41);
\definecolor{drawColor}{RGB}{0,0,0}

\node[text=drawColor,anchor=base west,inner sep=0pt, outer sep=0pt, scale=  0.50] at ( 11.29,101.76) {\Large $\xi_n$};
\end{scope}
\begin{scope}
\path[clip] (  0.00,  0.00) rectangle (216.81,108.41);
\definecolor{drawColor}{RGB}{0,0,0}

\node[text=drawColor,anchor=base west,inner sep=0pt, outer sep=0pt, scale=  0.50] at ( 35.03,101.76) {\Large $T_n$};
\end{scope}
\begin{scope}
\path[clip] (  0.00,  0.00) rectangle (216.81,108.41);
\definecolor{drawColor}{RGB}{0,0,0}

\node[text=drawColor,anchor=base west,inner sep=0pt, outer sep=0pt, scale=  0.50] at ( 57.93,101.76) {\Large $\nu_n$};
\end{scope}
\begin{scope}
\path[clip] (  0.00,  0.00) rectangle (216.81,108.41);
\definecolor{drawColor}{RGB}{0,0,0}

\node[text=drawColor,anchor=base west,inner sep=0pt, outer sep=0pt, scale=  0.50] at ( 84.19,101.76) {\Large $\hat\eta^{\operatorname{RKHS}}$};
\end{scope}
\begin{scope}
\path[clip] (  0.00,  0.00) rectangle (216.81,108.41);
\definecolor{drawColor}{RGB}{0,0,0}

\node[text=drawColor,anchor=base west,inner sep=0pt, outer sep=0pt, scale=  0.50] at (120.96,101.76)  {\Large $\hat\eta^{K\operatorname{-NN}}$};
\end{scope}
\begin{scope}
\path[clip] (  0.00,  0.00) rectangle (216.81,108.41);
\definecolor{drawColor}{RGB}{0,0,0}

\node[text=drawColor,anchor=base west,inner sep=0pt, outer sep=0pt, scale=  0.50] at (158.99,101.76)  {\Large $\hat D^{\operatorname{RKHS}}$};
\end{scope}
\begin{scope}
\path[clip] (  0.00,  0.00) rectangle (216.81,108.41);
\definecolor{drawColor}{RGB}{0,0,0}

\node[text=drawColor,anchor=base west,inner sep=0pt, outer sep=0pt, scale=  0.50] at (195.76,101.76) {\Large $\hat D^{K\operatorname{-NN}}$};
\end{scope}
\begin{scope}
\path[clip] (  0.00,  0.00) rectangle (108.40, 98.55);
\definecolor{drawColor}{RGB}{255,255,255}
\definecolor{fillColor}{RGB}{255,255,255}

\path[draw=drawColor,line width= 0.5pt,line join=round,line cap=round,fill=fillColor] (  0.00,  0.00) rectangle (108.41, 98.55);
\end{scope}
\begin{scope}
\path[clip] ( 21.62, 16.56) rectangle (108.31, 98.05);
\definecolor{fillColor}{RGB}{255,255,255}

\path[fill=fillColor] ( 21.62, 16.56) rectangle (108.31, 98.05);
\definecolor{drawColor}{RGB}{102,166,30}

\path[draw=drawColor,line width= 0.2pt,line join=round] ( 25.56, 88.42) --
	( 33.44, 49.01) --
	( 41.32, 39.53) --
	( 49.20, 33.01) --
	( 57.08, 28.57) --
	( 64.96, 26.05) --
	( 72.84, 25.01) --
	( 80.72, 24.12) --
	( 88.60, 23.38) --
	( 96.48, 22.94) --
	(104.36, 22.94);
\definecolor{drawColor}{RGB}{31,119,180}

\path[draw=drawColor,line width= 0.2pt,line join=round] ( 25.56, 67.53) --
	( 33.44, 37.45) --
	( 41.32, 31.82) --
	( 49.20, 29.16) --
	( 57.08, 27.23) --
	( 64.96, 26.19) --
	( 72.84, 25.16) --
	( 80.72, 24.57) --
	( 88.60, 23.97) --
	( 96.48, 24.27) --
	(104.36, 23.82);
\definecolor{drawColor}{RGB}{246,141,98}

\path[draw=drawColor,line width= 0.2pt,line join=round] ( 25.56, 83.53) --
	( 33.44, 48.12) --
	( 41.32, 39.83) --
	( 49.20, 33.75) --
	( 57.08, 29.60) --
	( 64.96, 27.38) --
	( 72.84, 25.31) --
	( 80.72, 24.42) --
	( 88.60, 23.23) --
	( 96.48, 22.94) --
	(104.36, 22.94);
\definecolor{drawColor}{RGB}{27,158,119}

\path[draw=drawColor,line width= 0.2pt,line join=round] ( 25.56, 94.35) --
	( 33.44, 94.35) --
	( 41.32, 94.35) --
	( 49.20, 93.61) --
	( 57.08, 81.75) --
	( 64.96, 58.79) --
	( 72.84, 38.64) --
	( 80.72, 30.05) --
	( 88.60, 25.60) --
	( 96.48, 24.27) --
	(104.36, 23.97);
\definecolor{drawColor}{RGB}{117,112,179}

\path[draw=drawColor,line width= 0.2pt,line join=round] ( 25.56, 94.35) --
	( 33.44, 94.35) --
	( 41.32, 94.05) --
	( 49.20, 85.01) --
	( 57.08, 60.57) --
	( 64.96, 41.60) --
	( 72.84, 30.20) --
	( 80.72, 26.05) --
	( 88.60, 24.42) --
	( 96.48, 23.82) --
	(104.36, 23.68);
\definecolor{drawColor}{RGB}{8,95,2}

\path[draw=drawColor,line width= 0.2pt,line join=round] ( 25.56, 94.35) --
	( 33.44, 94.35) --
	( 41.32, 94.35) --
	( 49.20, 94.35) --
	( 57.08, 93.46) --
	( 64.96, 75.53) --
	( 72.84, 50.34) --
	( 80.72, 32.27) --
	( 88.60, 25.90) --
	( 96.48, 23.53) --
	(104.36, 23.38);
\definecolor{drawColor}{RGB}{231,41,138}

\path[draw=drawColor,line width= 0.2pt,line join=round] ( 25.56, 94.35) --
	( 33.44, 94.35) --
	( 41.32, 94.35) --
	( 49.20, 92.42) --
	( 57.08, 76.12) --
	( 64.96, 52.42) --
	( 72.84, 35.97) --
	( 80.72, 27.08) --
	( 88.60, 24.57) --
	( 96.48, 23.82) --
	(104.36, 23.82);
\definecolor{fillColor}{RGB}{102,166,30}

\path[fill=fillColor] ( 23.99, 86.85) --
	( 27.13, 86.85) --
	( 27.13, 89.99) --
	( 23.99, 89.99) --
	cycle;
\definecolor{fillColor}{RGB}{31,119,180}

\path[fill=fillColor] ( 23.99, 67.53) --
	( 25.56, 69.10) --
	( 27.13, 67.53) --
	( 25.56, 65.96) --
	cycle;
\definecolor{drawColor}{RGB}{246,141,98}

\path[draw=drawColor,line width= 0.5pt,line join=round,line cap=round] ( 23.99, 81.96) -- ( 27.13, 85.10);

\path[draw=drawColor,line width= 0.5pt,line join=round,line cap=round] ( 23.99, 85.10) -- ( 27.13, 81.96);
\definecolor{fillColor}{RGB}{27,158,119}

\path[fill=fillColor] ( 25.56, 94.35) circle (  1.57);
\definecolor{drawColor}{RGB}{117,112,179}

\path[draw=drawColor,line width= 0.5pt,line join=round,line cap=round] ( 25.56, 94.35) circle (  1.57);
\definecolor{fillColor}{RGB}{8,95,2}

\path[fill=fillColor] ( 25.56, 96.79) --
	( 27.68, 93.13) --
	( 23.45, 93.13) --
	cycle;
\definecolor{drawColor}{RGB}{231,41,138}

\path[draw=drawColor,line width= 0.5pt,line join=round,line cap=round] ( 25.56, 96.79) --
	( 27.68, 93.13) --
	( 23.45, 93.13) --
	cycle;
\definecolor{fillColor}{RGB}{102,166,30}

\path[fill=fillColor] ( 31.88, 47.44) --
	( 35.01, 47.44) --
	( 35.01, 50.58) --
	( 31.88, 50.58) --
	cycle;
\definecolor{fillColor}{RGB}{31,119,180}

\path[fill=fillColor] ( 31.88, 37.45) --
	( 33.44, 39.02) --
	( 35.01, 37.45) --
	( 33.44, 35.89) --
	cycle;
\definecolor{drawColor}{RGB}{246,141,98}

\path[draw=drawColor,line width= 0.5pt,line join=round,line cap=round] ( 31.88, 46.55) -- ( 35.01, 49.69);

\path[draw=drawColor,line width= 0.5pt,line join=round,line cap=round] ( 31.88, 49.69) -- ( 35.01, 46.55);
\definecolor{fillColor}{RGB}{27,158,119}

\path[fill=fillColor] ( 33.44, 94.35) circle (  1.57);
\definecolor{drawColor}{RGB}{117,112,179}

\path[draw=drawColor,line width= 0.5pt,line join=round,line cap=round] ( 33.44, 94.35) circle (  1.57);
\definecolor{fillColor}{RGB}{8,95,2}

\path[fill=fillColor] ( 33.44, 96.79) --
	( 35.56, 93.13) --
	( 31.33, 93.13) --
	cycle;
\definecolor{drawColor}{RGB}{231,41,138}

\path[draw=drawColor,line width= 0.5pt,line join=round,line cap=round] ( 33.44, 96.79) --
	( 35.56, 93.13) --
	( 31.33, 93.13) --
	cycle;
\definecolor{fillColor}{RGB}{102,166,30}

\path[fill=fillColor] ( 39.76, 37.96) --
	( 42.89, 37.96) --
	( 42.89, 41.10) --
	( 39.76, 41.10) --
	cycle;
\definecolor{fillColor}{RGB}{31,119,180}

\path[fill=fillColor] ( 39.76, 31.82) --
	( 41.32, 33.39) --
	( 42.89, 31.82) --
	( 41.32, 30.26) --
	cycle;
\definecolor{drawColor}{RGB}{246,141,98}

\path[draw=drawColor,line width= 0.5pt,line join=round,line cap=round] ( 39.76, 38.26) -- ( 42.89, 41.39);

\path[draw=drawColor,line width= 0.5pt,line join=round,line cap=round] ( 39.76, 41.39) -- ( 42.89, 38.26);
\definecolor{fillColor}{RGB}{27,158,119}

\path[fill=fillColor] ( 41.32, 94.35) circle (  1.57);
\definecolor{drawColor}{RGB}{117,112,179}

\path[draw=drawColor,line width= 0.5pt,line join=round,line cap=round] ( 41.32, 94.05) circle (  1.57);
\definecolor{fillColor}{RGB}{8,95,2}

\path[fill=fillColor] ( 41.32, 96.79) --
	( 43.44, 93.13) --
	( 39.21, 93.13) --
	cycle;
\definecolor{drawColor}{RGB}{231,41,138}

\path[draw=drawColor,line width= 0.5pt,line join=round,line cap=round] ( 41.32, 96.79) --
	( 43.44, 93.13) --
	( 39.21, 93.13) --
	cycle;
\definecolor{fillColor}{RGB}{102,166,30}

\path[fill=fillColor] ( 47.64, 31.44) --
	( 50.77, 31.44) --
	( 50.77, 34.58) --
	( 47.64, 34.58) --
	cycle;
\definecolor{fillColor}{RGB}{31,119,180}

\path[fill=fillColor] ( 47.64, 29.16) --
	( 49.20, 30.73) --
	( 50.77, 29.16) --
	( 49.20, 27.59) --
	cycle;
\definecolor{drawColor}{RGB}{246,141,98}

\path[draw=drawColor,line width= 0.5pt,line join=round,line cap=round] ( 47.64, 32.18) -- ( 50.77, 35.32);

\path[draw=drawColor,line width= 0.5pt,line join=round,line cap=round] ( 47.64, 35.32) -- ( 50.77, 32.18);
\definecolor{fillColor}{RGB}{27,158,119}

\path[fill=fillColor] ( 49.20, 93.61) circle (  1.57);
\definecolor{drawColor}{RGB}{117,112,179}

\path[draw=drawColor,line width= 0.5pt,line join=round,line cap=round] ( 49.20, 85.01) circle (  1.57);
\definecolor{fillColor}{RGB}{8,95,2}

\path[fill=fillColor] ( 49.20, 96.79) --
	( 51.32, 93.13) --
	( 47.09, 93.13) --
	cycle;
\definecolor{drawColor}{RGB}{231,41,138}

\path[draw=drawColor,line width= 0.5pt,line join=round,line cap=round] ( 49.20, 94.86) --
	( 51.32, 91.20) --
	( 47.09, 91.20) --
	cycle;
\definecolor{fillColor}{RGB}{102,166,30}

\path[fill=fillColor] ( 55.52, 27.00) --
	( 58.65, 27.00) --
	( 58.65, 30.13) --
	( 55.52, 30.13) --
	cycle;
\definecolor{fillColor}{RGB}{31,119,180}

\path[fill=fillColor] ( 55.52, 27.23) --
	( 57.08, 28.80) --
	( 58.65, 27.23) --
	( 57.08, 25.66) --
	cycle;
\definecolor{drawColor}{RGB}{246,141,98}

\path[draw=drawColor,line width= 0.5pt,line join=round,line cap=round] ( 55.52, 28.03) -- ( 58.65, 31.17);

\path[draw=drawColor,line width= 0.5pt,line join=round,line cap=round] ( 55.52, 31.17) -- ( 58.65, 28.03);
\definecolor{fillColor}{RGB}{27,158,119}

\path[fill=fillColor] ( 57.08, 81.75) circle (  1.57);
\definecolor{drawColor}{RGB}{117,112,179}

\path[draw=drawColor,line width= 0.5pt,line join=round,line cap=round] ( 57.08, 60.57) circle (  1.57);
\definecolor{fillColor}{RGB}{8,95,2}

\path[fill=fillColor] ( 57.08, 95.90) --
	( 59.20, 92.24) --
	( 54.97, 92.24) --
	cycle;
\definecolor{drawColor}{RGB}{231,41,138}

\path[draw=drawColor,line width= 0.5pt,line join=round,line cap=round] ( 57.08, 78.56) --
	( 59.20, 74.90) --
	( 54.97, 74.90) --
	cycle;
\definecolor{fillColor}{RGB}{102,166,30}

\path[fill=fillColor] ( 63.40, 24.48) --
	( 66.53, 24.48) --
	( 66.53, 27.62) --
	( 63.40, 27.62) --
	cycle;
\definecolor{fillColor}{RGB}{31,119,180}

\path[fill=fillColor] ( 63.40, 26.19) --
	( 64.96, 27.76) --
	( 66.53, 26.19) --
	( 64.96, 24.63) --
	cycle;
\definecolor{drawColor}{RGB}{246,141,98}

\path[draw=drawColor,line width= 0.5pt,line join=round,line cap=round] ( 63.40, 25.81) -- ( 66.53, 28.95);

\path[draw=drawColor,line width= 0.5pt,line join=round,line cap=round] ( 63.40, 28.95) -- ( 66.53, 25.81);
\definecolor{fillColor}{RGB}{27,158,119}

\path[fill=fillColor] ( 64.96, 58.79) circle (  1.57);
\definecolor{drawColor}{RGB}{117,112,179}

\path[draw=drawColor,line width= 0.5pt,line join=round,line cap=round] ( 64.96, 41.60) circle (  1.57);
\definecolor{fillColor}{RGB}{8,95,2}

\path[fill=fillColor] ( 64.96, 77.97) --
	( 67.08, 74.31) --
	( 62.85, 74.31) --
	cycle;
\definecolor{drawColor}{RGB}{231,41,138}

\path[draw=drawColor,line width= 0.5pt,line join=round,line cap=round] ( 64.96, 54.86) --
	( 67.08, 51.20) --
	( 62.85, 51.20) --
	cycle;
\definecolor{fillColor}{RGB}{102,166,30}

\path[fill=fillColor] ( 71.28, 23.44) --
	( 74.41, 23.44) --
	( 74.41, 26.58) --
	( 71.28, 26.58) --
	cycle;
\definecolor{fillColor}{RGB}{31,119,180}

\path[fill=fillColor] ( 71.28, 25.16) --
	( 72.84, 26.73) --
	( 74.41, 25.16) --
	( 72.84, 23.59) --
	cycle;
\definecolor{drawColor}{RGB}{246,141,98}

\path[draw=drawColor,line width= 0.5pt,line join=round,line cap=round] ( 71.28, 23.74) -- ( 74.41, 26.87);

\path[draw=drawColor,line width= 0.5pt,line join=round,line cap=round] ( 71.28, 26.87) -- ( 74.41, 23.74);
\definecolor{fillColor}{RGB}{27,158,119}

\path[fill=fillColor] ( 72.84, 38.64) circle (  1.57);
\definecolor{drawColor}{RGB}{117,112,179}

\path[draw=drawColor,line width= 0.5pt,line join=round,line cap=round] ( 72.84, 30.20) circle (  1.57);
\definecolor{fillColor}{RGB}{8,95,2}

\path[fill=fillColor] ( 72.84, 52.78) --
	( 74.96, 49.12) --
	( 70.73, 49.12) --
	cycle;
\definecolor{drawColor}{RGB}{231,41,138}

\path[draw=drawColor,line width= 0.5pt,line join=round,line cap=round] ( 72.84, 38.41) --
	( 74.96, 34.75) --
	( 70.73, 34.75) --
	cycle;
\definecolor{fillColor}{RGB}{102,166,30}

\path[fill=fillColor] ( 79.16, 22.55) --
	( 82.29, 22.55) --
	( 82.29, 25.69) --
	( 79.16, 25.69) --
	cycle;
\definecolor{fillColor}{RGB}{31,119,180}

\path[fill=fillColor] ( 79.16, 24.57) --
	( 80.72, 26.13) --
	( 82.29, 24.57) --
	( 80.72, 23.00) --
	cycle;
\definecolor{drawColor}{RGB}{246,141,98}

\path[draw=drawColor,line width= 0.5pt,line join=round,line cap=round] ( 79.16, 22.85) -- ( 82.29, 25.99);

\path[draw=drawColor,line width= 0.5pt,line join=round,line cap=round] ( 79.16, 25.99) -- ( 82.29, 22.85);
\definecolor{fillColor}{RGB}{27,158,119}

\path[fill=fillColor] ( 80.72, 30.05) circle (  1.57);
\definecolor{drawColor}{RGB}{117,112,179}

\path[draw=drawColor,line width= 0.5pt,line join=round,line cap=round] ( 80.72, 26.05) circle (  1.57);
\definecolor{fillColor}{RGB}{8,95,2}

\path[fill=fillColor] ( 80.72, 34.71) --
	( 82.84, 31.05) --
	( 78.61, 31.05) --
	cycle;
\definecolor{drawColor}{RGB}{231,41,138}

\path[draw=drawColor,line width= 0.5pt,line join=round,line cap=round] ( 80.72, 29.52) --
	( 82.84, 25.86) --
	( 78.61, 25.86) --
	cycle;
\definecolor{fillColor}{RGB}{102,166,30}

\path[fill=fillColor] ( 87.04, 21.81) --
	( 90.17, 21.81) --
	( 90.17, 24.95) --
	( 87.04, 24.95) --
	cycle;
\definecolor{fillColor}{RGB}{31,119,180}

\path[fill=fillColor] ( 87.04, 23.97) --
	( 88.60, 25.54) --
	( 90.17, 23.97) --
	( 88.60, 22.40) --
	cycle;
\definecolor{drawColor}{RGB}{246,141,98}

\path[draw=drawColor,line width= 0.5pt,line join=round,line cap=round] ( 87.04, 21.66) -- ( 90.17, 24.80);

\path[draw=drawColor,line width= 0.5pt,line join=round,line cap=round] ( 87.04, 24.80) -- ( 90.17, 21.66);
\definecolor{fillColor}{RGB}{27,158,119}

\path[fill=fillColor] ( 88.60, 25.60) circle (  1.57);
\definecolor{drawColor}{RGB}{117,112,179}

\path[draw=drawColor,line width= 0.5pt,line join=round,line cap=round] ( 88.60, 24.42) circle (  1.57);
\definecolor{fillColor}{RGB}{8,95,2}

\path[fill=fillColor] ( 88.60, 28.34) --
	( 90.72, 24.68) --
	( 86.49, 24.68) --
	cycle;
\definecolor{drawColor}{RGB}{231,41,138}

\path[draw=drawColor,line width= 0.5pt,line join=round,line cap=round] ( 88.60, 27.01) --
	( 90.72, 23.35) --
	( 86.49, 23.35) --
	cycle;
\definecolor{fillColor}{RGB}{102,166,30}

\path[fill=fillColor] ( 94.92, 21.37) --
	( 98.05, 21.37) --
	( 98.05, 24.50) --
	( 94.92, 24.50) --
	cycle;
\definecolor{fillColor}{RGB}{31,119,180}

\path[fill=fillColor] ( 94.92, 24.27) --
	( 96.48, 25.84) --
	( 98.05, 24.27) --
	( 96.48, 22.70) --
	cycle;
\definecolor{drawColor}{RGB}{246,141,98}

\path[draw=drawColor,line width= 0.5pt,line join=round,line cap=round] ( 94.92, 21.37) -- ( 98.05, 24.50);

\path[draw=drawColor,line width= 0.5pt,line join=round,line cap=round] ( 94.92, 24.50) -- ( 98.05, 21.37);
\definecolor{fillColor}{RGB}{27,158,119}

\path[fill=fillColor] ( 96.48, 24.27) circle (  1.57);
\definecolor{drawColor}{RGB}{117,112,179}

\path[draw=drawColor,line width= 0.5pt,line join=round,line cap=round] ( 96.48, 23.82) circle (  1.57);
\definecolor{fillColor}{RGB}{8,95,2}

\path[fill=fillColor] ( 96.48, 25.97) --
	( 98.60, 22.31) --
	( 94.37, 22.31) --
	cycle;
\definecolor{drawColor}{RGB}{231,41,138}

\path[draw=drawColor,line width= 0.5pt,line join=round,line cap=round] ( 96.48, 26.26) --
	( 98.60, 22.60) --
	( 94.37, 22.60) --
	cycle;
\definecolor{fillColor}{RGB}{102,166,30}

\path[fill=fillColor] (102.80, 21.37) --
	(105.93, 21.37) --
	(105.93, 24.50) --
	(102.80, 24.50) --
	cycle;
\definecolor{fillColor}{RGB}{31,119,180}

\path[fill=fillColor] (102.80, 23.82) --
	(104.36, 25.39) --
	(105.93, 23.82) --
	(104.36, 22.26) --
	cycle;
\definecolor{drawColor}{RGB}{246,141,98}

\path[draw=drawColor,line width= 0.5pt,line join=round,line cap=round] (102.80, 21.37) -- (105.93, 24.50);

\path[draw=drawColor,line width= 0.5pt,line join=round,line cap=round] (102.80, 24.50) -- (105.93, 21.37);
\definecolor{fillColor}{RGB}{27,158,119}

\path[fill=fillColor] (104.36, 23.97) circle (  1.57);
\definecolor{drawColor}{RGB}{117,112,179}

\path[draw=drawColor,line width= 0.5pt,line join=round,line cap=round] (104.36, 23.68) circle (  1.57);
\definecolor{fillColor}{RGB}{8,95,2}

\path[fill=fillColor] (104.36, 25.82) --
	(106.48, 22.16) --
	(102.25, 22.16) --
	cycle;
\definecolor{drawColor}{RGB}{231,41,138}

\path[draw=drawColor,line width= 0.5pt,line join=round,line cap=round] (104.36, 26.26) --
	(106.48, 22.60) --
	(102.25, 22.60) --
	cycle;
\definecolor{drawColor}{gray}{0.20}

\path[draw=drawColor,line width= 0.5pt,line join=round,line cap=round] ( 21.62, 16.56) rectangle (108.31, 98.05);
\end{scope}
\begin{scope}
\path[clip] (  0.00,  0.00) rectangle (216.81,108.41);
\definecolor{drawColor}{gray}{0.30}

\node[text=drawColor,anchor=base east,inner sep=0pt, outer sep=0pt, scale=  0.50] at ( 17.57, 18.55) {0.00};

\node[text=drawColor,anchor=base east,inner sep=0pt, outer sep=0pt, scale=  0.50] at ( 17.57, 37.07) {0.25};

\node[text=drawColor,anchor=base east,inner sep=0pt, outer sep=0pt, scale=  0.50] at ( 17.57, 55.59) {0.50};

\node[text=drawColor,anchor=base east,inner sep=0pt, outer sep=0pt, scale=  0.50] at ( 17.57, 74.10) {0.75};

\node[text=drawColor,anchor=base east,inner sep=0pt, outer sep=0pt, scale=  0.50] at ( 17.57, 92.62) {1.00};
\end{scope}
\begin{scope}
\path[clip] (  0.00,  0.00) rectangle (216.81,108.41);
\definecolor{drawColor}{gray}{0.20}

\path[draw=drawColor,line width= 0.5pt,line join=round] ( 19.37, 20.27) --
	( 21.62, 20.27);

\path[draw=drawColor,line width= 0.5pt,line join=round] ( 19.37, 38.79) --
	( 21.62, 38.79);

\path[draw=drawColor,line width= 0.5pt,line join=round] ( 19.37, 57.31) --
	( 21.62, 57.31);

\path[draw=drawColor,line width= 0.5pt,line join=round] ( 19.37, 75.83) --
	( 21.62, 75.83);

\path[draw=drawColor,line width= 0.5pt,line join=round] ( 19.37, 94.35) --
	( 21.62, 94.35);
\end{scope}
\begin{scope}
\path[clip] (  0.00,  0.00) rectangle (216.81,108.41);
\definecolor{drawColor}{gray}{0.20}

\path[draw=drawColor,line width= 0.5pt,line join=round] ( 25.56, 14.31) --
	( 25.56, 16.56);

\path[draw=drawColor,line width= 0.5pt,line join=round] ( 45.26, 14.31) --
	( 45.26, 16.56);

\path[draw=drawColor,line width= 0.5pt,line join=round] ( 64.96, 14.31) --
	( 64.96, 16.56);

\path[draw=drawColor,line width= 0.5pt,line join=round] ( 84.66, 14.31) --
	( 84.66, 16.56);

\path[draw=drawColor,line width= 0.5pt,line join=round] (104.36, 14.31) --
	(104.36, 16.56);
\end{scope}
\begin{scope}
\path[clip] (  0.00,  0.00) rectangle (216.81,108.41);
\definecolor{drawColor}{gray}{0.30}

\node[text=drawColor,anchor=base,inner sep=0pt, outer sep=0pt, scale=  0.50] at ( 25.56,  9.07) {0.00};

\node[text=drawColor,anchor=base,inner sep=0pt, outer sep=0pt, scale=  0.50] at ( 45.26,  9.07) {0.25};

\node[text=drawColor,anchor=base,inner sep=0pt, outer sep=0pt, scale=  0.50] at ( 64.96,  9.07) {0.50};

\node[text=drawColor,anchor=base,inner sep=0pt, outer sep=0pt, scale=  0.50] at ( 84.66,  9.07) {0.75};

\node[text=drawColor,anchor=base,inner sep=0pt, outer sep=0pt, scale=  0.50] at (104.36,  9.07) {1.00};
\end{scope}
\begin{scope}
\path[clip] (  0.00,  0.00) rectangle (216.81,108.41);
\definecolor{drawColor}{RGB}{0,0,0}

\node[text=drawColor,anchor=base,inner sep=0pt, outer sep=0pt, scale=  0.60] at ( 64.96,  1.42) {Homoscedasticity Level $(\lambda)$};
\end{scope}
\begin{scope}
\path[clip] (  0.00,  0.00) rectangle (216.81,108.41);
\definecolor{drawColor}{RGB}{0,0,0}

\node[text=drawColor,rotate= 90.00,anchor=base,inner sep=0pt, outer sep=0pt, scale=  0.60] at (  4.23, 57.31) {Power};
\end{scope}
\begin{scope}
\path[clip] (108.41,  0.00) rectangle (216.81, 98.55);
\definecolor{drawColor}{RGB}{255,255,255}
\definecolor{fillColor}{RGB}{255,255,255}

\path[draw=drawColor,line width= 0.5pt,line join=round,line cap=round,fill=fillColor] (108.41,  0.00) rectangle (216.81, 98.55);
\end{scope}
\begin{scope}
\path[clip] (130.03, 16.56) rectangle (216.71, 98.05);
\definecolor{fillColor}{RGB}{255,255,255}

\path[fill=fillColor] (130.03, 16.56) rectangle (216.71, 98.05);
\definecolor{drawColor}{RGB}{27,158,119}

\path[draw=drawColor,line width= 0.2pt,line join=round] (133.97, 78.79) --
	(141.85, 73.46) --
	(149.73, 58.05) --
	(157.61, 34.05) --
	(165.49, 21.60) --
	(173.37, 20.27) --
	(181.25, 20.27) --
	(189.13, 20.27) --
	(197.01, 20.27) --
	(204.89, 20.27) --
	(212.77, 20.27);
\definecolor{drawColor}{RGB}{117,112,179}

\path[draw=drawColor,line width= 0.2pt,line join=round] (133.97, 94.35) --
	(141.85, 94.35) --
	(149.73, 94.35) --
	(157.61, 87.23) --
	(165.49, 40.57) --
	(173.37, 20.86) --
	(181.25, 20.27) --
	(189.13, 20.27) --
	(197.01, 20.27) --
	(204.89, 20.27) --
	(212.77, 20.27);
\definecolor{drawColor}{RGB}{8,95,2}

\path[draw=drawColor,line width= 0.2pt,line join=round] (133.97, 94.35) --
	(141.85, 94.35) --
	(149.73, 94.35) --
	(157.61, 94.35) --
	(165.49, 94.35) --
	(173.37, 94.35) --
	(181.25, 94.35) --
	(189.13, 94.35) --
	(197.01, 94.35) --
	(204.89, 94.05) --
	(212.77, 93.61);
\definecolor{drawColor}{RGB}{231,41,138}

\path[draw=drawColor,line width= 0.2pt,line join=round] (133.97, 94.35) --
	(141.85, 94.35) --
	(149.73, 94.35) --
	(157.61, 94.35) --
	(165.49, 94.35) --
	(173.37, 94.35) --
	(181.25, 94.35) --
	(189.13, 94.35) --
	(197.01, 94.20) --
	(204.89, 94.20) --
	(212.77, 93.61);
\definecolor{fillColor}{RGB}{27,158,119}

\path[fill=fillColor] (133.97, 78.79) circle (  1.57);
\definecolor{drawColor}{RGB}{117,112,179}

\path[draw=drawColor,line width= 0.5pt,line join=round,line cap=round] (133.97, 94.35) circle (  1.57);
\definecolor{fillColor}{RGB}{8,95,2}

\path[fill=fillColor] (133.97, 96.79) --
	(136.08, 93.13) --
	(131.86, 93.13) --
	cycle;
\definecolor{drawColor}{RGB}{231,41,138}

\path[draw=drawColor,line width= 0.5pt,line join=round,line cap=round] (133.97, 96.79) --
	(136.08, 93.13) --
	(131.86, 93.13) --
	cycle;
\definecolor{fillColor}{RGB}{27,158,119}

\path[fill=fillColor] (141.85, 73.46) circle (  1.57);
\definecolor{drawColor}{RGB}{117,112,179}

\path[draw=drawColor,line width= 0.5pt,line join=round,line cap=round] (141.85, 94.35) circle (  1.57);
\definecolor{fillColor}{RGB}{8,95,2}

\path[fill=fillColor] (141.85, 96.79) --
	(143.96, 93.13) --
	(139.74, 93.13) --
	cycle;
\definecolor{drawColor}{RGB}{231,41,138}

\path[draw=drawColor,line width= 0.5pt,line join=round,line cap=round] (141.85, 96.79) --
	(143.96, 93.13) --
	(139.74, 93.13) --
	cycle;
\definecolor{fillColor}{RGB}{27,158,119}

\path[fill=fillColor] (149.73, 58.05) circle (  1.57);
\definecolor{drawColor}{RGB}{117,112,179}

\path[draw=drawColor,line width= 0.5pt,line join=round,line cap=round] (149.73, 94.35) circle (  1.57);
\definecolor{fillColor}{RGB}{8,95,2}

\path[fill=fillColor] (149.73, 96.79) --
	(151.84, 93.13) --
	(147.62, 93.13) --
	cycle;
\definecolor{drawColor}{RGB}{231,41,138}

\path[draw=drawColor,line width= 0.5pt,line join=round,line cap=round] (149.73, 96.79) --
	(151.84, 93.13) --
	(147.62, 93.13) --
	cycle;
\definecolor{fillColor}{RGB}{27,158,119}

\path[fill=fillColor] (157.61, 34.05) circle (  1.57);
\definecolor{drawColor}{RGB}{117,112,179}

\path[draw=drawColor,line width= 0.5pt,line join=round,line cap=round] (157.61, 87.23) circle (  1.57);
\definecolor{fillColor}{RGB}{8,95,2}

\path[fill=fillColor] (157.61, 96.79) --
	(159.72, 93.13) --
	(155.50, 93.13) --
	cycle;
\definecolor{drawColor}{RGB}{231,41,138}

\path[draw=drawColor,line width= 0.5pt,line join=round,line cap=round] (157.61, 96.79) --
	(159.72, 93.13) --
	(155.50, 93.13) --
	cycle;
\definecolor{fillColor}{RGB}{27,158,119}

\path[fill=fillColor] (165.49, 21.60) circle (  1.57);
\definecolor{drawColor}{RGB}{117,112,179}

\path[draw=drawColor,line width= 0.5pt,line join=round,line cap=round] (165.49, 40.57) circle (  1.57);
\definecolor{fillColor}{RGB}{8,95,2}

\path[fill=fillColor] (165.49, 96.79) --
	(167.60, 93.13) --
	(163.38, 93.13) --
	cycle;
\definecolor{drawColor}{RGB}{231,41,138}

\path[draw=drawColor,line width= 0.5pt,line join=round,line cap=round] (165.49, 96.79) --
	(167.60, 93.13) --
	(163.38, 93.13) --
	cycle;
\definecolor{fillColor}{RGB}{27,158,119}

\path[fill=fillColor] (173.37, 20.27) circle (  1.57);
\definecolor{drawColor}{RGB}{117,112,179}

\path[draw=drawColor,line width= 0.5pt,line join=round,line cap=round] (173.37, 20.86) circle (  1.57);
\definecolor{fillColor}{RGB}{8,95,2}

\path[fill=fillColor] (173.37, 96.79) --
	(175.48, 93.13) --
	(171.26, 93.13) --
	cycle;
\definecolor{drawColor}{RGB}{231,41,138}

\path[draw=drawColor,line width= 0.5pt,line join=round,line cap=round] (173.37, 96.79) --
	(175.48, 93.13) --
	(171.26, 93.13) --
	cycle;
\definecolor{fillColor}{RGB}{27,158,119}

\path[fill=fillColor] (181.25, 20.27) circle (  1.57);
\definecolor{drawColor}{RGB}{117,112,179}

\path[draw=drawColor,line width= 0.5pt,line join=round,line cap=round] (181.25, 20.27) circle (  1.57);
\definecolor{fillColor}{RGB}{8,95,2}

\path[fill=fillColor] (181.25, 96.79) --
	(183.36, 93.13) --
	(179.14, 93.13) --
	cycle;
\definecolor{drawColor}{RGB}{231,41,138}

\path[draw=drawColor,line width= 0.5pt,line join=round,line cap=round] (181.25, 96.79) --
	(183.36, 93.13) --
	(179.14, 93.13) --
	cycle;
\definecolor{fillColor}{RGB}{27,158,119}

\path[fill=fillColor] (189.13, 20.27) circle (  1.57);
\definecolor{drawColor}{RGB}{117,112,179}

\path[draw=drawColor,line width= 0.5pt,line join=round,line cap=round] (189.13, 20.27) circle (  1.57);
\definecolor{fillColor}{RGB}{8,95,2}

\path[fill=fillColor] (189.13, 96.79) --
	(191.24, 93.13) --
	(187.02, 93.13) --
	cycle;
\definecolor{drawColor}{RGB}{231,41,138}

\path[draw=drawColor,line width= 0.5pt,line join=round,line cap=round] (189.13, 96.79) --
	(191.24, 93.13) --
	(187.02, 93.13) --
	cycle;
\definecolor{fillColor}{RGB}{27,158,119}

\path[fill=fillColor] (197.01, 20.27) circle (  1.57);
\definecolor{drawColor}{RGB}{117,112,179}

\path[draw=drawColor,line width= 0.5pt,line join=round,line cap=round] (197.01, 20.27) circle (  1.57);
\definecolor{fillColor}{RGB}{8,95,2}

\path[fill=fillColor] (197.01, 96.79) --
	(199.12, 93.13) --
	(194.90, 93.13) --
	cycle;
\definecolor{drawColor}{RGB}{231,41,138}

\path[draw=drawColor,line width= 0.5pt,line join=round,line cap=round] (197.01, 96.64) --
	(199.12, 92.98) --
	(194.90, 92.98) --
	cycle;
\definecolor{fillColor}{RGB}{27,158,119}

\path[fill=fillColor] (204.89, 20.27) circle (  1.57);
\definecolor{drawColor}{RGB}{117,112,179}

\path[draw=drawColor,line width= 0.5pt,line join=round,line cap=round] (204.89, 20.27) circle (  1.57);
\definecolor{fillColor}{RGB}{8,95,2}

\path[fill=fillColor] (204.89, 96.49) --
	(207.00, 92.83) --
	(202.78, 92.83) --
	cycle;
\definecolor{drawColor}{RGB}{231,41,138}

\path[draw=drawColor,line width= 0.5pt,line join=round,line cap=round] (204.89, 96.64) --
	(207.00, 92.98) --
	(202.78, 92.98) --
	cycle;
\definecolor{fillColor}{RGB}{27,158,119}

\path[fill=fillColor] (212.77, 20.27) circle (  1.57);
\definecolor{drawColor}{RGB}{117,112,179}

\path[draw=drawColor,line width= 0.5pt,line join=round,line cap=round] (212.77, 20.27) circle (  1.57);
\definecolor{fillColor}{RGB}{8,95,2}

\path[fill=fillColor] (212.77, 96.05) --
	(214.88, 92.39) --
	(210.66, 92.39) --
	cycle;
\definecolor{drawColor}{RGB}{231,41,138}

\path[draw=drawColor,line width= 0.5pt,line join=round,line cap=round] (212.77, 96.05) --
	(214.88, 92.39) --
	(210.66, 92.39) --
	cycle;
\definecolor{drawColor}{gray}{0.20}

\path[draw=drawColor,line width= 0.5pt,line join=round,line cap=round] (130.03, 16.56) rectangle (216.71, 98.05);
\end{scope}
\begin{scope}
\path[clip] (  0.00,  0.00) rectangle (216.81,108.41);
\definecolor{drawColor}{gray}{0.30}

\node[text=drawColor,anchor=base east,inner sep=0pt, outer sep=0pt, scale=  0.50] at (125.98, 18.55) {0.00};

\node[text=drawColor,anchor=base east,inner sep=0pt, outer sep=0pt, scale=  0.50] at (125.98, 37.07) {0.25};

\node[text=drawColor,anchor=base east,inner sep=0pt, outer sep=0pt, scale=  0.50] at (125.98, 55.59) {0.50};

\node[text=drawColor,anchor=base east,inner sep=0pt, outer sep=0pt, scale=  0.50] at (125.98, 74.10) {0.75};

\node[text=drawColor,anchor=base east,inner sep=0pt, outer sep=0pt, scale=  0.50] at (125.98, 92.62) {1.00};
\end{scope}
\begin{scope}
\path[clip] (  0.00,  0.00) rectangle (216.81,108.41);
\definecolor{drawColor}{gray}{0.20}

\path[draw=drawColor,line width= 0.5pt,line join=round] (127.78, 20.27) --
	(130.03, 20.27);

\path[draw=drawColor,line width= 0.5pt,line join=round] (127.78, 38.79) --
	(130.03, 38.79);

\path[draw=drawColor,line width= 0.5pt,line join=round] (127.78, 57.31) --
	(130.03, 57.31);

\path[draw=drawColor,line width= 0.5pt,line join=round] (127.78, 75.83) --
	(130.03, 75.83);

\path[draw=drawColor,line width= 0.5pt,line join=round] (127.78, 94.35) --
	(130.03, 94.35);
\end{scope}
\begin{scope}
\path[clip] (  0.00,  0.00) rectangle (216.81,108.41);
\definecolor{drawColor}{gray}{0.20}

\path[draw=drawColor,line width= 0.5pt,line join=round] (133.97, 14.31) --
	(133.97, 16.56);

\path[draw=drawColor,line width= 0.5pt,line join=round] (153.67, 14.31) --
	(153.67, 16.56);

\path[draw=drawColor,line width= 0.5pt,line join=round] (173.37, 14.31) --
	(173.37, 16.56);

\path[draw=drawColor,line width= 0.5pt,line join=round] (193.07, 14.31) --
	(193.07, 16.56);

\path[draw=drawColor,line width= 0.5pt,line join=round] (212.77, 14.31) --
	(212.77, 16.56);
\end{scope}
\begin{scope}
\path[clip] (  0.00,  0.00) rectangle (216.81,108.41);
\definecolor{drawColor}{gray}{0.30}

\node[text=drawColor,anchor=base,inner sep=0pt, outer sep=0pt, scale=  0.50] at (133.97,  9.07) {0.00};

\node[text=drawColor,anchor=base,inner sep=0pt, outer sep=0pt, scale=  0.50] at (153.67,  9.07) {0.25};

\node[text=drawColor,anchor=base,inner sep=0pt, outer sep=0pt, scale=  0.50] at (173.37,  9.07) {0.50};

\node[text=drawColor,anchor=base,inner sep=0pt, outer sep=0pt, scale=  0.50] at (193.07,  9.07) {0.75};

\node[text=drawColor,anchor=base,inner sep=0pt, outer sep=0pt, scale=  0.50] at (212.77,  9.07) {1.00};
\end{scope}
\begin{scope}
\path[clip] (  0.00,  0.00) rectangle (216.81,108.41);
\definecolor{drawColor}{RGB}{0,0,0}

\node[text=drawColor,anchor=base,inner sep=0pt, outer sep=0pt, scale=  0.60] at (173.37,  1.42) {Noise Scale $(\lambda)$};
\end{scope}
\begin{scope}
\path[clip] (  0.00,  0.00) rectangle (216.81,108.41);
\definecolor{drawColor}{RGB}{0,0,0}

\node[text=drawColor,rotate= 90.00,anchor=base,inner sep=0pt, outer sep=0pt, scale=  0.60] at (112.64, 57.31) {Power};
\end{scope}
\end{tikzpicture}

%% file: fig/tikz_combined/power_MSD_time_final.tex
\begin{tikzpicture}[x=1pt,y=1pt]
\definecolor{fillColor}{RGB}{255,255,255}
\path[use as bounding box,fill=fillColor,fill opacity=0.00] (0,0) rectangle (216.81,108.41);
\begin{scope}
\path[clip] (  0.00,  0.00) rectangle (216.81,108.41);
\definecolor{drawColor}{RGB}{255,255,255}
\definecolor{fillColor}{RGB}{255,255,255}

\path[draw=drawColor,line width= 1.1pt,line join=round,line cap=round,fill=fillColor] (  9.28, 94.50) rectangle (207.53,112.46);
\end{scope}
\begin{scope}
\path[clip] (  0.00,  0.00) rectangle (216.81,108.41);
\definecolor{drawColor}{RGB}{255,255,255}
\definecolor{fillColor}{RGB}{255,255,255}

\path[draw=drawColor,line width= 1.1pt,line join=round,line cap=round,fill=fillColor] (  9.28, 94.50) rectangle (207.53,112.46);
\end{scope}
\begin{scope}
\path[clip] (  0.00,  0.00) rectangle (216.81,108.41);
\definecolor{drawColor}{RGB}{31,119,180}

\path[draw=drawColor,line width= 0.5pt,line join=round] ( 12.41,103.48) -- ( 21.52,103.48);
\end{scope}
\begin{scope}
\path[clip] (  0.00,  0.00) rectangle (216.81,108.41);
\definecolor{fillColor}{RGB}{31,119,180}

\path[fill=fillColor] ( 14.86,103.48) --
	( 16.97,105.58) --
	( 19.07,103.48) --
	( 16.97,101.37) --
	cycle;
\end{scope}
\begin{scope}
\path[clip] (  0.00,  0.00) rectangle (216.81,108.41);
\definecolor{drawColor}{RGB}{246,141,98}

\path[draw=drawColor,line width= 0.5pt,line join=round] ( 35.32,103.48) -- ( 44.42,103.48);
\end{scope}
\begin{scope}
\path[clip] (  0.00,  0.00) rectangle (216.81,108.41);
\definecolor{drawColor}{RGB}{246,141,98}

\path[draw=drawColor,line width= 0.5pt,line join=round,line cap=round] ( 37.76,101.37) -- ( 41.97,105.58);

\path[draw=drawColor,line width= 0.5pt,line join=round,line cap=round] ( 37.76,105.58) -- ( 41.97,101.37);
\end{scope}
\begin{scope}
\path[clip] (  0.00,  0.00) rectangle (216.81,108.41);
\definecolor{drawColor}{RGB}{27,158,119}

\path[draw=drawColor,line width= 0.5pt,line join=round] ( 61.58,103.48) -- ( 70.68,103.48);
\end{scope}
\begin{scope}
\path[clip] (  0.00,  0.00) rectangle (216.81,108.41);
\definecolor{fillColor}{RGB}{27,158,119}

\path[fill=fillColor] ( 66.13,103.48) circle (  2.10);
\end{scope}
\begin{scope}
\path[clip] (  0.00,  0.00) rectangle (216.81,108.41);
\definecolor{drawColor}{RGB}{117,112,179}

\path[draw=drawColor,line width= 0.5pt,line join=round] ( 98.35,103.48) -- (107.45,103.48);
\end{scope}
\begin{scope}
\path[clip] (  0.00,  0.00) rectangle (216.81,108.41);
\definecolor{drawColor}{RGB}{117,112,179}

\path[draw=drawColor,line width= 0.5pt,line join=round,line cap=round] (102.90,103.48) circle (  2.10);
\end{scope}
\begin{scope}
\path[clip] (  0.00,  0.00) rectangle (216.81,108.41);
\definecolor{drawColor}{RGB}{8,95,2}

\path[draw=drawColor,line width= 0.5pt,line join=round] (136.38,103.48) -- (145.48,103.48);
\end{scope}
\begin{scope}
\path[clip] (  0.00,  0.00) rectangle (216.81,108.41);
\definecolor{fillColor}{RGB}{8,95,2}

\path[fill=fillColor] (140.93,106.75) --
	(143.76,101.84) --
	(138.09,101.84) --
	cycle;
\end{scope}
\begin{scope}
\path[clip] (  0.00,  0.00) rectangle (216.81,108.41);
\definecolor{drawColor}{RGB}{231,41,138}

\path[draw=drawColor,line width= 0.5pt,line join=round] (173.14,103.48) -- (182.25,103.48);
\end{scope}
\begin{scope}
\path[clip] (  0.00,  0.00) rectangle (216.81,108.41);
\definecolor{drawColor}{RGB}{231,41,138}

\path[draw=drawColor,line width= 0.5pt,line join=round,line cap=round] (177.70,106.75) --
	(180.53,101.84) --
	(174.86,101.84) --
	cycle;
\end{scope}
\begin{scope}
\path[clip] (  0.00,  0.00) rectangle (216.81,108.41);
\definecolor{drawColor}{RGB}{0,0,0}

\node[text=drawColor,anchor=base west,inner sep=0pt, outer sep=0pt, scale=  0.50] at ( 23.16,101.76) {\Large $T_n$};
\end{scope}
\begin{scope}
\path[clip] (  0.00,  0.00) rectangle (216.81,108.41);
\definecolor{drawColor}{RGB}{0,0,0}

\node[text=drawColor,anchor=base west,inner sep=0pt, outer sep=0pt, scale=  0.50] at ( 46.06,101.76) {\Large $\nu_n$};
\end{scope}
\begin{scope}
\path[clip] (  0.00,  0.00) rectangle (216.81,108.41);
\definecolor{drawColor}{RGB}{0,0,0}

\node[text=drawColor,anchor=base west,inner sep=0pt, outer sep=0pt, scale=  0.50] at ( 72.32,101.76) {\Large $\hat{\eta}^{\operatorname{RKHS}}$};
\end{scope}
\begin{scope}
\path[clip] (  0.00,  0.00) rectangle (216.81,108.41);
\definecolor{drawColor}{RGB}{0,0,0}

\node[text=drawColor,anchor=base west,inner sep=0pt, outer sep=0pt, scale=  0.50] at (109.09,101.76) {\Large $\hat{\eta}^{K\operatorname{-NN}}$};
\end{scope}
\begin{scope}
\path[clip] (  0.00,  0.00) rectangle (216.81,108.41);
\definecolor{drawColor}{RGB}{0,0,0}

\node[text=drawColor,anchor=base west,inner sep=0pt, outer sep=0pt, scale=  0.50] at (147.12,101.76) {\Large $\hat{D}^{\operatorname{RKHS}}$};;
\end{scope}
\begin{scope}
\path[clip] (  0.00,  0.00) rectangle (216.81,108.41);
\definecolor{drawColor}{RGB}{0,0,0}

\node[text=drawColor,anchor=base west,inner sep=0pt, outer sep=0pt, scale=  0.50] at (183.89,101.76) {\Large $\hat{D}^{K\operatorname{-NN}}$};;
\end{scope}
\begin{scope}
\path[clip] (  0.00,  0.00) rectangle (108.40, 98.55);
\definecolor{drawColor}{RGB}{255,255,255}
\definecolor{fillColor}{RGB}{255,255,255}

\path[draw=drawColor,line width= 0.5pt,line join=round,line cap=round,fill=fillColor] (  0.00,  0.00) rectangle (108.41, 98.55);
\end{scope}
\begin{scope}
\path[clip] ( 21.62, 16.56) rectangle (108.31, 98.05);
\definecolor{fillColor}{RGB}{255,255,255}

\path[fill=fillColor] ( 21.62, 16.56) rectangle (108.31, 98.05);
\definecolor{drawColor}{RGB}{31,119,180}

\path[draw=drawColor,line width= 0.2pt,line join=round] ( 25.56, 22.12) --
	( 41.62, 21.01) --
	( 51.02, 21.01) --
	( 57.68, 22.12) --
	( 62.85, 21.38) --
	( 67.08, 21.38) --
	( 73.74, 22.12) --
	( 78.91, 23.60) --
	( 83.14, 22.12) --
	( 86.71, 23.97) --
	( 88.31, 23.60) --
	( 89.80, 23.60) --
	( 92.53, 25.08) --
	( 94.97, 22.49) --
	(100.14, 27.31) --
	(104.36, 28.05);
\definecolor{drawColor}{RGB}{246,141,98}

\path[draw=drawColor,line width= 0.2pt,line join=round] ( 25.56, 20.64) --
	( 41.62, 21.38) --
	( 51.02, 21.75) --
	( 57.68, 23.97) --
	( 62.85, 27.68) --
	( 67.08, 25.45) --
	( 73.74, 25.82) --
	( 78.91, 24.71) --
	( 83.14, 31.38) --
	( 86.71, 29.16) --
	( 88.31, 33.23) --
	( 89.80, 34.34) --
	( 92.53, 36.20) --
	( 94.97, 36.94) --
	(100.14, 43.23) --
	(104.36, 52.49);
\definecolor{drawColor}{RGB}{27,158,119}

\path[draw=drawColor,line width= 0.2pt,line join=round] ( 25.56, 20.27) --
	( 41.62, 21.75) --
	( 51.02, 22.86) --
	( 57.68, 26.57) --
	( 62.85, 35.08) --
	( 67.08, 38.05) --
	( 73.74, 51.01) --
	( 78.91, 71.38) --
	( 83.14, 85.83) --
	( 86.71, 90.27) --
	( 88.31, 91.75) --
	( 89.80, 92.86) --
	( 92.53, 94.35) --
	( 94.97, 94.35) --
	(100.14, 94.35) --
	(104.36, 94.35);
\definecolor{drawColor}{RGB}{117,112,179}

\path[draw=drawColor,line width= 0.2pt,line join=round] ( 25.56, 20.64) --
	( 41.62, 21.38) --
	( 51.02, 21.38) --
	( 57.68, 21.38) --
	( 62.85, 24.34) --
	( 67.08, 25.45) --
	( 73.74, 27.68) --
	( 78.91, 29.16) --
	( 83.14, 33.23) --
	( 86.71, 35.45) --
	( 88.31, 38.79) --
	( 89.80, 39.53) --
	( 92.53, 46.57) --
	( 94.97, 53.23) --
	(100.14, 62.12) --
	(104.36, 73.23);
\definecolor{drawColor}{RGB}{8,95,2}

\path[draw=drawColor,line width= 0.2pt,line join=round] ( 25.56, 20.64) --
	( 41.62, 22.86) --
	( 51.02, 24.34) --
	( 57.68, 27.68) --
	( 62.85, 34.71) --
	( 67.08, 39.90) --
	( 73.74, 51.38) --
	( 78.91, 76.57) --
	( 83.14, 87.68) --
	( 86.71, 91.01) --
	( 88.31, 92.12) --
	( 89.80, 92.86) --
	( 92.53, 94.35) --
	( 94.97, 94.35) --
	(100.14, 94.35) --
	(104.36, 94.35);
\definecolor{drawColor}{RGB}{231,41,138}

\path[draw=drawColor,line width= 0.2pt,line join=round] ( 25.56, 20.64) --
	( 41.62, 21.01) --
	( 51.02, 21.01) --
	( 57.68, 22.12) --
	( 62.85, 23.23) --
	( 67.08, 25.45) --
	( 73.74, 25.45) --
	( 78.91, 27.68) --
	( 83.14, 33.60) --
	( 86.71, 34.71) --
	( 88.31, 38.42) --
	( 89.80, 38.42) --
	( 92.53, 44.71) --
	( 94.97, 50.64) --
	(100.14, 63.60) --
	(104.36, 69.53);
\definecolor{fillColor}{RGB}{31,119,180}

\path[fill=fillColor] ( 40.05, 21.01) --
	( 41.62, 22.58) --
	( 43.19, 21.01) --
	( 41.62, 19.44) --
	cycle;
\definecolor{drawColor}{RGB}{246,141,98}

\path[draw=drawColor,line width= 0.5pt,line join=round,line cap=round] ( 40.05, 19.81) -- ( 43.19, 22.95);

\path[draw=drawColor,line width= 0.5pt,line join=round,line cap=round] ( 40.05, 22.95) -- ( 43.19, 19.81);
\definecolor{fillColor}{RGB}{27,158,119}

\path[fill=fillColor] ( 41.62, 21.75) circle (  1.57);
\definecolor{drawColor}{RGB}{117,112,179}

\path[draw=drawColor,line width= 0.5pt,line join=round,line cap=round] ( 41.62, 21.38) circle (  1.57);
\definecolor{fillColor}{RGB}{8,95,2}

\path[fill=fillColor] ( 41.62, 25.30) --
	( 43.74, 21.64) --
	( 39.51, 21.64) --
	cycle;
\definecolor{drawColor}{RGB}{231,41,138}

\path[draw=drawColor,line width= 0.5pt,line join=round,line cap=round] ( 41.62, 23.45) --
	( 43.74, 19.79) --
	( 39.51, 19.79) --
	cycle;
\definecolor{fillColor}{RGB}{31,119,180}

\path[fill=fillColor] ( 93.40, 22.49) --
	( 94.97, 24.06) --
	( 96.54, 22.49) --
	( 94.97, 20.92) --
	cycle;
\definecolor{drawColor}{RGB}{246,141,98}

\path[draw=drawColor,line width= 0.5pt,line join=round,line cap=round] ( 93.40, 35.37) -- ( 96.54, 38.51);

\path[draw=drawColor,line width= 0.5pt,line join=round,line cap=round] ( 93.40, 38.51) -- ( 96.54, 35.37);
\definecolor{fillColor}{RGB}{27,158,119}

\path[fill=fillColor] ( 94.97, 94.35) circle (  1.57);
\definecolor{drawColor}{RGB}{117,112,179}

\path[draw=drawColor,line width= 0.5pt,line join=round,line cap=round] ( 94.97, 53.23) circle (  1.57);
\definecolor{fillColor}{RGB}{8,95,2}

\path[fill=fillColor] ( 94.97, 96.79) --
	( 97.08, 93.13) --
	( 92.86, 93.13) --
	cycle;
\definecolor{drawColor}{RGB}{231,41,138}

\path[draw=drawColor,line width= 0.5pt,line join=round,line cap=round] ( 94.97, 53.08) --
	( 97.08, 49.42) --
	( 92.86, 49.42) --
	cycle;
\definecolor{fillColor}{RGB}{31,119,180}

\path[fill=fillColor] ( 98.57, 27.31) --
	(100.14, 28.88) --
	(101.71, 27.31) --
	(100.14, 25.74) --
	cycle;
\definecolor{drawColor}{RGB}{246,141,98}

\path[draw=drawColor,line width= 0.5pt,line join=round,line cap=round] ( 98.57, 41.66) -- (101.71, 44.80);

\path[draw=drawColor,line width= 0.5pt,line join=round,line cap=round] ( 98.57, 44.80) -- (101.71, 41.66);
\definecolor{fillColor}{RGB}{27,158,119}

\path[fill=fillColor] (100.14, 94.35) circle (  1.57);
\definecolor{drawColor}{RGB}{117,112,179}

\path[draw=drawColor,line width= 0.5pt,line join=round,line cap=round] (100.14, 62.12) circle (  1.57);
\definecolor{fillColor}{RGB}{8,95,2}

\path[fill=fillColor] (100.14, 96.79) --
	(102.25, 93.13) --
	( 98.03, 93.13) --
	cycle;
\definecolor{drawColor}{RGB}{231,41,138}

\path[draw=drawColor,line width= 0.5pt,line join=round,line cap=round] (100.14, 66.04) --
	(102.25, 62.38) --
	( 98.03, 62.38) --
	cycle;
\definecolor{fillColor}{RGB}{31,119,180}

\path[fill=fillColor] ( 49.45, 21.01) --
	( 51.02, 22.58) --
	( 52.59, 21.01) --
	( 51.02, 19.44) --
	cycle;
\definecolor{drawColor}{RGB}{246,141,98}

\path[draw=drawColor,line width= 0.5pt,line join=round,line cap=round] ( 49.45, 20.18) -- ( 52.59, 23.32);

\path[draw=drawColor,line width= 0.5pt,line join=round,line cap=round] ( 49.45, 23.32) -- ( 52.59, 20.18);
\definecolor{fillColor}{RGB}{27,158,119}

\path[fill=fillColor] ( 51.02, 22.86) circle (  1.57);
\definecolor{drawColor}{RGB}{117,112,179}

\path[draw=drawColor,line width= 0.5pt,line join=round,line cap=round] ( 51.02, 21.38) circle (  1.57);
\definecolor{fillColor}{RGB}{8,95,2}

\path[fill=fillColor] ( 51.02, 26.78) --
	( 53.13, 23.12) --
	( 48.90, 23.12) --
	cycle;
\definecolor{drawColor}{RGB}{231,41,138}

\path[draw=drawColor,line width= 0.5pt,line join=round,line cap=round] ( 51.02, 23.45) --
	( 53.13, 19.79) --
	( 48.90, 19.79) --
	cycle;
\definecolor{fillColor}{RGB}{31,119,180}

\path[fill=fillColor] (102.80, 28.05) --
	(104.36, 29.62) --
	(105.93, 28.05) --
	(104.36, 26.48) --
	cycle;
\definecolor{drawColor}{RGB}{246,141,98}

\path[draw=drawColor,line width= 0.5pt,line join=round,line cap=round] (102.80, 50.92) -- (105.93, 54.06);

\path[draw=drawColor,line width= 0.5pt,line join=round,line cap=round] (102.80, 54.06) -- (105.93, 50.92);
\definecolor{fillColor}{RGB}{27,158,119}

\path[fill=fillColor] (104.36, 94.35) circle (  1.57);
\definecolor{drawColor}{RGB}{117,112,179}

\path[draw=drawColor,line width= 0.5pt,line join=round,line cap=round] (104.36, 73.23) circle (  1.57);
\definecolor{fillColor}{RGB}{8,95,2}

\path[fill=fillColor] (104.36, 96.79) --
	(106.48, 93.13) --
	(102.25, 93.13) --
	cycle;
\definecolor{drawColor}{RGB}{231,41,138}

\path[draw=drawColor,line width= 0.5pt,line join=round,line cap=round] (104.36, 71.97) --
	(106.48, 68.31) --
	(102.25, 68.31) --
	cycle;
\definecolor{fillColor}{RGB}{31,119,180}

\path[fill=fillColor] ( 56.11, 22.12) --
	( 57.68, 23.69) --
	( 59.25, 22.12) --
	( 57.68, 20.55) --
	cycle;
\definecolor{drawColor}{RGB}{246,141,98}

\path[draw=drawColor,line width= 0.5pt,line join=round,line cap=round] ( 56.11, 22.40) -- ( 59.25, 25.54);

\path[draw=drawColor,line width= 0.5pt,line join=round,line cap=round] ( 56.11, 25.54) -- ( 59.25, 22.40);
\definecolor{fillColor}{RGB}{27,158,119}

\path[fill=fillColor] ( 57.68, 26.57) circle (  1.57);
\definecolor{drawColor}{RGB}{117,112,179}

\path[draw=drawColor,line width= 0.5pt,line join=round,line cap=round] ( 57.68, 21.38) circle (  1.57);
\definecolor{fillColor}{RGB}{8,95,2}

\path[fill=fillColor] ( 57.68, 30.12) --
	( 59.80, 26.46) --
	( 55.57, 26.46) --
	cycle;
\definecolor{drawColor}{RGB}{231,41,138}

\path[draw=drawColor,line width= 0.5pt,line join=round,line cap=round] ( 57.68, 24.56) --
	( 59.80, 20.90) --
	( 55.57, 20.90) --
	cycle;
\definecolor{fillColor}{RGB}{31,119,180}

\path[fill=fillColor] ( 61.28, 21.38) --
	( 62.85, 22.95) --
	( 64.42, 21.38) --
	( 62.85, 19.81) --
	cycle;
\definecolor{drawColor}{RGB}{246,141,98}

\path[draw=drawColor,line width= 0.5pt,line join=round,line cap=round] ( 61.28, 26.11) -- ( 64.42, 29.25);

\path[draw=drawColor,line width= 0.5pt,line join=round,line cap=round] ( 61.28, 29.25) -- ( 64.42, 26.11);
\definecolor{fillColor}{RGB}{27,158,119}

\path[fill=fillColor] ( 62.85, 35.08) circle (  1.57);
\definecolor{drawColor}{RGB}{117,112,179}

\path[draw=drawColor,line width= 0.5pt,line join=round,line cap=round] ( 62.85, 24.34) circle (  1.57);
\definecolor{fillColor}{RGB}{8,95,2}

\path[fill=fillColor] ( 62.85, 37.15) --
	( 64.97, 33.49) --
	( 60.74, 33.49) --
	cycle;
\definecolor{drawColor}{RGB}{231,41,138}

\path[draw=drawColor,line width= 0.5pt,line join=round,line cap=round] ( 62.85, 25.67) --
	( 64.97, 22.01) --
	( 60.74, 22.01) --
	cycle;
\definecolor{fillColor}{RGB}{31,119,180}

\path[fill=fillColor] ( 65.51, 21.38) --
	( 67.08, 22.95) --
	( 68.65, 21.38) --
	( 67.08, 19.81) --
	cycle;
\definecolor{drawColor}{RGB}{246,141,98}

\path[draw=drawColor,line width= 0.5pt,line join=round,line cap=round] ( 65.51, 23.89) -- ( 68.65, 27.02);

\path[draw=drawColor,line width= 0.5pt,line join=round,line cap=round] ( 65.51, 27.02) -- ( 68.65, 23.89);
\definecolor{fillColor}{RGB}{27,158,119}

\path[fill=fillColor] ( 67.08, 38.05) circle (  1.57);
\definecolor{drawColor}{RGB}{117,112,179}

\path[draw=drawColor,line width= 0.5pt,line join=round,line cap=round] ( 67.08, 25.45) circle (  1.57);
\definecolor{fillColor}{RGB}{8,95,2}

\path[fill=fillColor] ( 67.08, 42.34) --
	( 69.19, 38.68) --
	( 64.96, 38.68) --
	cycle;
\definecolor{drawColor}{RGB}{231,41,138}

\path[draw=drawColor,line width= 0.5pt,line join=round,line cap=round] ( 67.08, 27.89) --
	( 69.19, 24.23) --
	( 64.96, 24.23) --
	cycle;
\definecolor{fillColor}{RGB}{31,119,180}

\path[fill=fillColor] ( 72.17, 22.12) --
	( 73.74, 23.69) --
	( 75.31, 22.12) --
	( 73.74, 20.55) --
	cycle;
\definecolor{drawColor}{RGB}{246,141,98}

\path[draw=drawColor,line width= 0.5pt,line join=round,line cap=round] ( 72.17, 24.26) -- ( 75.31, 27.39);

\path[draw=drawColor,line width= 0.5pt,line join=round,line cap=round] ( 72.17, 27.39) -- ( 75.31, 24.26);
\definecolor{fillColor}{RGB}{27,158,119}

\path[fill=fillColor] ( 73.74, 51.01) circle (  1.57);
\definecolor{drawColor}{RGB}{117,112,179}

\path[draw=drawColor,line width= 0.5pt,line join=round,line cap=round] ( 73.74, 27.68) circle (  1.57);
\definecolor{fillColor}{RGB}{8,95,2}

\path[fill=fillColor] ( 73.74, 53.82) --
	( 75.85, 50.16) --
	( 71.63, 50.16) --
	cycle;
\definecolor{drawColor}{RGB}{231,41,138}

\path[draw=drawColor,line width= 0.5pt,line join=round,line cap=round] ( 73.74, 27.89) --
	( 75.85, 24.23) --
	( 71.63, 24.23) --
	cycle;
\definecolor{fillColor}{RGB}{31,119,180}

\path[fill=fillColor] ( 23.99, 22.12) --
	( 25.56, 23.69) --
	( 27.13, 22.12) --
	( 25.56, 20.55) --
	cycle;
\definecolor{drawColor}{RGB}{246,141,98}

\path[draw=drawColor,line width= 0.5pt,line join=round,line cap=round] ( 23.99, 19.07) -- ( 27.13, 22.21);

\path[draw=drawColor,line width= 0.5pt,line join=round,line cap=round] ( 23.99, 22.21) -- ( 27.13, 19.07);
\definecolor{fillColor}{RGB}{27,158,119}

\path[fill=fillColor] ( 25.56, 20.27) circle (  1.57);
\definecolor{drawColor}{RGB}{117,112,179}

\path[draw=drawColor,line width= 0.5pt,line join=round,line cap=round] ( 25.56, 20.64) circle (  1.57);
\definecolor{fillColor}{RGB}{8,95,2}

\path[fill=fillColor] ( 25.56, 23.08) --
	( 27.68, 19.42) --
	( 23.45, 19.42) --
	cycle;
\definecolor{drawColor}{RGB}{231,41,138}

\path[draw=drawColor,line width= 0.5pt,line join=round,line cap=round] ( 25.56, 23.08) --
	( 27.68, 19.42) --
	( 23.45, 19.42) --
	cycle;
\definecolor{fillColor}{RGB}{31,119,180}

\path[fill=fillColor] ( 77.34, 23.60) --
	( 78.91, 25.17) --
	( 80.48, 23.60) --
	( 78.91, 22.03) --
	cycle;
\definecolor{drawColor}{RGB}{246,141,98}

\path[draw=drawColor,line width= 0.5pt,line join=round,line cap=round] ( 77.34, 23.14) -- ( 80.48, 26.28);

\path[draw=drawColor,line width= 0.5pt,line join=round,line cap=round] ( 77.34, 26.28) -- ( 80.48, 23.14);
\definecolor{fillColor}{RGB}{27,158,119}

\path[fill=fillColor] ( 78.91, 71.38) circle (  1.57);
\definecolor{drawColor}{RGB}{117,112,179}

\path[draw=drawColor,line width= 0.5pt,line join=round,line cap=round] ( 78.91, 29.16) circle (  1.57);
\definecolor{fillColor}{RGB}{8,95,2}

\path[fill=fillColor] ( 78.91, 79.01) --
	( 81.02, 75.35) --
	( 76.80, 75.35) --
	cycle;
\definecolor{drawColor}{RGB}{231,41,138}

\path[draw=drawColor,line width= 0.5pt,line join=round,line cap=round] ( 78.91, 30.12) --
	( 81.02, 26.46) --
	( 76.80, 26.46) --
	cycle;
\definecolor{fillColor}{RGB}{31,119,180}

\path[fill=fillColor] ( 81.57, 22.12) --
	( 83.14, 23.69) --
	( 84.70, 22.12) --
	( 83.14, 20.55) --
	cycle;
\definecolor{drawColor}{RGB}{246,141,98}

\path[draw=drawColor,line width= 0.5pt,line join=round,line cap=round] ( 81.57, 29.81) -- ( 84.70, 32.95);

\path[draw=drawColor,line width= 0.5pt,line join=round,line cap=round] ( 81.57, 32.95) -- ( 84.70, 29.81);
\definecolor{fillColor}{RGB}{27,158,119}

\path[fill=fillColor] ( 83.14, 85.83) circle (  1.57);
\definecolor{drawColor}{RGB}{117,112,179}

\path[draw=drawColor,line width= 0.5pt,line join=round,line cap=round] ( 83.14, 33.23) circle (  1.57);
\definecolor{fillColor}{RGB}{8,95,2}

\path[fill=fillColor] ( 83.14, 90.12) --
	( 85.25, 86.46) --
	( 81.02, 86.46) --
	cycle;
\definecolor{drawColor}{RGB}{231,41,138}

\path[draw=drawColor,line width= 0.5pt,line join=round,line cap=round] ( 83.14, 36.04) --
	( 85.25, 32.38) --
	( 81.02, 32.38) --
	cycle;
\definecolor{fillColor}{RGB}{31,119,180}

\path[fill=fillColor] ( 85.14, 23.97) --
	( 86.71, 25.54) --
	( 88.28, 23.97) --
	( 86.71, 22.40) --
	cycle;
\definecolor{drawColor}{RGB}{246,141,98}

\path[draw=drawColor,line width= 0.5pt,line join=round,line cap=round] ( 85.14, 27.59) -- ( 88.28, 30.73);

\path[draw=drawColor,line width= 0.5pt,line join=round,line cap=round] ( 85.14, 30.73) -- ( 88.28, 27.59);
\definecolor{fillColor}{RGB}{27,158,119}

\path[fill=fillColor] ( 86.71, 90.27) circle (  1.57);
\definecolor{drawColor}{RGB}{117,112,179}

\path[draw=drawColor,line width= 0.5pt,line join=round,line cap=round] ( 86.71, 35.45) circle (  1.57);
\definecolor{fillColor}{RGB}{8,95,2}

\path[fill=fillColor] ( 86.71, 93.45) --
	( 88.82, 89.79) --
	( 84.59, 89.79) --
	cycle;
\definecolor{drawColor}{RGB}{231,41,138}

\path[draw=drawColor,line width= 0.5pt,line join=round,line cap=round] ( 86.71, 37.15) --
	( 88.82, 33.49) --
	( 84.59, 33.49) --
	cycle;
\definecolor{fillColor}{RGB}{31,119,180}

\path[fill=fillColor] ( 86.74, 23.60) --
	( 88.31, 25.17) --
	( 89.87, 23.60) --
	( 88.31, 22.03) --
	cycle;
\definecolor{drawColor}{RGB}{246,141,98}

\path[draw=drawColor,line width= 0.5pt,line join=round,line cap=round] ( 86.74, 31.66) -- ( 89.87, 34.80);

\path[draw=drawColor,line width= 0.5pt,line join=round,line cap=round] ( 86.74, 34.80) -- ( 89.87, 31.66);
\definecolor{fillColor}{RGB}{27,158,119}

\path[fill=fillColor] ( 88.31, 91.75) circle (  1.57);
\definecolor{drawColor}{RGB}{117,112,179}

\path[draw=drawColor,line width= 0.5pt,line join=round,line cap=round] ( 88.31, 38.79) circle (  1.57);
\definecolor{fillColor}{RGB}{8,95,2}

\path[fill=fillColor] ( 88.31, 94.56) --
	( 90.42, 90.90) --
	( 86.19, 90.90) --
	cycle;
\definecolor{drawColor}{RGB}{231,41,138}

\path[draw=drawColor,line width= 0.5pt,line join=round,line cap=round] ( 88.31, 40.86) --
	( 90.42, 37.20) --
	( 86.19, 37.20) --
	cycle;
\definecolor{fillColor}{RGB}{31,119,180}

\path[fill=fillColor] ( 88.23, 23.60) --
	( 89.80, 25.17) --
	( 91.37, 23.60) --
	( 89.80, 22.03) --
	cycle;
\definecolor{drawColor}{RGB}{246,141,98}

\path[draw=drawColor,line width= 0.5pt,line join=round,line cap=round] ( 88.23, 32.77) -- ( 91.37, 35.91);

\path[draw=drawColor,line width= 0.5pt,line join=round,line cap=round] ( 88.23, 35.91) -- ( 91.37, 32.77);
\definecolor{fillColor}{RGB}{27,158,119}

\path[fill=fillColor] ( 89.80, 92.86) circle (  1.57);
\definecolor{drawColor}{RGB}{117,112,179}

\path[draw=drawColor,line width= 0.5pt,line join=round,line cap=round] ( 89.80, 39.53) circle (  1.57);
\definecolor{fillColor}{RGB}{8,95,2}

\path[fill=fillColor] ( 89.80, 95.30) --
	( 91.91, 91.64) --
	( 87.69, 91.64) --
	cycle;
\definecolor{drawColor}{RGB}{231,41,138}

\path[draw=drawColor,line width= 0.5pt,line join=round,line cap=round] ( 89.80, 40.86) --
	( 91.91, 37.20) --
	( 87.69, 37.20) --
	cycle;
\definecolor{fillColor}{RGB}{31,119,180}

\path[fill=fillColor] ( 90.96, 25.08) --
	( 92.53, 26.65) --
	( 94.10, 25.08) --
	( 92.53, 23.51) --
	cycle;
\definecolor{drawColor}{RGB}{246,141,98}

\path[draw=drawColor,line width= 0.5pt,line join=round,line cap=round] ( 90.96, 34.63) -- ( 94.10, 37.76);

\path[draw=drawColor,line width= 0.5pt,line join=round,line cap=round] ( 90.96, 37.76) -- ( 94.10, 34.63);
\definecolor{fillColor}{RGB}{27,158,119}

\path[fill=fillColor] ( 92.53, 94.35) circle (  1.57);
\definecolor{drawColor}{RGB}{117,112,179}

\path[draw=drawColor,line width= 0.5pt,line join=round,line cap=round] ( 92.53, 46.57) circle (  1.57);
\definecolor{fillColor}{RGB}{8,95,2}

\path[fill=fillColor] ( 92.53, 96.79) --
	( 94.64, 93.13) --
	( 90.42, 93.13) --
	cycle;
\definecolor{drawColor}{RGB}{231,41,138}

\path[draw=drawColor,line width= 0.5pt,line join=round,line cap=round] ( 92.53, 47.15) --
	( 94.64, 43.49) --
	( 90.42, 43.49) --
	cycle;
\definecolor{drawColor}{gray}{0.20}

\path[draw=drawColor,line width= 0.5pt,line join=round,line cap=round] ( 21.62, 16.56) rectangle (108.31, 98.05);
\end{scope}
\begin{scope}
\path[clip] (  0.00,  0.00) rectangle (216.81,108.41);
\definecolor{drawColor}{gray}{0.30}

\node[text=drawColor,anchor=base east,inner sep=0pt, outer sep=0pt, scale=  0.50] at ( 17.57, 18.55) {0.00};

\node[text=drawColor,anchor=base east,inner sep=0pt, outer sep=0pt, scale=  0.50] at ( 17.57, 37.07) {0.25};

\node[text=drawColor,anchor=base east,inner sep=0pt, outer sep=0pt, scale=  0.50] at ( 17.57, 55.59) {0.50};

\node[text=drawColor,anchor=base east,inner sep=0pt, outer sep=0pt, scale=  0.50] at ( 17.57, 74.10) {0.75};

\node[text=drawColor,anchor=base east,inner sep=0pt, outer sep=0pt, scale=  0.50] at ( 17.57, 92.62) {1.00};
\end{scope}
\begin{scope}
\path[clip] (  0.00,  0.00) rectangle (216.81,108.41);
\definecolor{drawColor}{gray}{0.20}

\path[draw=drawColor,line width= 0.5pt,line join=round] ( 19.37, 20.27) --
	( 21.62, 20.27);

\path[draw=drawColor,line width= 0.5pt,line join=round] ( 19.37, 38.79) --
	( 21.62, 38.79);

\path[draw=drawColor,line width= 0.5pt,line join=round] ( 19.37, 57.31) --
	( 21.62, 57.31);

\path[draw=drawColor,line width= 0.5pt,line join=round] ( 19.37, 75.83) --
	( 21.62, 75.83);

\path[draw=drawColor,line width= 0.5pt,line join=round] ( 19.37, 94.35) --
	( 21.62, 94.35);
\end{scope}
\begin{scope}
\path[clip] (  0.00,  0.00) rectangle (216.81,108.41);
\definecolor{drawColor}{gray}{0.20}

\path[draw=drawColor,line width= 0.5pt,line join=round] ( 41.62, 14.31) --
	( 41.62, 16.56);

\path[draw=drawColor,line width= 0.5pt,line join=round] ( 67.08, 14.31) --
	( 67.08, 16.56);

\path[draw=drawColor,line width= 0.5pt,line join=round] ( 94.97, 14.31) --
	( 94.97, 16.56);
\end{scope}
\begin{scope}
\path[clip] (  0.00,  0.00) rectangle (216.81,108.41);
\definecolor{drawColor}{gray}{0.30}

\node[text=drawColor,anchor=base,inner sep=0pt, outer sep=0pt, scale=  0.50] at ( 41.62,  9.07) {100};

\node[text=drawColor,anchor=base,inner sep=0pt, outer sep=0pt, scale=  0.50] at ( 67.08,  9.07) {300};

\node[text=drawColor,anchor=base,inner sep=0pt, outer sep=0pt, scale=  0.50] at ( 94.97,  9.07) {1000};
\end{scope}
\begin{scope}
\path[clip] (  0.00,  0.00) rectangle (216.81,108.41);
\definecolor{drawColor}{RGB}{0,0,0}

\node[text=drawColor,anchor=base,inner sep=0pt, outer sep=0pt, scale=  0.60] at ( 64.96,  1.42) {Sample Size $(n)$};
\end{scope}
\begin{scope}
\path[clip] (  0.00,  0.00) rectangle (216.81,108.41);
\definecolor{drawColor}{RGB}{0,0,0}

\node[text=drawColor,rotate= 90.00,anchor=base,inner sep=0pt, outer sep=0pt, scale=  0.60] at (  4.23, 57.31) {Power};
\end{scope}
\begin{scope}
\path[clip] (108.41,  0.00) rectangle (216.81, 98.55);
\definecolor{drawColor}{RGB}{255,255,255}
\definecolor{fillColor}{RGB}{255,255,255}

\path[draw=drawColor,line width= 0.5pt,line join=round,line cap=round,fill=fillColor] (108.41,  0.00) rectangle (216.81, 98.55);
\end{scope}
\begin{scope}
\path[clip] (134.35, 16.56) rectangle (216.71, 98.05);
\definecolor{fillColor}{RGB}{255,255,255}

\path[fill=fillColor] (134.35, 16.56) rectangle (216.71, 98.05);
\definecolor{drawColor}{RGB}{31,119,180}

\path[draw=drawColor,line width= 0.2pt,line join=round] (138.10, 66.05) --
	(146.45, 28.24) --
	(157.49, 29.43) --
	(165.84, 30.54) --
	(174.19, 33.08) --
	(185.23, 36.94) --
	(193.58, 40.40) --
	(201.93, 44.34) --
	(212.97, 50.78);
\definecolor{drawColor}{RGB}{246,141,98}

\path[draw=drawColor,line width= 0.2pt,line join=round] (138.10, 47.87) --
	(146.45, 20.27) --
	(157.49, 25.56) --
	(165.84, 29.80) --
	(174.19, 40.08) --
	(185.23, 45.90) --
	(193.58, 50.19) --
	(201.93, 54.10) --
	(212.97, 63.25);
\definecolor{drawColor}{RGB}{27,158,119}

\path[draw=drawColor,line width= 0.2pt,line join=round] (138.10, 41.43) --
	(146.45, 35.79) --
	(157.49, 36.63) --
	(165.84, 39.50) --
	(174.19, 44.43) --
	(185.23, 56.39) --
	(193.58, 71.18) --
	(201.93, 80.07) --
	(212.97, 93.64);
\definecolor{drawColor}{RGB}{117,112,179}

\path[draw=drawColor,line width= 0.2pt,line join=round] (138.10, 44.95) --
	(146.45, 50.86) --
	(157.49, 56.64) --
	(165.84, 61.42) --
	(174.19, 65.95) --
	(185.23, 72.00) --
	(193.58, 77.30) --
	(201.93, 82.54) --
	(212.97, 90.33);
\definecolor{drawColor}{RGB}{8,95,2}

\path[draw=drawColor,line width= 0.2pt,line join=round] (138.10, 37.87) --
	(146.45, 34.91) --
	(157.49, 36.06) --
	(165.84, 39.17) --
	(174.19, 44.28) --
	(185.23, 60.77) --
	(193.58, 67.25) --
	(201.93, 78.72) --
	(212.97, 94.35);
\definecolor{drawColor}{RGB}{231,41,138}

\path[draw=drawColor,line width= 0.2pt,line join=round] (138.10, 31.49) --
	(146.45, 32.37) --
	(157.49, 35.20) --
	(165.84, 39.51) --
	(174.19, 45.39) --
	(185.23, 56.00) --
	(193.58, 65.10) --
	(201.93, 74.51) --
	(212.97, 87.78);
\definecolor{fillColor}{RGB}{31,119,180}

\path[fill=fillColor] (136.53, 66.05) --
	(138.10, 67.62) --
	(139.67, 66.05) --
	(138.10, 64.48) --
	cycle;
\definecolor{drawColor}{RGB}{246,141,98}

\path[draw=drawColor,line width= 0.5pt,line join=round,line cap=round] (136.53, 46.30) -- (139.67, 49.44);

\path[draw=drawColor,line width= 0.5pt,line join=round,line cap=round] (136.53, 49.44) -- (139.67, 46.30);
\definecolor{fillColor}{RGB}{27,158,119}

\path[fill=fillColor] (138.10, 41.43) circle (  1.57);
\definecolor{drawColor}{RGB}{117,112,179}

\path[draw=drawColor,line width= 0.5pt,line join=round,line cap=round] (138.10, 44.95) circle (  1.57);
\definecolor{fillColor}{RGB}{8,95,2}

\path[fill=fillColor] (138.10, 40.31) --
	(140.21, 36.65) --
	(135.99, 36.65) --
	cycle;
\definecolor{drawColor}{RGB}{231,41,138}

\path[draw=drawColor,line width= 0.5pt,line join=round,line cap=round] (138.10, 33.93) --
	(140.21, 30.27) --
	(135.99, 30.27) --
	cycle;
\definecolor{fillColor}{RGB}{31,119,180}

\path[fill=fillColor] (144.88, 28.24) --
	(146.45, 29.81) --
	(148.02, 28.24) --
	(146.45, 26.68) --
	cycle;
\definecolor{drawColor}{RGB}{246,141,98}

\path[draw=drawColor,line width= 0.5pt,line join=round,line cap=round] (144.88, 18.70) -- (148.02, 21.84);

\path[draw=drawColor,line width= 0.5pt,line join=round,line cap=round] (144.88, 21.84) -- (148.02, 18.70);
\definecolor{fillColor}{RGB}{27,158,119}

\path[fill=fillColor] (146.45, 35.79) circle (  1.57);
\definecolor{drawColor}{RGB}{117,112,179}

\path[draw=drawColor,line width= 0.5pt,line join=round,line cap=round] (146.45, 50.86) circle (  1.57);
\definecolor{fillColor}{RGB}{8,95,2}

\path[fill=fillColor] (146.45, 37.35) --
	(148.56, 33.69) --
	(144.34, 33.69) --
	cycle;
\definecolor{drawColor}{RGB}{231,41,138}

\path[draw=drawColor,line width= 0.5pt,line join=round,line cap=round] (146.45, 34.81) --
	(148.56, 31.15) --
	(144.34, 31.15) --
	cycle;
\definecolor{fillColor}{RGB}{31,119,180}

\path[fill=fillColor] (155.92, 29.43) --
	(157.49, 31.00) --
	(159.06, 29.43) --
	(157.49, 27.86) --
	cycle;
\definecolor{drawColor}{RGB}{246,141,98}

\path[draw=drawColor,line width= 0.5pt,line join=round,line cap=round] (155.92, 23.99) -- (159.06, 27.13);

\path[draw=drawColor,line width= 0.5pt,line join=round,line cap=round] (155.92, 27.13) -- (159.06, 23.99);
\definecolor{fillColor}{RGB}{27,158,119}

\path[fill=fillColor] (157.49, 36.63) circle (  1.57);
\definecolor{drawColor}{RGB}{117,112,179}

\path[draw=drawColor,line width= 0.5pt,line join=round,line cap=round] (157.49, 56.64) circle (  1.57);
\definecolor{fillColor}{RGB}{8,95,2}

\path[fill=fillColor] (157.49, 38.50) --
	(159.60, 34.84) --
	(155.37, 34.84) --
	cycle;
\definecolor{drawColor}{RGB}{231,41,138}

\path[draw=drawColor,line width= 0.5pt,line join=round,line cap=round] (157.49, 37.64) --
	(159.60, 33.98) --
	(155.37, 33.98) --
	cycle;
\definecolor{fillColor}{RGB}{31,119,180}

\path[fill=fillColor] (164.27, 30.54) --
	(165.84, 32.11) --
	(167.41, 30.54) --
	(165.84, 28.97) --
	cycle;
\definecolor{drawColor}{RGB}{246,141,98}

\path[draw=drawColor,line width= 0.5pt,line join=round,line cap=round] (164.27, 28.23) -- (167.41, 31.37);

\path[draw=drawColor,line width= 0.5pt,line join=round,line cap=round] (164.27, 31.37) -- (167.41, 28.23);
\definecolor{fillColor}{RGB}{27,158,119}

\path[fill=fillColor] (165.84, 39.50) circle (  1.57);
\definecolor{drawColor}{RGB}{117,112,179}

\path[draw=drawColor,line width= 0.5pt,line join=round,line cap=round] (165.84, 61.42) circle (  1.57);
\definecolor{fillColor}{RGB}{8,95,2}

\path[fill=fillColor] (165.84, 41.61) --
	(167.95, 37.95) --
	(163.72, 37.95) --
	cycle;
\definecolor{drawColor}{RGB}{231,41,138}

\path[draw=drawColor,line width= 0.5pt,line join=round,line cap=round] (165.84, 41.95) --
	(167.95, 38.29) --
	(163.72, 38.29) --
	cycle;
\definecolor{fillColor}{RGB}{31,119,180}

\path[fill=fillColor] (172.62, 33.08) --
	(174.19, 34.65) --
	(175.76, 33.08) --
	(174.19, 31.51) --
	cycle;
\definecolor{drawColor}{RGB}{246,141,98}

\path[draw=drawColor,line width= 0.5pt,line join=round,line cap=round] (172.62, 38.51) -- (175.76, 41.65);

\path[draw=drawColor,line width= 0.5pt,line join=round,line cap=round] (172.62, 41.65) -- (175.76, 38.51);
\definecolor{fillColor}{RGB}{27,158,119}

\path[fill=fillColor] (174.19, 44.43) circle (  1.57);
\definecolor{drawColor}{RGB}{117,112,179}

\path[draw=drawColor,line width= 0.5pt,line join=round,line cap=round] (174.19, 65.95) circle (  1.57);
\definecolor{fillColor}{RGB}{8,95,2}

\path[fill=fillColor] (174.19, 46.72) --
	(176.30, 43.06) --
	(172.08, 43.06) --
	cycle;
\definecolor{drawColor}{RGB}{231,41,138}

\path[draw=drawColor,line width= 0.5pt,line join=round,line cap=round] (174.19, 47.83) --
	(176.30, 44.17) --
	(172.08, 44.17) --
	cycle;
\definecolor{fillColor}{RGB}{31,119,180}

\path[fill=fillColor] (183.66, 36.94) --
	(185.23, 38.51) --
	(186.80, 36.94) --
	(185.23, 35.37) --
	cycle;
\definecolor{drawColor}{RGB}{246,141,98}

\path[draw=drawColor,line width= 0.5pt,line join=round,line cap=round] (183.66, 44.33) -- (186.80, 47.47);

\path[draw=drawColor,line width= 0.5pt,line join=round,line cap=round] (183.66, 47.47) -- (186.80, 44.33);
\definecolor{fillColor}{RGB}{27,158,119}

\path[fill=fillColor] (185.23, 56.39) circle (  1.57);
\definecolor{drawColor}{RGB}{117,112,179}

\path[draw=drawColor,line width= 0.5pt,line join=round,line cap=round] (185.23, 72.00) circle (  1.57);
\definecolor{fillColor}{RGB}{8,95,2}

\path[fill=fillColor] (185.23, 63.21) --
	(187.34, 59.55) --
	(183.11, 59.55) --
	cycle;
\definecolor{drawColor}{RGB}{231,41,138}

\path[draw=drawColor,line width= 0.5pt,line join=round,line cap=round] (185.23, 58.44) --
	(187.34, 54.78) --
	(183.11, 54.78) --
	cycle;
\definecolor{fillColor}{RGB}{31,119,180}

\path[fill=fillColor] (192.01, 40.40) --
	(193.58, 41.97) --
	(195.15, 40.40) --
	(193.58, 38.83) --
	cycle;
\definecolor{drawColor}{RGB}{246,141,98}

\path[draw=drawColor,line width= 0.5pt,line join=round,line cap=round] (192.01, 48.63) -- (195.15, 51.76);

\path[draw=drawColor,line width= 0.5pt,line join=round,line cap=round] (192.01, 51.76) -- (195.15, 48.63);
\definecolor{fillColor}{RGB}{27,158,119}

\path[fill=fillColor] (193.58, 71.18) circle (  1.57);
\definecolor{drawColor}{RGB}{117,112,179}

\path[draw=drawColor,line width= 0.5pt,line join=round,line cap=round] (193.58, 77.30) circle (  1.57);
\definecolor{fillColor}{RGB}{8,95,2}

\path[fill=fillColor] (193.58, 69.69) --
	(195.69, 66.03) --
	(191.46, 66.03) --
	cycle;
\definecolor{drawColor}{RGB}{231,41,138}

\path[draw=drawColor,line width= 0.5pt,line join=round,line cap=round] (193.58, 67.54) --
	(195.69, 63.88) --
	(191.46, 63.88) --
	cycle;
\definecolor{fillColor}{RGB}{31,119,180}

\path[fill=fillColor] (200.36, 44.34) --
	(201.93, 45.91) --
	(203.50, 44.34) --
	(201.93, 42.77) --
	cycle;
\definecolor{drawColor}{RGB}{246,141,98}

\path[draw=drawColor,line width= 0.5pt,line join=round,line cap=round] (200.36, 52.53) -- (203.50, 55.67);

\path[draw=drawColor,line width= 0.5pt,line join=round,line cap=round] (200.36, 55.67) -- (203.50, 52.53);
\definecolor{fillColor}{RGB}{27,158,119}

\path[fill=fillColor] (201.93, 80.07) circle (  1.57);
\definecolor{drawColor}{RGB}{117,112,179}

\path[draw=drawColor,line width= 0.5pt,line join=round,line cap=round] (201.93, 82.54) circle (  1.57);
\definecolor{fillColor}{RGB}{8,95,2}

\path[fill=fillColor] (201.93, 81.16) --
	(204.04, 77.50) --
	(199.81, 77.50) --
	cycle;
\definecolor{drawColor}{RGB}{231,41,138}

\path[draw=drawColor,line width= 0.5pt,line join=round,line cap=round] (201.93, 76.95) --
	(204.04, 73.29) --
	(199.81, 73.29) --
	cycle;
\definecolor{fillColor}{RGB}{31,119,180}

\path[fill=fillColor] (211.40, 50.78) --
	(212.97, 52.35) --
	(214.54, 50.78) --
	(212.97, 49.21) --
	cycle;
\definecolor{drawColor}{RGB}{246,141,98}

\path[draw=drawColor,line width= 0.5pt,line join=round,line cap=round] (211.40, 61.68) -- (214.54, 64.82);

\path[draw=drawColor,line width= 0.5pt,line join=round,line cap=round] (211.40, 64.82) -- (214.54, 61.68);
\definecolor{fillColor}{RGB}{27,158,119}

\path[fill=fillColor] (212.97, 93.64) circle (  1.57);
\definecolor{drawColor}{RGB}{117,112,179}

\path[draw=drawColor,line width= 0.5pt,line join=round,line cap=round] (212.97, 90.33) circle (  1.57);
\definecolor{fillColor}{RGB}{8,95,2}

\path[fill=fillColor] (212.97, 96.79) --
	(215.08, 93.13) --
	(210.85, 93.13) --
	cycle;
\definecolor{drawColor}{RGB}{231,41,138}

\path[draw=drawColor,line width= 0.5pt,line join=round,line cap=round] (212.97, 90.22) --
	(215.08, 86.56) --
	(210.85, 86.56) --
	cycle;
\definecolor{drawColor}{gray}{0.20}

\path[draw=drawColor,line width= 0.5pt,line join=round,line cap=round] (134.35, 16.56) rectangle (216.71, 98.05);
\end{scope}
\begin{scope}
\path[clip] (  0.00,  0.00) rectangle (216.81,108.41);
\definecolor{drawColor}{gray}{0.30}

\node[text=drawColor,anchor=base east,inner sep=0pt, outer sep=0pt, scale=  0.50] at (130.30, 21.49) {1e-03};

\node[text=drawColor,anchor=base east,inner sep=0pt, outer sep=0pt, scale=  0.50] at (130.30, 38.06) {1e-02};

\node[text=drawColor,anchor=base east,inner sep=0pt, outer sep=0pt, scale=  0.50] at (130.30, 54.62) {1e-01};

\node[text=drawColor,anchor=base east,inner sep=0pt, outer sep=0pt, scale=  0.50] at (130.30, 71.19) {1e+00};

\node[text=drawColor,anchor=base east,inner sep=0pt, outer sep=0pt, scale=  0.50] at (130.30, 87.75) {1e+01};
\end{scope}
\begin{scope}
\path[clip] (  0.00,  0.00) rectangle (216.81,108.41);
\definecolor{drawColor}{gray}{0.20}

\path[draw=drawColor,line width= 0.5pt,line join=round] (132.10, 23.21) --
	(134.35, 23.21);

\path[draw=drawColor,line width= 0.5pt,line join=round] (132.10, 39.78) --
	(134.35, 39.78);

\path[draw=drawColor,line width= 0.5pt,line join=round] (132.10, 56.34) --
	(134.35, 56.34);

\path[draw=drawColor,line width= 0.5pt,line join=round] (132.10, 72.91) --
	(134.35, 72.91);

\path[draw=drawColor,line width= 0.5pt,line join=round] (132.10, 89.47) --
	(134.35, 89.47);
\end{scope}
\begin{scope}
\path[clip] (  0.00,  0.00) rectangle (216.81,108.41);
\definecolor{drawColor}{gray}{0.20}

\path[draw=drawColor,line width= 0.5pt,line join=round] (138.10, 14.31) --
	(138.10, 16.56);

\path[draw=drawColor,line width= 0.5pt,line join=round] (165.84, 14.31) --
	(165.84, 16.56);

\path[draw=drawColor,line width= 0.5pt,line join=round] (193.58, 14.31) --
	(193.58, 16.56);
\end{scope}
\begin{scope}
\path[clip] (  0.00,  0.00) rectangle (216.81,108.41);
\definecolor{drawColor}{gray}{0.30}

\node[text=drawColor,anchor=base,inner sep=0pt, outer sep=0pt, scale=  0.50] at (138.10,  9.07) {10};

\node[text=drawColor,anchor=base,inner sep=0pt, outer sep=0pt, scale=  0.50] at (165.84,  9.07) {100};

\node[text=drawColor,anchor=base,inner sep=0pt, outer sep=0pt, scale=  0.50] at (193.58,  9.07) {1000};
\end{scope}
\begin{scope}
\path[clip] (  0.00,  0.00) rectangle (216.81,108.41);
\definecolor{drawColor}{RGB}{0,0,0}

\node[text=drawColor,anchor=base,inner sep=0pt, outer sep=0pt, scale=  0.60] at (175.53,  1.42) {Sample Size  $(n)$};
\end{scope}
\begin{scope}
\path[clip] (  0.00,  0.00) rectangle (216.81,108.41);
\definecolor{drawColor}{RGB}{0,0,0}

\node[text=drawColor,rotate= 90.00,anchor=base,inner sep=0pt, outer sep=0pt, scale=  0.60] at (112.64, 57.31) {Runtime (s)};
\end{scope}
\end{tikzpicture}